\documentclass{article} %
\usepackage{iclr2027_conference,times}

\usepackage{amsmath,amsfonts,bm}

\def\eqref#1{equation~\ref{#1}}

\def\1{\bm{1}}

\def\eps{{\epsilon}}

\DeclareMathAlphabet{\mathsfit}{\encodingdefault}{\sfdefault}{m}{sl}
\SetMathAlphabet{\mathsfit}{bold}{\encodingdefault}{\sfdefault}{bx}{n}

\newcommand{\E}{\mathbb{E}}

\newcommand{\R}{\mathbb{R}}

\newcommand{\KL}{D_{\mathrm{KL}}}

\usepackage{hyperref,cleveref}
\usepackage{url}

\usepackage{amsmath,amssymb,amsthm,mathtools}
\usepackage{algorithm}
\usepackage{algpseudocode}
\usepackage{enumitem}
\usepackage{xcolor}         %
\usepackage{booktabs}
\usepackage{multirow}
\usepackage{graphicx}
\usepackage{colortbl}
\usepackage{xfp}
\usepackage{capt-of}
\usepackage[normalem]{ulem}

\newtheorem{theorem}{Theorem}
\newtheorem{lemma}{Lemma}
\newtheorem{proposition}{Proposition}
\newtheorem{definition}{Definition}
\newtheorem{corollary}{Corollary}
\newtheorem{assumption}{Assumption}

\theoremstyle{remark}
\newtheorem{remark}{Remark}

\newcommand{\cA}{\mathcal{A}}
\newcommand{\cD}{\mathcal{D}}
\newcommand{\cM}{\mathcal{M}}
\newcommand{\cF}{\mathcal{F}}
\newcommand{\cH}{\mathcal{H}}

\newcommand{\norm}[1]{\left\|#1\right\|}
\newcommand{\ip}[2]{\left\langle #1,\,#2\right\rangle}

\newcommand{\ceps}{c_{\eps}}

\definecolor{ourrow}{gray}{0.92}
\newcommand{\ours}{\rowcolor{ourrow}}

\definecolor{heatgray}{gray}{0.60}

\title{Asynchronous LLM Post-Training: Group-Mass Capping and Convergence Analysis}

\author{
Qijia He$^{1}$ \quad Ruinan Jin$^{1}$ \quad Jun Luo$^{1}$  \quad Shaofeng Zou$^{2}$  \quad Yingbin Liang$^{1}$ \\
$^{1}$The Ohio State University \qquad $^{2}$Arizona State University \\
}

\iclrfinalcopy %
\begin{document}
\maketitle
\fancyhead[L]{Preprint}     %

\begin{abstract}
Asynchronous reinforcement learning (RL) improves the efficiency of large language model post-training but introduces stale rollouts generated by earlier policies. Theoretical understanding of how this staleness affects convergence and how to mitigate its impact remains limited. We derive a convergence bound for GRPO-style algorithms that explicitly characterizes the tradeoff between the gradient estimator's second moment and bias. For trajectory-level importance-weighted estimators, our analysis shows that once the second moment is uniformly controlled, delay enters the bound through the bias introduced by clipping or rescaling. Guided by this insight, we propose a novel group mass capping GRPO (GMC-GRPO) method, which minimizes a ratio-based bias bound within a class of weighted estimators sharing a common second-moment guarantee. We establish convergence guarantees for asynchronous GMC-GRPO and show that, compared with TIC-GRPO, it improves the threshold dependence of the fourth-order delay term from $O(\eps^{-4})$ to $O(\eps^{-2})$ as $\eps\to0$, where $1+\eps$ is the ratio threshold. Under local policy overlap, the delay-dependent term decreases as $G^{-2/5}$ after tuning the step size, where $G$ is the group size. For fixed behavior and current policies, the bias introduced by group rescaling also vanishes as $G\to\infty$, whereas the bias from trajectory-wise clipping can persist. Experiments across Qwen3 models and reasoning benchmarks demonstrate improved robustness to stale rollouts, with GMC-GRPO achieving the best performance among stable baselines under large rollout delays.

\end{abstract}

\section{introduction}
\vspace{-0.4em}

Reinforcement learning (RL) has become a widely used paradigm for post-training
large language models (LLMs), enabling them to improve from outcome-level
feedback rather than relying solely on static demonstrations \citep{ouyang2022training, lambert2024tulu}.
Given a prompt, RL allows an LLM to generate thinking tokens before producing a
final answer, and reinforces the generated trajectory based on its final quality
\citep{jaech2024openai, guo2025deepseek}.
Models post-trained with algorithms such as GRPO \citep{shao2024deepseekmath} and
its variants \citep{yu2026dapo, liu2025understanding, zheng2026prosperity} have
shown strong capabilities on challenging reasoning problems, such as math  \citep{ cobbe2021training,he2024olympiadbench},
coding \citep{jain2025livecodebench}, and agentic tasks \citep{du2026pact}.

To reduce the inefficiency of repeatedly alternating between rollout generation and policy optimization, modern RL training pipelines increasingly adopt {\bf asynchronous} execution \citep{zhong2024rlhfuse,sheng2025hybridflow}. Such systems decouple rollout generation from policy optimization: rollout workers continuously generate responses using their local policy versions, while the learner updates the policy as soon as sufficient completed rollouts have accumulated and publishes the updated parameters without blocking generation \citep{fu2026areal,hu2026dora,noukhovitch2025asynchronous}. Together with optimized training engines and workload scheduling, this design allows inference and training to proceed concurrently, improving resource utilization and substantially increasing the training throughput of GRPO-style methods \citep{wang2025infinite}.

The decoupling, however, introduces \emph{data staleness}. Since rollout workers and the learner progress independently, the learner may update the policy several times before a generated rollout is consumed. Consequently, a training batch can mix trajectories collected under multiple earlier policy versions, creating a mismatch between the behavior policies that generated the data and the current policy being optimized \citep{fu2026areal}. As the delay grows, these trajectories can become increasingly off-policy, making policy optimization more difficult and reducing training efficiency \citep{tian2026off, zheng2026prosperity, xu2026gac}. Despite the empirical success of delay-aware algorithms \citep{roux2025tapered, xi2026bapo, huang2026stable}, how this staleness affects the convergence of GRPO-style algorithms, and how to mitigate its impact through principled algorithm design, remain poorly understood. This raises a natural but fundamental research question:

\textit{Can we theoretically characterize how data staleness affects the convergence of GRPO-style algorithms, and use this to design RL algorithms better suited to asynchronous training systems?}

We answer this question by first identifying what any importance-weighted algorithm must control in order to converge, and then designing the correction that best meets this requirement under delay. Our key contributions are as follows:

\begin{itemize}[leftmargin=4mm, itemsep=2pt, topsep=2pt, parsep=0pt, partopsep=0pt]
    \item \textbf{Second moment and bias tradeoff.}
    We derive a convergence bound for GRPO-style algorithms which explicitly captures the tradeoff between the second moment of gradient estimators and the bias error. With trajectory-level importance sampling, stale rollouts can produce large importance ratios and enlarge the gradient estimator's second moment. Clipping these ratios controls the second moment but can introduce bias. When the resulting second-moment bound is independent of delay, delay enters the convergence bound through this clipping bias. This motivates algorithms that minimize a bias bound while maintaining a common second-moment guarantee.

    \item \textbf{GMC-GRPO.} Guided by this principle, we introduce a class of weighted gradient estimators that %
    satisfy a common second-moment bound. This class includes several GRPO-style methods as special cases. Within this class, we derive group mass capping (GMC)-GRPO by minimizing a bias bound and selecting a solution that preserves each group's gradient direction, while retaining the same second-moment guarantee. This makes GMC-GRPO particularly well suited to asynchronous training, where its principled weighting rule effectively limits the bias.

    \item \textbf{Convergence of asynchronous GMC-GRPO.}
    We establish a convergence guarantee for asynchronous GMC-GRPO with explicit dependence
    on stochastic rollout delays. Compared with TIC-GRPO \citep{pang2025tic}, GMC-GRPO yields a strictly smaller delay-dependent
    penalty: in the tight-clipping regime, the penalty improves from
    $O(\eps^{-4})$ to $O(\eps^{-2})$, where $1+\eps$ is the threshold for individual importance ratios in TIC-GRPO and their group average in GMC-GRPO. Moreover, because GMC-GRPO rescales a group
    only when its average importance ratio exceeds the stability cap, its
    correction bias vanishes as $G$ grows and the group average concentrates. In contrast, trajectory-wise clipping acts separately on every ratio, leaving a
    nonzero clipping bias even in the large-group limit. 

   \item \textbf{Experiments.}
    Finally, we evaluate GMC-GRPO under both controlled fixed delays and in a simulated asynchronous pipeline where rollout staleness is correlated with response
    length. Across Qwen3 models and reasoning benchmarks, GMC-GRPO achieves the best performance among vanilla GRPO-style algorithms and other baselines designed for asynchronous training. These results demonstrate effectiveness of our theory-guided design under asynchronous rollout delays.
    
\end{itemize}

\textbf{Related Work.} 
We defer the comprehensive discussion of related work to Appendix~\ref{app:related work}. 
Our paper is most closely related to off-policy correction for stale LLM rollouts in asynchronous RL systems \citep{noukhovitch2025asynchronous,fu2026areal,hu2026dora}.
Typical strategies include adaptive clipping thresholds \citep{roux2025tapered,xi2026bapo,zhao2026deconstructing, li20253po}, variance-aware step-size control \citep{zheng2026prosperity,huang2026stable,xu2026gac}, and reuse of stale experience through staged training with selective filtering or replay buffers \citep{tian2026off,arnal2026efficient,ritter2026llms}.
Despite the empirical success of these methods, few works have studied GRPO-style algorithms from a theoretical perspective, especially under asynchronous rollout delays.
\citet{zhou2026demystifying} characterize the
GRPO gradient as a U-statistic and analyze its finite-sample error and
dependence on group size, while \citet{pang2025tic} introduce TIC-GRPO and
provide the first rigorous convergence analysis of GRPO-style methods. Nevertheless, these studies are not under asynchronous training. More
recently, \citet{xu2026gac} analyze asynchronous GRPO through a generic
biased-gradient framework, but leave the staleness-induced bias and higher-order
remainder terms as abstract quantities in the final bound. \textbf{To the best of our knowledge, we are the first to provide an
explicit delay-dependent convergence analysis of GRPO-style algorithms in asynchronous RL systems.}

\section{Preliminaries}
\label{sec:prelim}\vspace{-0.4em}

\subsection{RL post-training}\label{sec:prelim-rl}
\vspace{-0.4em}

We denote $\pi_\theta$ as an autoregressive language model parameterized
by $\theta\in\mathbb{R}^d$.
Given a prompt $x$ drawn from a training distribution $\cD$ over
the prompt space $\mathcal{X}$, the likelihood of a response
$y=(y_1,\ldots,y_T)$ is
$\pi_\theta(y\mid x)
=\prod_{t=1}^{T}\pi_\theta(y_t\mid x,y_{<t})$,
where $\cA$ denotes the token vocabulary,
$y_{<t}:=(y_1,\ldots,y_{t-1})$ is the generated prefix,
and $T$ is the maximum response length.
A verifier $r_\phi:\mathcal{X}\times\cA^T\to[0,R]$
scores each completed response.

The goal of post-training is to maximize the expected verifier score.
Following standard practice in reinforcement learning, we divide each global batch into
mini-batches and update the parameters sequentially.
The rollout samples used by an update may therefore be generated by
a behavior policy $\pi_{\theta_{\mathrm{behav}}}$ that differs
from the current policy.
Using trajectory-level importance sampling to account for this
mismatch, we write the objective as
\begin{equation}\label{eq:rl-objective}
\begin{aligned}
J(\theta)
&:= \E_{x\sim\cD}\mathbb{E}_{y\thicksim\pi_\theta(\cdot|x)}\bigl[r_\phi(x,y)\bigr] = \E_{x\sim\cD}
   \E_{y\sim\pi_{\theta_{\mathrm{behav}}}(\cdot\mid x)}
   \bigl[\rho\,r_\phi(x,y)\bigr],
\end{aligned}
\end{equation}
where
$\rho:=\pi_\theta(y\mid x)/
\pi_{\theta_{\mathrm{behav}}}(y\mid x)$.
We further denote the optimal value by
$J^\star:=\sup_{\theta\in\mathbb{R}^d}J(\theta)$.

To find an optimal solution, at learner iteration $n\in\{1,\ldots,N\}$, where $N\ge1$ is the total number of updates, GRPO style algorithms draws $B$ prompts
$x^{(1)},\ldots,x^{(B)}$ from $\cD$ and a group of $G\ge2$ responses
$y^{(b,1)},\ldots,y^{(b,G)}\sim\pi_{\theta_{\mathrm{behav}}}(\cdot\mid x^{(b)})$
per prompt, and takes the ascent step $\theta_{n+1}=\theta_n+\eta\,
g_n^{\mathrm{IS}}$ with step size $\eta>0$, starting from $\theta_0$ and ending at $\theta_N$. Here $g_n^{\mathrm{IS}}$ is the
stochastic policy gradient of \cref{eq:rl-objective} on the returned mini-batch, normalized by the response length $T$,
\begin{equation}\label{eq:tic-grad}
g_n^{\mathrm{IS}}
:=\frac1{TBG}\sum_{b=1}^{B}\sum_{i=1}^{G}
\rho_n^{(b,i)}
\nabla\log\pi_{\theta_n}\bigl(y^{(b,i)}\mid x^{(b)}\bigr)\,
\widehat A^{(b,i)},
\end{equation}
where $\rho_n^{(b,i)}:=\pi_{\theta_n}(y^{(b,i)}\mid
x^{(b)})/\pi_{\theta_{\mathrm{behav}}}(y^{(b,i)}\mid x^{(b)})$ is the importance ratio evaluated at $\theta_n$, and
$\widehat A^{(b,i)}:=r^{(b,i)}-\frac1{G-1}\sum_{j\ne i}r^{(b,j)}$ with
$r^{(b,i)}:=r_\phi(x^{(b)},y^{(b,i)})$ being the leave-one-out group advantage. Both
the ratio and the score function $\nabla_\theta\log\pi_\theta(y\mid x)$ act on the whole trajectory, so a single scalar
advantage is shared by all tokens of $y^{(b,i)}$. 

The policy gradient in \cref{eq:tic-grad} depends on the policy through
its score function and the importance sampling ratio.
Our later analysis will be based on the
following token-level regularity and coverage assumptions.

\begin{assumption}[Token-level Lipschitz score]
\label{as:lip}
There is $L>0$ such that, for every token context and all $\theta,\theta'$,
\[
\norm{\nabla_\theta\log\pi_\theta(y_t\mid x,y_{<t})
-\nabla_\theta\log\pi_{\theta'}(y_t\mid x,y_{<t})}_2
\le L\norm{\theta-\theta'}_2.
\]
\end{assumption}
\begin{assumption}[Token-level coverage]\label{as:coverage}
There exists $\mathcal{C}\ge1$ such that $\pi_\theta(y_t\mid x,y_{<t})\ge1/{\mathcal{C}}
$
for every token, context, and learner or behavior policy encountered during
training.
\end{assumption}
Although the constant $\mathcal{C}$ can be large in LLM settings, it appears only through $\log\mathcal{C}$ in our analytical bounds, which justifies it as a controlled constant.

\subsection{Asynchronous RL Systems and Data Staleness}
\label{sec:prelim-async}
\vspace{-0.4em}

As illustrated in \Cref{fig:async}, we consider an asynchronous RL
system in which multiple rollout workers generate responses
continuously while a single learner updates the policy
\citep{fu2026areal, hu2026dora, noukhovitch2025asynchronous}.
Each worker draws a prompt $x\sim\cD$ and generates a group of $G$
responses under the policy snapshot it holds at launch.
Once all responses have been scored, the worker deposits the group
and its behavior-policy version into a shared buffer. At iteration $n$, the learner takes $B$ completed groups from the
buffer, performs the update in \cref{eq:tic-grad} at $\theta_n$, and publishes
$\theta_{n+1}$ while rollout generation continues.
Since groups start under different snapshots and take different
amounts of time to generate, score, and wait in the buffer, a single
mini-batch can mix several degrees of staleness.

\noindent
\begin{minipage}[t]{0.68\linewidth}
\vspace{-10pt}
Formally, let $s_n^{(b)}\in\{0,\ldots,n\}$ denote the parameter-version index used
to generate group $b$ consumed at iteration $n$, and define its delay as
$\tau_n^{(b)}:=n-s_n^{(b)}$, the lag between the behavior-policy version and the
current learner version, measured in learner updates. All $G$ responses in a group
share one delay, whereas groups within a mini-batch need not. For example, a batch used at
$\theta_8$ may contain groups generated under $\theta_5$ and $\theta_7$, with
delays $3$ and $1$. The behavior policy of Section~\ref{sec:prelim-rl} is thus
group dependent, $y^{(b,i)}\sim\pi_{\theta_{n-\tau_n^{(b)}}}(\cdot\mid x^{(b)})$,
and the ratio in \cref{eq:tic-grad} becomes
\begin{equation}\label{eq:delayed-ratio}
\rho_n^{(b,i)}
=\frac{\pi_{\theta_n}(y^{(b,i)}\mid x^{(b)})}
{\pi_{\theta_{n-\tau_n^{(b)}}}(y^{(b,i)}\mid x^{(b)})}.
\end{equation}

To state the sampling conditions for \cref{eq:delayed-ratio},
let $\cF_n^-$ collect the randomness available before the mini-batch is
returned, so that $\{\theta_m\}_{m=0}^n$ is $\cF_n^-$-measurable while the prompts
and delays carried by the current batch need not be, and let
$\cH_n:=\sigma(\cF_n^-,\{x^{(b)},\tau_n^{(b)}\}_{b=1}^B)$. Every behavior parameter
$\theta_{n-\tau_n^{(b)}}$ is $\cH_n$-measurable, so \cref{eq:delayed-ratio} can be
analyzed with the behavior policy held fixed.
For our analysis, we impose the following regularity assumption on asynchronous sampling.
\end{minipage}\hfill
\begin{minipage}[t]{0.30\linewidth}
\vspace{0pt}
\centering
\includegraphics[width=\linewidth]{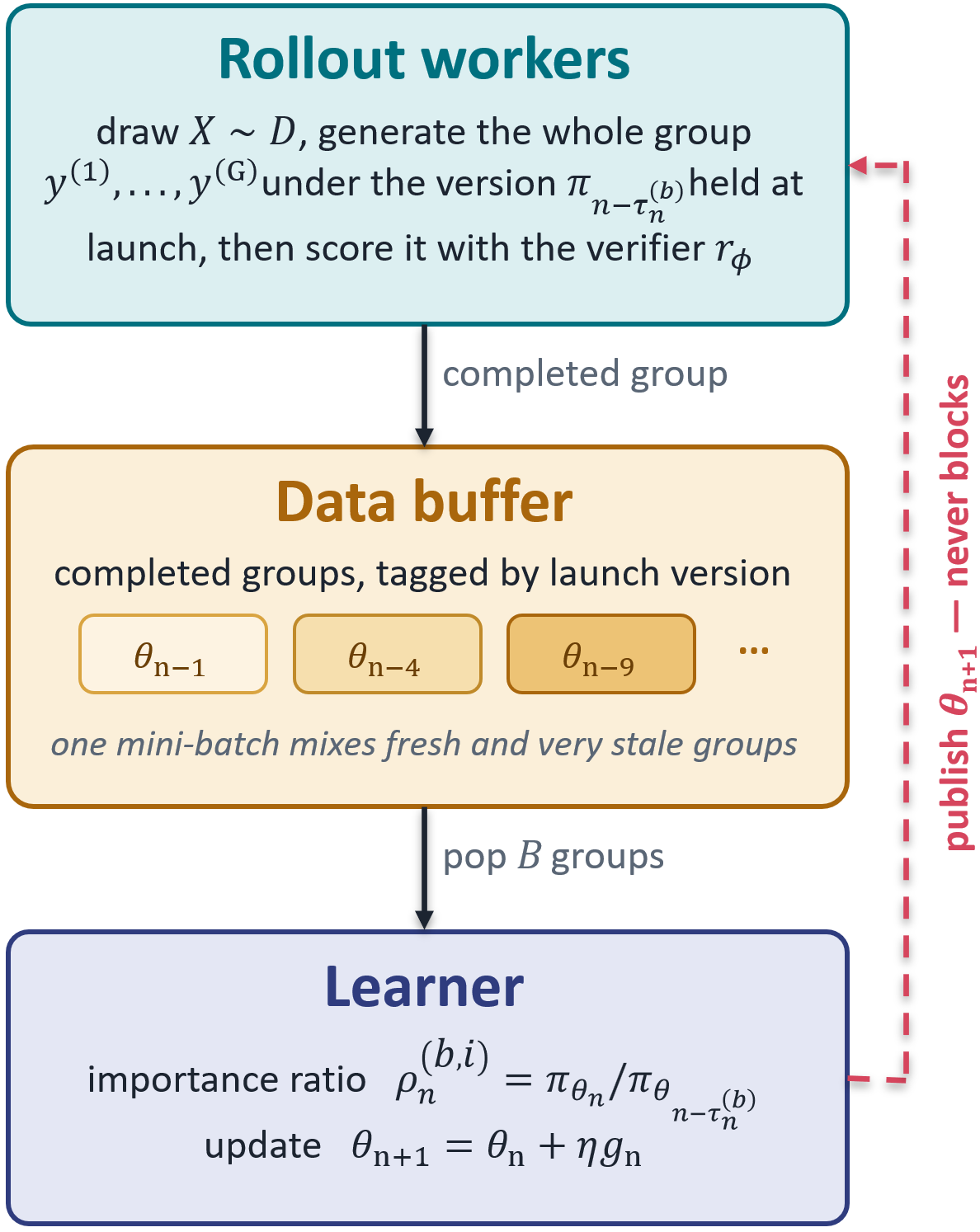}
\captionof{figure}{Asynchronous rollout. Each worker generates a group using its launch-time policy, a batch may contain groups with different delays.}
\label{fig:async}
\end{minipage}
\par\medskip

\begin{assumption}[Asynchronous sampling regularity]\label{as:async}
Conditional on $\cH_n$, the responses within each returned group are
independent draws from
$\pi_{\theta_{n-\tau_n^{(b)}}}(\cdot\mid x^{(b)})$.
The returned mini-batch preserves conditional unbiasedness of the raw
importance-sampling update:
$\E[g_n^{\mathrm{IS}}\mid\cF_n^-]
=T^{-1}\nabla J(\theta_n)$.
Queue admission and mini-batch selection may otherwise depend on launch and
completion times. For the sharper delay-moment conversion only, we additionally
require the delay-regularity condition
$\E[\norm{\theta_n-\theta_{n-t}}_2^2\mid\tau_n^{(b)}=t]
=\E\norm{\theta_n-\theta_{n-t}}_2^2$
whenever $\Pr(\tau_n^{(b)}=t)>0$.
\end{assumption}
Assumption~\ref{as:async} is a natural abstraction for stationary, large-scale
RL training. With a fixed prompt stream, lossless rollout collection, and many
concurrent workers, the delay of any individual rollout is approximately
decoupled from both the distribution of returned samples and the aggregate
learner trajectory. Thus, delay affects only the staleness of a rollout, without systematically biasing the returned data or the associated learner drift. We exclude RL training systems that explicitly filter rollouts by
completion time, such as timeouts or maximum-delay gating \citep{fu2026areal,zeng2026glm}, as they can
favor shorter responses and introduce length bias \citep{wang2025infinite}.

\section{Second Moment-Bias tradeoff and Group-Mass-Capping GRPO} \label{sec3}
\vspace{-0.4em}

How we design gradient estimators plays an important role in convergence, particularly in asynchronous training. In this section, we first establish a general convergence bound for a broad class of gradient estimators, which shows how their second moments and bias affect the convergence. This bound will then guide us toward a principled gradient estimator that minimizes a bias upper bound while maintaining a common second-moment guarantee, thereby reducing sensitivity to asynchronous rollout delays.

\begin{lemma}
\label{lem:master-descent}
Suppose that $\theta_{n+1}=\theta_n+\eta g_n$ and
$\E[g_n\mid\cF_n^-]=T^{-1}\nabla J(\theta_n)+b_n$. Under Assumption~\ref{as:lip}-\ref{as:coverage}, for every $\eta>0$ and
$N\ge1$, 
\[
\frac{\sum_{n=0}^{N-1}
\E\norm{\nabla J(\theta_n)}_2^2}{N}
\!\le\!
\frac{2T\bigl(J^\star\!-\!\E[J(\theta_0)]\bigr)}
{\eta N}
\!+\!
\frac{T^2}{N}\sum_{n=0}^{N-1}
\Bigl[
\eta RL(2\log|\cA|\!+\!1)\E\norm{g_n}_2^2
\!+\!\E\norm{b_n}_2^2
\Bigr].
\]
\end{lemma}

The proof of Lemma \ref{lem:master-descent} is deferred to Appendix~\ref{sec:assumptions}. Lemma~\ref{lem:master-descent} characterizes the convergence bound into three terms. The first is the standard decay term $O(1/(\eta N))$, which vanishes as the number $N$ of iterations becomes large. The other two depend on the \textbf{gradient estimator's second moment} $\E\norm{g_n}_2^2$ and its \textbf{squared bias} relative to the current policy gradient $\E\norm{b_n}_2^2$. A good gradient estimator should therefore keep both quantities small. A common way to prevent gradient's second moment become large is to {\em clip} individual samples whose importance ratios exceed a prescribed threshold, as in GRPO and its variants. However, clipping can also enlarge bias and slow convergence \citep{zheng2026prosperity}. This tradeoff becomes especially important in asynchronous post-training, where delayed rollouts increase the mismatch between the behavior and learner policies. As this mismatch grows, more importance ratios may exceed the clipping threshold, increasing the bias introduced by clipping. Although raising the clipping threshold allows more samples to contribute, it weakens control over the gradient estimator's second moment and can lead to unstable training or even collapse \citep{zheng2026prosperity, huang2026stable, tian2026off}.

This tradeoff motivates us to consider a class of {\bf weighted} gradient estimators, in which each sample receives an additional weight to adjust its contribution. We require the average weighted importance ratio within each group to stay below a fixed threshold, while allowing individual weighted ratios to exceed it. As we show below, this constraint gives all estimators in the class a common second-moment guarantee. We can then compare different weighting rules and choose one that minimizes the bias bound.

\begin{definition}[Weighted gradient estimator class]\label{def:admissible class}
Define the weighted gradient estimator as
\begin{equation}
\label{eq:weighted-update}
g_n^{w}:=\frac1B\sum_{b=1}^B \frac1{TG}\sum_{i=1}^G \boldsymbol{w}_n^{(b,i)}\rho_n^{(b,i)}
\nabla\log\pi_{\theta_n}(y^{(b,i)}\mid x^{(b)})\widehat A^{(b,i)},
\end{equation}
where the weights $w_n^{(b,i)}$ should satisfy, for every realized group,
\begin{equation}\textstyle
\label{eq:admissible-weights}
0\le w_n^{(b,i)}\le1\quad\text{for all }i\in[G],
\qquad
\frac1G\sum_{i=1}^G w_n^{(b,i)}\rho_n^{(b,i)}\le1+\eps.
\end{equation}
\end{definition}

\begin{remark}
In \cref{eq:admissible-weights}, the first condition prevents the correction from amplifying any trajectory, and the second precisely controls the group's average gradient second moment. Such a class is broad, and includes many existing algorithms as special cases. For example $w_n^{(b,i)}=\mathbf 1\{\rho_n^{(b,i)}\le1+\eps\}$ corresponds to TIC-GRPO \citep{pang2025tic};  $w_n^{(b,i)}=\min\{1,c/\rho_n^{(b,i)}\}$ corresponds to VCPO \citep{huang2026stable} which applies sequence-level truncated importance sampling with a detached weight $\min\{\rho,c\}$; and CISPO \citep{chen2025minimax} and SAPO \citep{gao2025soft} can be viewed as token-level analogues. However, the vanilla GRPO \citep{shao2024deepseekmath} is not in this class, because for negative advantages, it retains the gradient of tokens with arbitrarily large ratios, so that its second moment scales as $O(\mathcal{C}^2)$ in the worst case \citep{pang2025tic}. 
\end{remark}

The following proposition shows that all gradient estimators in this class satisfy a common second-moment bound. This allows us to design a principled bias-minimal weighting rule.
\begin{proposition}
\label{prop:weighted-second-moment}
Under Assumptions~\ref{as:lip}-\ref{as:async}, the gradient estimators in the class given in \Cref{def:admissible class} share the following common bound
\begin{equation}\textstyle
\label{eq:weighted-second-moment}
\E[\norm{g_n^{w}}_2^2\mid\cF_n^-]
\le\frac{8}{T}{(1+\eps)LR^2\log|\cA|}.
\end{equation}
Further, their group bias $b_{n,b}^{w}
:=-\frac{1}{TG}\sum_{i=1}^{G}
(1-w_n^{(b,i)})\rho_n^{(b,i)}
\nabla\log\pi_{\theta_n}(y^{(b,i)}\mid x^{(b)})
\widehat A^{(b,i)}$ satisfies
\begin{equation}\textstyle
\label{eq:weighted-bias-envelope1}
\norm{\E[b_{n,b}^{w}\mid\cH_n]}_2
\le
2R\sqrt{2L\log\mathcal C}\,
\E\left[\frac{1}{G}\sum_{i=1}^G
(1-w_n^{(b,i)})\rho_n^{(b,i)}\mid\cH_n\right].
\end{equation}
\end{proposition}

\Cref{prop:weighted-second-moment} shows that all weighting rules in this class share a common second-moment guarantee. With this guarantee in place, we focus on choosing weights that reduce the bias. Ideally, we would minimize the exact bias $\|\E[b_{n,b}^{w}\mid\cH_n]\|_2$. However, this quantity depends on the directions of individual gradient contributions and how they cancel, which cannot generally be determined from the importance ratios alone. We therefore minimize the upper bound in \cref{eq:weighted-bias-envelope1}, which depends only on the importance ratios and the chosen weights. This bound is proportional to the expected reduction in the group’s average importance ratio caused by weighting. Since the unweighted average $\frac1G\sum_{i=1}^G\rho_n^{(b,i)}$ is fixed for each realized group, minimizing this reduction is equivalent to maximizing the average weighted importance ratio subject to the class constraints:
\begin{equation}\textstyle
\label{eq:retained-mass-program}
\max_{0\le w_n^{(b,i)}\le1}\ \frac1G\sum_{i=1}^G w_n^{(b,i)}\rho_n^{(b,i)}
\quad\text{subject to}\quad
\frac1G\sum_{i=1}^G w_n^{(b,i)}\rho_n^{(b,i)}\le1+\eps.
\end{equation}
The optimal weights are not unique in general. When the unweighted average importance ratio exceeds $1+\eps$, any feasible weighting rule whose average weighted importance ratio equals $1+\eps$ attains the optimum. Different optimal rules can nevertheless produce different gradient directions within the same group. To select among these rules, we restrict attention to \emph{direction-preserving} rules that assign a common positive weight to all responses in a group: $w_n^{(b,i)}=\alpha_n^{(b)}>0$ for all $i\in[G]$. This scales the original importance-weighted group gradient estimator by a single factor, preserving the relative contributions of the responses and hence the direction of the group gradient estimate. The following proposition gives the optimal weighting rule under this restriction.
\begin{proposition}
\label{prop:gmc-retained-mass}
Fix a realized group and let
$\overline\rho_n^{(b)}:=G^{-1}\sum_{i=1}^G\rho_n^{(b,i)}$.
Among weighted direction-preserving corrections, the unique solution to
\cref{eq:retained-mass-program} is $\alpha_n^{(b),\star}
=\min\left\{1,{(1+\eps)}/{\overline\rho_n^{(b)}}\right\}.$
\end{proposition}

We call this weighting rule \emph{group mass capping} (GMC), which caps the group’s average importance ratio by replacing each raw ratio with $\widehat\rho_n^{(b,i)}:=\alpha_n^{(b),\star}\rho_n^{(b,i)}$, yielding the following gradient estimator:
\begin{equation}
\label{eq:gmcgrad-main}
g_n^{\mathrm{GMC}}
:=\frac1{TBG}\sum_{b=1}^B\sum_{i=1}^G
\frac{\rho_n^{(b,i)}}{\max\{1,\bar{\rho}_n^{(b)}/(1+\eps
)\}}
\nabla\log\pi_{\theta_n}(y^{(b,i)}\mid x^{(b)})\widehat A^{(b,i)}.
\end{equation}
We provide a comprehensive pseudo code of asynchronous GMC-GRPO in Appendix \ref{sec:algorithm}.

\section{Convergence Analysis}
\vspace{-0.4em}

\subsection{Convergence Guarantees for GMC-GRPO}
\vspace{-0.4em}

We now apply Lemma~\ref{lem:master-descent} to establish convergence
guarantees for asynchronous GMC-GRPO.
Since Proposition~\ref{prop:weighted-second-moment} already controls the gradient second moment, it remains to bound the correction bias in terms of rollout delays.
The resulting bound depends on the averaged fourth  moment of the delay defined as
$\overline{\tau}_4
:= \frac{1}{NB}\sum_{n=0}^{N-1}\sum_{b=1}^{B}
\E[(\tau_n^{(b)})^4]$.
\begin{theorem}[Convergence of asynchronous GMC-GRPO]\label{thm:gmc-main}
    Under Assumptions~\ref{as:lip}--\ref{as:async}, for every $\eta>0$ and
    $N\ge1$, the convergence rate of GMC-GRPO satisfies
    \[
    \frac{1}{N}\sum_{n=0}^{N-1}
    \mathbb{E}\left\|\nabla J(\theta_n)\right\|_2^2
    =
    O\!\left(
    \frac{1}{\eta N}
    +(1+\eps)\eta
    +\frac{(1+\eps)^3}{\log^2(1+2\eps)}
    \eta^4\overline{\tau}_4
    \right).
    \]
For $\overline\tau_4>0$, taking
$\eta^\star\asymp\min\{1/\sqrt{(1+\eps)N},
(\log^2(1+2\eps)/((1+\eps)^3\overline\tau_4N))^{1/5}\}$
gives
\[
\frac{1}{N}\sum_{n=0}^{N-1}
\mathbb{E}\|\nabla J(\theta_n)\|_2^2
=
O\left(
\sqrt{\frac{1+\eps}{N}}
+
\left(
\frac{(1+\eps)^3\overline\tau_4}
{\log^2(1+2\eps)N^4}
\right)^{1/5}
\right).
\]
When $\overline\tau_4=0$, the convergence recovers the standard $O(N^{-1/2})$ rate for fixed $\eps$.
\end{theorem}

\begin{remark}[Effect of rollout delay on convergence rate] 
\label{rem:delay-tolerance}
We use \Cref{thm:gmc-main} to characterize how much rollout delay can be tolerated without degrading the convergence rate. Suppose $\overline\tau_4>0$.
If $\overline\tau_4=O(N^{3/2})$, choosing
$\eps\asymp\min\{1,\overline\tau_4^{1/7}N^{-3/14}\}$ and
$\eta\asymp N^{-1/2}$ preserves the no-delay $O(N^{-1/2})$
convergence guarantee; the delay contribution is lower order if
$\overline\tau_4=o(N^{3/2})$.
For larger scale delays, where
$\overline\tau_4=\Theta(N^p)$ with $3/2<p<4$, the rate-optimal choices
are $\eps=\Theta(1)$ and $\eta\asymp N^{-(p+1)/5}$, which yield the
slower convergence guarantee $O(N^{-(4-p)/5})$. Thus, the bound tolerates fourth delay moments up to $N^{3/2}$ without degrading the standard rate. Under polynomial scaling, moments below $N^4$ still yield a vanishing convergence guarantee, but at a slower convergence rate.
\end{remark}

\begin{corollary}[Asymptotically optimal GMC threshold]
\label{cor:gmc-threshold}
Under the conditions of Theorem~\ref{thm:gmc-main}, assume that
$\overline\tau_4 >0$ does not scale with $N$.
Then as $N\to\infty$, any minimizer of the leading $\eps$-dependent terms over $\eps>0$ in the derived GMC upper bound satisfies
$
\eps_{\mathrm{GMC}}^\star
=\Theta\left(\overline\tau_4^{1/7}N^{-3/14}\right).
$
\end{corollary}
\Cref{cor:gmc-threshold} reflects a simple intuition: larger delays favor a higher threshold to avoid reducing samples' contributions too aggressively, while more training iterations favor a lower threshold to keep the gradient estimator's second moment small.
\begin{corollary}
\label{cor:gmc-local-overlap}
Under the conditions of Theorem~\ref{thm:gmc-main}, further suppose that for some finite $\delta\ge0$,
$|\rho_n^{(b,i)}-1|\le\delta$ almost surely for all $n,b,i$.
For $\overline\tau_4>0$, taking
$\eta^\star\asymp\min\{1/\sqrt{(1+\eps)N},
(G^2\eps^2/((1+\delta)^2(1+\eps)^3\overline\tau_4N))^{1/5}\}$, 
then the group-concentration branch in the GMC bound gives
\[
\frac{1}{N}\sum_{n=0}^{N-1}
\mathbb{E}\|\nabla J(\theta_n)\|_2^2
=
O\left(
\sqrt{\frac{1+\eps}{N}}
+
\left(
\frac{
(1+\delta)^2(1+\eps)^3\overline{\tau}_4
}{
G^2\eps^2N^4
}
\right)^{1/5}
\right).
\]
\end{corollary}
\Cref{cor:gmc-local-overlap} shows that the delay-dependent term decreases as $G^{-2/5}$, so larger groups reduce the effect of rollout delays on the convergence bound.
\begin{remark}
\label{rem:group-size-gain}
Corollary~\ref{cor:gmc-local-overlap} reflects a
simple distinction between group-mass capping (GMC) and individual clipping. Importance
ratios have conditional mean one under the behavior policy, and GMC intervenes
only when their empirical group average is unusually large. As the group size
grows, this average concentrates around one, so this excess as well as the bias introduced by correcting it vanishes. In contrast, individual clipping makes a
separate decision for every trajectory \citep{pang2025tic, huang2026stable}; averaging more such decisions does not
remove the clipping bias already introduced at the trajectory level.

The condition $|\rho_n^{(b,i)}-1|\le\delta$ in
Corollary~\ref{cor:gmc-local-overlap} identifies a
near-policy regime in which this concentration benefit carries over to delayed
training. When the behavior and current policies assign comparable likelihoods to
the sampled trajectories, the importance ratios remain bounded, and the delay
term decays with the group size. As this overlap deteriorates, the constant
$(1+\delta)^2$ grows and the refinement loses its advantage, although
Theorem~\ref{thm:gmc-main} continues to apply. The group-size gain should
 be interpreted as a local benefit of aggregating trajectories with
controlled likelihood ratios, rather than a universal claim that larger groups
eliminate bias under arbitrary distribution shift.
\end{remark}

\subsection{Proof Outline and the Role of Rollout Delay}\label{sec:proofoutline}
\vspace{-0.4em}

We outline the proof of \Cref{thm:gmc-main}, which further explains how rollout delay affects convergence of asynchronous training. Theorem~\ref{thm:gmc-main} follows from Lemma~\ref{lem:master-descent} by bounding the second moment and bias term separately. Proposition~\ref{prop:weighted-second-moment} controls the second-moment term uniformly, independently of delay, which yields the $O(\eta)$ term. Delay therefore enters only through the correction bias, via a three-step chain. First, Proposition~\ref{prop:weighted-second-moment} bounds the bias of each group by the importance ratio it removes, which for GMC equals the group average's excess over the cap: $\E[(\overline\rho_n^{(b)}-(1+\eps))_+\mid\cH_n]$. Second, a convexity argument bounds this excess by $D_n^{(b)}/\log(1+2\eps)$, where $D_n^{(b)}:=\KL(\pi_{\theta_n}(\cdot\mid x^{(b)})\,\|\,\pi_{\theta_{n-\tau_n^{(b)}}}(\cdot\mid x^{(b)}))$ measures the mismatch between the current and behavior policies. Third, Assumption~\ref{as:lip} gives $D_n^{(b)}\le\frac{TL}{2}\|\theta_n-\theta_{n-\tau_n^{(b)}}\|_2^2$, and the uniform second-moment bound limits the parameter drift accumulated over $\tau_n^{(b)}$ updates to $O(\eta^2(\tau_n^{(b)})^2)$. Squaring the bias thus produces the term of order $\eta^4\overline\tau_4/\log^2(1+2\eps)$. The second-moment control therefore plays a dual role: it stabilizes each update, and it also limits how far the policy can drift within a delay window.

\begin{remark}[Comparison with asynchronous SGD]
\label{rem:asgd}
Delay plays a different role in Theorem~\ref{thm:gmc-main} than in classical asynchronous stochastic gradient descent (ASGD) \citep{tan2023abs, maranjyan2025ringmaster, koloskova2022sharper, mishchenko2022asynchronous}. In ASGD, an update applied at iteration $n$ is conditionally unbiased for the gradient at the stale iterate $\theta_{n-\tau_n}$ rather than at the current iterate $\theta_n$, so staleness enters the descent argument directly through the discrepancy between the two gradients. In our setting, the trajectory-level ratio instead changes measure from the behavior policy $\pi_{\theta_{n-\tau_n^{(b)}}}$ to the current policy $\pi_{\theta_n}$, so the uncorrected estimator remains unbiased for $\nabla J(\theta_n)$ regardless of delay. Staleness instead enlarges the spread of the importance ratios, which makes a second-moment correction necessary; delay then affects the bound only through the bias of that correction as explained above.
This difference also changes the appropriate remedy. ASGD commonly bounds staleness by discarding updates whose delay exceeds a threshold \citep{tan2023abs, maranjyan2025ringmaster}, but the analogous gating in LLM post-training preferentially discards long responses and thus induces a length bias \citep{wang2025infinite}. Improving the delay dependence therefore calls for RL algorithms that minimize the bias under the same second-moment control, which is precisely what group-mass-capping is designed to do.
\end{remark}

\subsection{Comparison with TIC-GRPO}
\vspace{-0.4em}

We now compare GMC-GRPO with TIC-GRPO \citep{pang2025tic}, which replaces the token-level ratios of GRPO with a single trajectory-level ratio and clips it from above, corresponding to $w_n^{(b,i)}=\mathbf 1\{\rho_n^{(b,i)}\le1+\eps\}$ in \cref{eq:weighted-update}. TIC-GRPO is a natural baseline for two reasons. First, it provably converges faster than GRPO \citep{pang2025tic}. Second, it belongs to the weighted gradient estimator class and therefore shares the second-moment bound of Proposition~\ref{prop:weighted-second-moment} with GMC-GRPO. The following theorem establishes the convergence of TIC-GRPO under stochastic delays.

\begin{theorem}[Convergence of asynchronous TIC-GRPO]\label{thm:tic-main}
Let $\ceps:=\log(1+\eps)-\eps/(1+\eps)$. Under Assumptions~\ref{as:lip}--\ref{as:async}, for every $\eta>0$ and $N\ge1$, TIC-GRPO satisfies the bound of Theorem~\ref{thm:gmc-main} with $\log^2(1+2\eps)$ replaced by $\ceps^2$. Consequently, for $\overline\tau_4>0$, taking $\eta^\star\asymp\min\{1/\sqrt{(1+\eps)N},(\ceps^2/((1+\eps)^3\overline\tau_4N))^{1/5}\}$ gives
\[
\frac1N\sum_{n=0}^{N-1}\E\|\nabla J(\theta_n)\|_2^2
=O\left(\sqrt{\frac{1+\eps}{N}}+\left(\frac{(1+\eps)^3\overline\tau_4}{\ceps^2N^4}\right)^{1/5}\right).
\]
\end{theorem}

\begin{remark}[Convergence bound improvement]
\label{rem:gmc-fourth-order}
Theorems~\ref{thm:gmc-main} and~\ref{thm:tic-main} differ only in the
fourth-order delay coefficient: TIC-GRPO replaces $\log^2(1+2\eps)$ in the
GMC-GRPO bound with $\ceps^2$. Since $\ceps<\log(1+2\eps)$ for every
$\eps>0$, the GMC coefficient is strictly smaller, and after optimizing the
step size the delay-dependent rate improves by the factor
$(\ceps/\log(1+2\eps))^{2/5}<1$. This distinction is substantial for a tight
threshold. As $\eps\to0$, $\ceps=\Theta(\eps^2)$ while
$\log(1+2\eps)=\Theta(\eps)$; hence the TIC penalty scales as $\eps^{-4}$,
whereas the GMC penalty scales as $\eps^{-2}$. GMC obtains this improvement
without enlarging the conditional second-moment upper bound.

This improvement arises from how much importance weight each algorithm
removes. 
Following the discussion in \Cref{sec:proofoutline}, 
GMC-GRPO removes only the amount by which the group average exceeds $1+\eps$, whereas TIC-GRPO removes the full contribution of every trajectory above the threshold. The KL divergence between the current and behavior policies bounds
these two quantities with denominators $\log(1+2\eps)$ and $\ceps$, respectively, producing the coefficient difference above.
\end{remark}

Finally, the delay term in Theorem~\ref{thm:tic-main} does not improve with the group size, while GMC-GRPO benefits from larger groups under local overlap (Remark~\ref{rem:group-size-gain}). Section~\ref{sec:experiments} examines whether these theoretical gaps translate into faster training in practice.

\begin{figure}[t]
\centering
\includegraphics[width=0.95\linewidth]{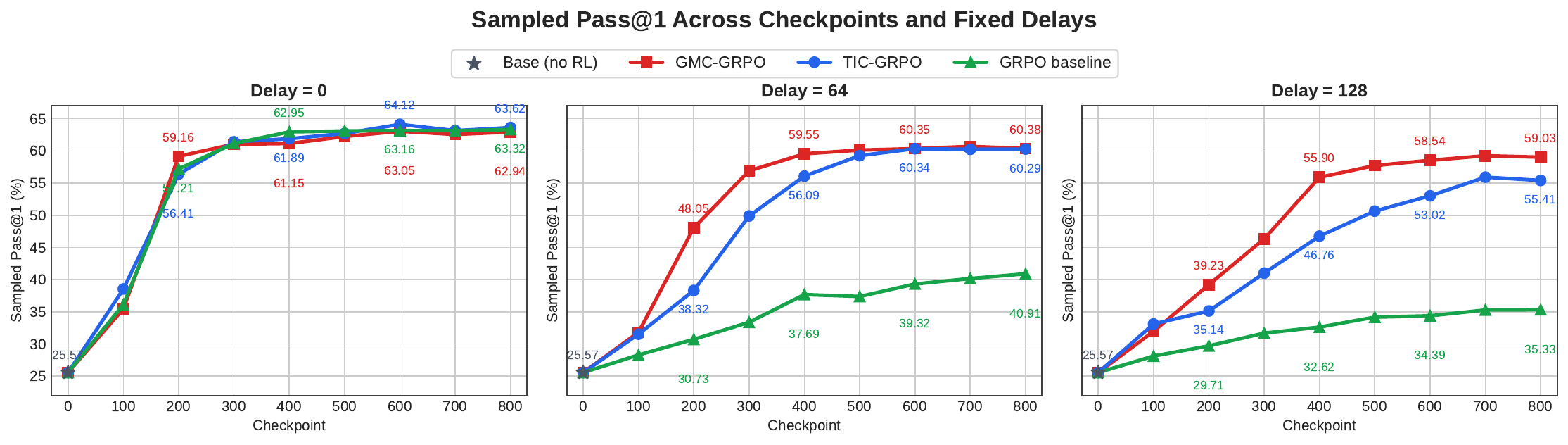}
\caption{Sampled Pass@1 on MATH500 across training checkpoints under fixed
rollout delays. GMC-GRPO remains robust as staleness increases,
substantially outperforming vanilla GRPO and competitive with or
exceeding TIC-GRPO.}
\label{fig:fixed-delay}
\end{figure}

\section{Experiments}\label{sec:experiments}
\vspace{-0.4em}

In this section, we evaluate GMC-GRPO against several baselines: (i) vanilla GRPO
\citep{shao2024deepseekmath}; (ii) GAC, which controls stale-aligned update
directions through gradient projection \citep{xu2026gac}; (iii) VCPO, which
combines effective-sample-size-guided step-size scaling with an off-policy
minimum-variance baseline \citep{huang2026stable}; and (iv) TIC-GRPO
\citep{pang2025tic}. Among these methods, GAC and VCPO are designed specifically for asynchronous training. We first evaluate robustness under controlled fixed delays
and then construct a simulated asynchronous system in which rollout delay is
correlated with response length. Detailed parameter settings and implementation
details are deferred to Appendix~\ref{sec:experiment-setup}.
We also discuss related approaches that selectively clip a fraction of importance ratios
\citep{fu2026areal, li20253po} and report an additional comparison with A-3PO \citep{li20253po}
in Appendix~\ref{additional-experiment-results}.

\textbf{Controlled Fixed Delays.}
To isolate the effect of staleness, we evaluate GRPO, TIC-GRPO, and GMC-GRPO
on Qwen3-1.7B under fixed rollout delays. Figure~\ref{fig:fixed-delay} reports
sampled Pass@1 on MATH500 for delays of $D=0,64,128$ learner updates; see
Appendix~\ref{sec:experiment-setup}.
The full sweep over $D\in\{0,32,64,128,256,512\}$ is deferred to
Appendix~\ref{additional-experiment-results}.
Without delay, all three methods attain similar final performance. At delay 64,
GMC-GRPO learns faster than TIC-GRPO and reaches comparable final Pass@1,
with both outperforming vanilla GRPO. At delay 128, GMC-GRPO
achieves 59.03\% Pass@1, compared with 55.41\% for TIC-GRPO
and 35.33\% for vanilla GRPO. These results show that 
GMC supports faster learning and stronger final performance as staleness increases.

\textbf{Simulated Asynchronous Delays.}
The fixed-delay study isolates staleness, but real rollout workers finish at
different times: longer responses typically take longer to generate and thus
reach the learner under older policies. To capture this coupling, we simulate an
asynchronous pipeline in which rollout delay depends on generation length,
producing heterogeneous delays averaging approximately 160 learner updates for
Qwen3-1.7B and 140 for Qwen3-4B and Qwen3-8B. This setting tests whether
GMC-GRPO remains robust when staleness varies across rollouts.
Among the methods reported in Table~\ref{tab:async-transfer}, GMC-GRPO achieves
the highest mean avg@$n$ scores across all three model sizes
(41.47, 56.20, and 59.62) and the highest mean greedy scores on Qwen3-1.7B and
Qwen3-8B (48.77 and 64.37). On Qwen3-4B, its mean greedy score of 60.63 is
0.83 points below TIC-GRPO. Notably, GMC-GRPO also outperforms the two of the most recent asynchronous-RL stabilization methods: GAC, which controls stale-aligned update directions via gradient projection, and VCPO, which combines ESS-guided step-size scaling with a minimum-variance baseline. Both improve over vanilla GRPO under delay but remain below TIC-GRPO-style trajectory correction at 4B and 8B. These results indicate that the robustness of
group-mass capping is not restricted to fixed delays, but holds under heterogeneous, length-dependent staleness. 
 
\begin{table}[t]
\centering
\small
\setlength{\tabcolsep}{3.5pt}
\renewcommand{\arraystretch}{1.25}
\caption{Performance under simulated asynchronous delays. Each benchmark cell is greedy/avg@$n$, with $n=16$ for MATH500, Minerva, and AMC23, and $n=8$ for GSM8K, Gaokao, and OlympiadBench. \emph{Avg. perf.} is the average over the six benchmarks, reported as mean greedy/mean avg@$n$.}
\label{tab:async-transfer}
\resizebox{\linewidth}{!}{%
\begin{tabular}{lcccccccc}
\toprule
\textbf{Method}
& \textbf{Avg. Delay}
& \textbf{MATH500}
& \textbf{GSM8K}
& \textbf{Gaokao}
& \textbf{Olympiad}
& \textbf{Minerva}
& \textbf{AMC23}
& \textbf{Avg. perf.} \\
\midrule
\multicolumn{9}{c}{\textit{Qwen3-1.7B}} \\
\midrule
GRPO
& 158.70
& 48.20 \,/\, 31.04
& 72.18 \,/\, 47.90
& 39.74 \,/\, 26.27
& 14.69 \,/\, 6.77
& 16.18 \,/\, 8.32
& 22.50 \,/\, 9.53
& 35.58 \,/\, 21.64 \\
GAC
& 160.02
& 47.00 \,/\, 31.03
& 72.18 \,/\, 47.09
& 40.26 \,/\, 25.16
& 15.73 \,/\, 6.70
& 15.81 \,/\, 8.18
& 20.00 \,/\, 11.09
& 35.16 \,/\, 21.54 \\
VCPO
& 157.87
& 65.20 \,/\, 54.55
& 80.21 \,/\, 67.76
& 55.06 \,/\, 45.06
& 27.60 \,/\, 20.49
& 26.47 \,/\, 18.31
& 27.50 \,/\, 26.72
& 47.01 \,/\, 38.82 \\
TIC-GRPO
& 159.60
& 62.20 \,/\, 53.61
& \textbf{82.34} \,/\, 67.14
& 53.25 \,/\, 43.02
& 27.60 \,/\, 19.51
& \textbf{26.84} \,/\, 17.53
& 32.50 \,/\, 26.25
& 47.45 \,/\, 37.84 \\
\ours \textbf{GMC-GRPO}
& 159.75
& \textbf{66.00} \,/\, \textbf{57.29}
& 81.65 \,/\, \textbf{72.66}
& \textbf{55.32} \,/\, \textbf{47.73}
& \textbf{28.19} \,/\, \textbf{22.61}
& 26.47 \,/\, \textbf{20.11}
& \textbf{35.00} \,/\, \textbf{28.44}
& \textbf{48.77} \,/\, \textbf{41.47} \\
\midrule
\multicolumn{9}{c}{\textit{Qwen3-4B}} \\
\midrule
GRPO
& 137.92
& 57.40 \,/\, 43.26
& 74.30 \,/\, 51.56
& 48.57 \,/\, 34.35
& 22.11 \,/\, 9.66
& 18.01 \,/\, 11.99
& 35.00 \,/\, 17.19
& 42.57 \,/\, 28.00 \\
GAC
& 138.85
& 58.00 \,/\, 42.39
& 77.79 \,/\, 55.32
& 50.39 \,/\, 33.64
& 21.81 \,/\, 10.55
& 16.18 \,/\, 12.22
& 35.00 \,/\, 18.91
& 43.19 \,/\, 28.84 \\
VCPO
& 139.50
& 61.00 \,/\, 48.39
& 76.88 \,/\, 64.30
& 46.75 \,/\, 36.30
& 29.38 \,/\, 13.61
& 26.84 \,/\, 16.98
& 37.50 \,/\, 21.88
& 46.39 \,/\, 33.58 \\
TIC-GRPO
& 138.87
& 76.40 \,/\, 68.31
& 90.75 \,/\, 81.70
& \textbf{65.19} \,/\, 55.68
& 42.88 \,/\, 33.33
& \textbf{36.03} \,/\, 27.55
& \textbf{57.50} \,/\, 43.91
& \textbf{61.46} \,/\, 51.75 \\
\ours \textbf{GMC-GRPO}
& 141.01
& \textbf{77.60} \,/\, \textbf{72.54}
& \textbf{91.13} \,/\, \textbf{87.03}
& \textbf{65.19} \,/\, \textbf{60.91}
& \textbf{43.18} \,/\, \textbf{37.72}
& 34.19 \,/\, \textbf{30.10}
& 52.50 \,/\, \textbf{48.91}
& 60.63 \,/\, \textbf{56.20} \\
\midrule
\multicolumn{9}{c}{\textit{Qwen3-8B}} \\
\midrule
GRPO 
& 140.90
& 61.80 \,/\, 49.41
& 88.70 \,/\, 72.74
& 55.06 \,/\, 38.83
& 27.30 \,/\, 16.49
& 21.69 \,/\, 16.96
& 25.00 \,/\, 22.03
& 46.59 \,/\, 36.08 \\
GAC
& 141.7
& 61.20 \,/\, 48.34
& 88.10 \,/\, 72.09
& 54.55 \,/\, 39.29
& 26.41 \,/\, 16.26
& 22.79 \,/\, 15.99
& 27.50 \,/\, 21.41
& 46.76 \,/\, 35.56 \\
VCPO
& 142.7
& 67.40 \,/\, 58.88
& 88.93 \,/\, 78.79
& 57.40 \,/\, 47.50
& 33.38 \,/\, 21.88
& 31.25 \,/\, 23.09
& 47.50 \,/\, 31.56
& 54.31 \,/\, 43.62 \\
TIC-GRPO
& 141.6
& \textbf{79.20} \,/\, 75.14
& 92.87 \,/\, 87.92
& 68.31 \,/\, 61.20
& \textbf{45.10} \,/\, 38.41
& \textbf{40.07} \,/\, \textbf{34.81}
& \textbf{60.00} \,/\, 50.47
& 64.26 \,/\, 57.99 \\
\ours \textbf{GMC-GRPO}
& 141.1
& 78.80 \,/\, \textbf{75.22}
& \textbf{93.33} \,/\, \textbf{91.37}
& \textbf{69.09} \,/\, \textbf{63.64}
& 44.96 \,/\, \textbf{40.34}
& \textbf{40.07} \,/\, 34.35
& \textbf{60.00} \,/\, \textbf{52.81}
& \textbf{64.37} \,/\, \textbf{59.62} \\
\bottomrule
\end{tabular}%
}
\end{table}

\section{Conclusions and Limitations}
\vspace{-0.4em}

This work studies how rollout staleness affects the convergence of GRPO-style algorithms in asynchronous LLM training and introduces GMC-GRPO to improve learning from delayed rollouts. Our theory and experiments show that group-level capping improves robustness to stale rollouts while maintaining performance on fresh data. Within the proposed class of weighted gradient estimators, our analysis provides a foundation for understanding and improving delay tolerance. Future work can extend this framework to more general weighting rules, such as those that use information beyond importance ratios, to further improve the balance between second-moment control and bias.

\bibliography{iclr2027_conference}

\begin{thebibliography}{52}
\providecommand{\natexlab}[1]{#1}
\providecommand{\url}[1]{\texttt{#1}}
\expandafter\ifx\csname urlstyle\endcsname\relax
  \providecommand{\doi}[1]{doi: #1}\else
  \providecommand{\doi}{doi: \begingroup \urlstyle{rm}\Url}\fi

\bibitem[Arnal et~al.(2026)Arnal, Cabannes, Cohen, Kempe, and Munos]{arnal2026efficient}
Charles Arnal, Vivien Cabannes, Taco Cohen, Julia Kempe, and Remi Munos.
\newblock Efficient rl training for llms with experience replay.
\newblock \emph{arXiv preprint arXiv:2604.08706}, 2026.

\bibitem[{Art of Problem Solving}(2024)]{aops-amc}
{Art of Problem Solving}.
\newblock {AMC} problems and solutions.
\newblock \url{https://artofproblemsolving.com/wiki/index.php/AMC_Problems_and_Solutions}, 2024.
\newblock Accessed: 2026-08-29.

\bibitem[Bartoldson et~al.(2026)Bartoldson, Venkatraman, Diffenderfer, Jain, Ben-Nun, Lee, Kim, Obando~Ceron, Bengio, and Kailkhura]{bartoldson2026trajectory}
Brian Bartoldson, Siddarth Venkatraman, James Diffenderfer, Moksh Jain, Tal Ben-Nun, Seanie Lee, Minsu Kim, Johan Obando~Ceron, Yoshua Bengio, and Bhavya Kailkhura.
\newblock Trajectory balance with asynchrony: Decoupling exploration and learning for fast, scalable llm post-training.
\newblock \emph{Advances in Neural Information Processing Systems}, 38:\penalty0 113901--113931, 2026.

\bibitem[Chen et~al.(2025)Chen, Li, Gong, Jiang, Fei, Yang, Shan, Yu, Wang, Zhu, et~al.]{chen2025minimax}
Aili Chen, Aonian Li, Bangwei Gong, Binyang Jiang, Bo~Fei, Bo~Yang, Boji Shan, Changqing Yu, Chao Wang, Cheng Zhu, et~al.
\newblock Minimax-m1: Scaling test-time compute efficiently with lightning attention.
\newblock \emph{arXiv preprint arXiv:2506.13585}, 2025.

\bibitem[Cobbe et~al.(2021)Cobbe, Kosaraju, Bavarian, Chen, Jun, Kaiser, Plappert, Tworek, Hilton, Nakano, et~al.]{cobbe2021training}
Karl Cobbe, Vineet Kosaraju, Mohammad Bavarian, Mark Chen, Heewoo Jun, Lukasz Kaiser, Matthias Plappert, Jerry Tworek, Jacob Hilton, Reiichiro Nakano, et~al.
\newblock Training verifiers to solve math word problems.
\newblock \emph{arXiv preprint arXiv:2110.14168}, 2021.

\bibitem[Dettmers et~al.(2021)Dettmers, Lewis, Shleifer, and Zettlemoyer]{dettmers20218}
Tim Dettmers, Mike Lewis, Sam Shleifer, and Luke Zettlemoyer.
\newblock 8-bit optimizers via block-wise quantization.
\newblock \emph{arXiv preprint arXiv:2110.02861}, 2021.

\bibitem[Du et~al.(2026)Du, Luo, Zheng, Yuan, Xia, Shi, Jin, He, Zou, Liang, et~al.]{du2026pact}
Zhenbang Du, Jun Luo, Zhiwei Zheng, Xiangchi Yuan, Kejing Xia, Dachuan Shi, Qirui Jin, Qijia He, Shaofeng Zou, Yingbin Liang, et~al.
\newblock Pact: Privileged trace co-training for multi-turn tool-use agents.
\newblock \emph{arXiv preprint arXiv:2606.16215}, 2026.

\bibitem[Dwyer et~al.(2025)Dwyer, Sobey, and Chapman]{dwyer2025s}
Madeleine Dwyer, Adam Sobey, and Adriane Chapman.
\newblock It's not you, it's clipping: A soft trust-region via probability smoothing for llm rl.
\newblock \emph{arXiv preprint arXiv:2509.21282}, 2025.

\bibitem[Foster et~al.(2025)Foster, Mhammedi, and Rohatgi]{foster2025good}
Dylan~J Foster, Zakaria Mhammedi, and Dhruv Rohatgi.
\newblock Is a good foundation necessary for efficient reinforcement learning? the computational role of the base model in exploration.
\newblock \emph{arXiv preprint arXiv:2503.07453}, 2025.

\bibitem[Fu et~al.(2026)Fu, Gao, Shen, Zhu, Mei, He, Xu, Wei, Mei, Wang, et~al.]{fu2026areal}
Wei Fu, Jiaxuan Gao, Xujie Shen, Chen Zhu, Zhiyu Mei, Chuyi He, Shusheng Xu, Guo Wei, Jun Mei, Jiashu Wang, et~al.
\newblock Areal: A large-scale asynchronous reinforcement learning system for language reasoning.
\newblock \emph{Advances in Neural Information Processing Systems}, 38:\penalty0 36256--36282, 2026.

\bibitem[Gao et~al.(2025)Gao, Zheng, Chen, Dang, Liu, Yu, Yang, Bai, Zhou, and Lin]{gao2025soft}
Chang Gao, Chujie Zheng, Xiong-Hui Chen, Kai Dang, Shixuan Liu, Bowen Yu, An~Yang, Shuai Bai, Jingren Zhou, and Junyang Lin.
\newblock Soft adaptive policy optimization.
\newblock \emph{arXiv preprint arXiv:2511.20347}, 2025.

\bibitem[Guo et~al.(2025)Guo, Yang, Zhang, Song, Wang, Zhu, Xu, Zhang, Ma, Bi, et~al.]{guo2025deepseek}
Daya Guo, Dejian Yang, Haowei Zhang, Junxiao Song, Peiyi Wang, Qihao Zhu, Runxin Xu, Ruoyu Zhang, Shirong Ma, Xiao Bi, et~al.
\newblock Deepseek-r1: Incentivizing reasoning capability in llms via reinforcement learning.
\newblock \emph{arXiv preprint arXiv:2501.12948}, 2025.

\bibitem[He et~al.(2024)He, Luo, Bai, Hu, Thai, Shen, Hu, Han, Huang, Zhang, et~al.]{he2024olympiadbench}
Chaoqun He, Renjie Luo, Yuzhuo Bai, Shengding Hu, Zhen Thai, Junhao Shen, Jinyi Hu, Xu~Han, Yujie Huang, Yuxiang Zhang, et~al.
\newblock Olympiadbench: A challenging benchmark for promoting agi with olympiad-level bilingual multimodal scientific problems.
\newblock In \emph{Proceedings of the 62nd Annual Meeting of the Association for Computational Linguistics (Volume 1: Long Papers)}, pp.\  3828--3850, 2024.

\bibitem[Hendrycks et~al.(2021)Hendrycks, Burns, Kadavath, Arora, Basart, Tang, Song, and Steinhardt]{hendrycks2021measuring}
Dan Hendrycks, Collin Burns, Saurav Kadavath, Akul Arora, Steven Basart, Eric Tang, Dawn Song, and Jacob Steinhardt.
\newblock Measuring mathematical problem solving with the math dataset.
\newblock \emph{arXiv preprint arXiv:2103.03874}, 2021.

\bibitem[Hu et~al.(2026)Hu, Liu, Miao, Xiao, Zang, Zheng, Huang, Ding, Zhang, Yang, et~al.]{hu2026dora}
Tianhao Hu, Xiangcheng Liu, Yuchun Miao, Youshao Xiao, Hongyu Zang, Yang Zheng, Xuan Huang, Jinrui Ding, Yufei Zhang, Yu~Yang, et~al.
\newblock Dora: A scalable asynchronous reinforcement learning system for language model training.
\newblock \emph{arXiv preprint arXiv:2604.26256}, 2026.

\bibitem[Huang et~al.(2025{\natexlab{a}})Huang, Block, Foster, Rohatgi, Zhang, Simchowitz, Ash, and Krishnamurthy]{huang2025self}
Audrey Huang, Adam Block, Dylan Foster, Dhruv Rohatgi, Cyril Zhang, Max Simchowitz, Jordan Ash, and Akshay Krishnamurthy.
\newblock Self-improvement in language models: The sharpening mechanism.
\newblock In \emph{International Conference on Learning Representations}, volume 2025, pp.\  76687--76739, 2025{\natexlab{a}}.

\bibitem[Huang et~al.(2025{\natexlab{b}})Huang, Zhan, Xie, Lee, Sun, Krishnamurthy, and Foster]{huang2025correcting}
Audrey Huang, Wenhao Zhan, Tengyang Xie, Jason Lee, Wen Sun, Akshay Krishnamurthy, and Dylan Foster.
\newblock Correcting the mythos of kl-regularization: Direct alignment without overoptimization via chi-squared preference optimization.
\newblock In \emph{International Conference on Learning Representations}, volume 2025, pp.\  92647--92701, 2025{\natexlab{b}}.

\bibitem[Huang et~al.(2026)Huang, Zhang, Hu, Yang, and Han]{huang2026stable}
Luke~J Huang, Zhuoyang Zhang, Qinghao Hu, Shang Yang, and Song Han.
\newblock Stable asynchrony: Variance-controlled off-policy rl for llms.
\newblock \emph{arXiv preprint arXiv:2602.17616}, 2026.

\bibitem[Jaech et~al.(2024)Jaech, Kalai, Lerer, Richardson, El-Kishky, Low, Helyar, Madry, Beutel, Carney, et~al.]{jaech2024openai}
Aaron Jaech, Adam Kalai, Adam Lerer, Adam Richardson, Ahmed El-Kishky, Aiden Low, Alec Helyar, Aleksander Madry, Alex Beutel, Alex Carney, et~al.
\newblock Openai o1 system card.
\newblock \emph{arXiv preprint arXiv:2412.16720}, 2024.

\bibitem[Jain et~al.(2025)Jain, Gu, Li, Yan, Zhang, Wang, Solar-Lezama, Sen, and Stoica]{jain2025livecodebench}
Naman Jain, Alex Gu, Wen-Ding Li, Fanjia Yan, Tianjun Zhang, Sida Wang, Armando Solar-Lezama, Koushik Sen, and Ion Stoica.
\newblock Livecodebench: Holistic and contamination free evaluation of large language models for code.
\newblock In \emph{International Conference on Learning Representations}, volume 2025, pp.\  58791--58831, 2025.

\bibitem[Joulani et~al.(2013)Joulani, Gyorgy, and Szepesv{\'a}ri]{joulani2013online}
Pooria Joulani, Andras Gyorgy, and Csaba Szepesv{\'a}ri.
\newblock Online learning under delayed feedback.
\newblock In \emph{International conference on machine learning}, pp.\  1453--1461. PMLR, 2013.

\bibitem[Koloskova et~al.(2022)Koloskova, Stich, and Jaggi]{koloskova2022sharper}
Anastasiia Koloskova, Sebastian~U Stich, and Martin Jaggi.
\newblock Sharper convergence guarantees for asynchronous sgd for distributed and federated learning.
\newblock \emph{Advances in Neural Information Processing Systems}, 35:\penalty0 17202--17215, 2022.

\bibitem[Kwon et~al.(2023)Kwon, Li, Zhuang, Sheng, Zheng, Yu, Gonzalez, Zhang, and Stoica]{kwon2023efficient}
Woosuk Kwon, Zhuohan Li, Siyuan Zhuang, Ying Sheng, Lianmin Zheng, Cody~Hao Yu, Joseph Gonzalez, Hao Zhang, and Ion Stoica.
\newblock Efficient memory management for large language model serving with pagedattention.
\newblock In \emph{Proceedings of the 29th symposium on operating systems principles}, pp.\  611--626, 2023.

\bibitem[Lambert et~al.(2024)Lambert, Morrison, Pyatkin, Huang, Ivison, Brahman, Miranda, Liu, Dziri, Lyu, et~al.]{lambert2024tulu}
Nathan Lambert, Jacob Morrison, Valentina Pyatkin, Shengyi Huang, Hamish Ivison, Faeze Brahman, Lester James~V Miranda, Alisa Liu, Nouha Dziri, Shane Lyu, et~al.
\newblock Tulu 3: Pushing frontiers in open language model post-training.
\newblock \emph{arXiv preprint arXiv:2411.15124}, 2024.

\bibitem[Lewkowycz et~al.(2022)Lewkowycz, Andreassen, Dohan, Dyer, Michalewski, Ramasesh, Slone, Anil, Schlag, Gutman-Solo, et~al.]{lewkowycz2022solving}
Aitor Lewkowycz, Anders Andreassen, David Dohan, Ethan Dyer, Henryk Michalewski, Vinay Ramasesh, Ambrose Slone, Cem Anil, Imanol Schlag, Theo Gutman-Solo, et~al.
\newblock Solving quantitative reasoning problems with language models.
\newblock \emph{Advances in neural information processing systems}, 35:\penalty0 3843--3857, 2022.

\bibitem[Li et~al.(2025)Li, Wu, and Shen]{li20253po}
Xiaocan Li, Shiliang Wu, and Zheng Shen.
\newblock A-3po: Accelerating asynchronous llm training with staleness-aware proximal policy approximation.
\newblock \emph{arXiv preprint arXiv:2512.06547}, 2025.

\bibitem[Liao et~al.(2024)Liao, Li, Luo, Jing, and Fan]{liao2024mario}
Minpeng Liao, Chengxi Li, Wei Luo, Wu~Jing, and Kai Fan.
\newblock Mario: Math reasoning with code interpreter output-a reproducible pipeline.
\newblock In \emph{Findings of the Association for Computational Linguistics: ACL 2024}, pp.\  905--924, 2024.

\bibitem[Lightman et~al.(2024)Lightman, Kosaraju, Burda, Edwards, Baker, Lee, Leike, Schulman, Sutskever, and Cobbe]{lightman2024let}
Hunter Lightman, Vineet Kosaraju, Yuri Burda, Harrison Edwards, Bowen Baker, Teddy Lee, Jan Leike, John Schulman, Ilya Sutskever, and Karl Cobbe.
\newblock Let's verify step by step.
\newblock In \emph{International Conference on Learning Representations}, volume 2024, pp.\  39578--39601, 2024.

\bibitem[Liu et~al.(2025)Liu, Chen, Li, Qi, Pang, Du, Lee, and Lin]{liu2025understanding}
Zichen Liu, Changyu Chen, Wenjun Li, Penghui Qi, Tianyu Pang, Chao Du, Wee~Sun Lee, and Min Lin.
\newblock Understanding r1-zero-like training: A critical perspective.
\newblock \emph{arXiv preprint arXiv:2503.20783}, 2025.

\bibitem[Maranjyan et~al.(2025)Maranjyan, Tyurin, and Richt{\'a}rik]{maranjyan2025ringmaster}
Artavazd Maranjyan, Alexander Tyurin, and Peter Richt{\'a}rik.
\newblock Ringmaster asgd: The first asynchronous sgd with optimal time complexity.
\newblock \emph{arXiv preprint arXiv:2501.16168}, 2025.

\bibitem[Mishchenko et~al.(2022)Mishchenko, Bach, Even, and Woodworth]{mishchenko2022asynchronous}
Konstantin Mishchenko, Francis Bach, Mathieu Even, and Blake~E Woodworth.
\newblock Asynchronous sgd beats minibatch sgd under arbitrary delays.
\newblock \emph{Advances in Neural Information Processing Systems}, 35:\penalty0 420--433, 2022.

\bibitem[Noukhovitch et~al.(2025)Noukhovitch, Huang, Xhonneux, Hosseini, Agarwal, and Courville]{noukhovitch2025asynchronous}
Michael Noukhovitch, Shengyi Huang, Sophie Xhonneux, Arian Hosseini, Rishabh Agarwal, and Aaron Courville.
\newblock Asynchronous rlhf: Faster and more efficient off-policy rl for language models.
\newblock In \emph{International Conference on Learning Representations}, volume 2025, pp.\  4003--4029, 2025.

\bibitem[Ouyang et~al.(2022)Ouyang, Wu, Jiang, Almeida, Wainwright, Mishkin, Zhang, Agarwal, Slama, Ray, et~al.]{ouyang2022training}
Long Ouyang, Jeffrey Wu, Xu~Jiang, Diogo Almeida, Carroll Wainwright, Pamela Mishkin, Chong Zhang, Sandhini Agarwal, Katarina Slama, Alex Ray, et~al.
\newblock Training language models to follow instructions with human feedback.
\newblock \emph{Advances in neural information processing systems}, 35:\penalty0 27730--27744, 2022.

\bibitem[Pang et~al.(2025)Pang, Luo, and Jin]{pang2025tic}
Lei Pang, Jun Luo, and Ruinan Jin.
\newblock Tic-grpo: Provable and efficient optimization for reinforcement learning from human feedback.
\newblock \emph{arXiv preprint arXiv:2508.02833}, 2025.

\bibitem[Ritter et~al.(2026)Ritter, Oertell, Guo, Chang, Brantley, and Sun]{ritter2026llms}
Daniel Ritter, Owen Oertell, Bradley Guo, Jonathan Chang, Kiant{\'e} Brantley, and Wen Sun.
\newblock Llms can learn to reason via off-policy rl.
\newblock \emph{arXiv preprint arXiv:2602.19362}, 2026.

\bibitem[Roux et~al.(2025)Roux, Bellemare, Lebensold, Bergeron, Greaves, Fr{\'e}chette, Pelletier, Thibodeau-Laufer, Toth, and Work]{roux2025tapered}
Nicolas~Le Roux, Marc~G Bellemare, Jonathan Lebensold, Arnaud Bergeron, Joshua Greaves, Alex Fr{\'e}chette, Carolyne Pelletier, Eric Thibodeau-Laufer, S{\'a}ndor Toth, and Sam Work.
\newblock Tapered off-policy reinforce: Stable and efficient reinforcement learning for llms.
\newblock \emph{arXiv preprint arXiv:2503.14286}, 2025.

\bibitem[Shao et~al.(2024)Shao, Wang, Zhu, Xu, Song, Bi, Zhang, Zhang, Li, Wu, et~al.]{shao2024deepseekmath}
Zhihong Shao, Peiyi Wang, Qihao Zhu, Runxin Xu, Junxiao Song, Xiao Bi, Haowei Zhang, Mingchuan Zhang, YK~Li, Yang Wu, et~al.
\newblock Deepseekmath: Pushing the limits of mathematical reasoning in open language models.
\newblock \emph{arXiv preprint arXiv:2402.03300}, 2024.

\bibitem[Sheng et~al.(2025)Sheng, Zhang, Ye, Wu, Zhang, Zhang, Peng, Lin, and Wu]{sheng2025hybridflow}
Guangming Sheng, Chi Zhang, Zilingfeng Ye, Xibin Wu, Wang Zhang, Ru~Zhang, Yanghua Peng, Haibin Lin, and Chuan Wu.
\newblock Hybridflow: A flexible and efficient rlhf framework.
\newblock In \emph{Proceedings of the Twentieth European Conference on Computer Systems}, pp.\  1279--1297, 2025.

\bibitem[Tan et~al.(2023)Tan, Zhu, and Zhang]{tan2023abs}
Qiao Tan, Feng Zhu, and Jingjing Zhang.
\newblock Abs: Adaptive bounded staleness converges faster and communicates less.
\newblock \emph{arXiv preprint arXiv:2301.08895}, 2023.

\bibitem[Tian et~al.(2026)Tian, Xie, and Wei]{tian2026off}
Minghao Tian, Yunfei Xie, and Chen Wei.
\newblock How off-policy can grpo be? mu-grpo for efficient llm reinforcement learning.
\newblock \emph{arXiv preprint arXiv:2605.17570}, 2026.

\bibitem[Tyurin et~al.(2026)Tyurin, Spiridonov, and Rudenko]{tyurin2026asynchronous}
Alexander Tyurin, Andrei Spiridonov, and Varvara Rudenko.
\newblock Asynchronous policy gradient aggregation for efficient distributed reinforcement learning.
\newblock In \emph{International Conference on Learning Representations}, volume 2026, pp.\  8826--8857, 2026.

\bibitem[Wang et~al.(2025)Wang, Xie, Wang, Huang, Li, and Wang]{wang2025infinite}
Liangyu Wang, Huanyi Xie, Xinhai Wang, Tianjin Huang, Mengdi Li, and Di~Wang.
\newblock Infinite sampling: Efficient and stable grouped rl training for large language models.
\newblock \emph{arXiv preprint arXiv:2506.22950}, 2025.

\bibitem[Xi et~al.(2026)Xi, Guo, Nan, Zhou, Shen, Chen, Liu, Huang, Deng, Zhang, et~al.]{xi2026bapo}
Zhiheng Xi, Xin Guo, Yang Nan, Enyu Zhou, Junrui Shen, Wenxiang Chen, Jiaqi Liu, Jixuan Huang, Xun Deng, Zhihao Zhang, et~al.
\newblock Bapo: Stabilizing off-policy reinforcement learning for llms via balanced policy optimization with adaptive clipping.
\newblock In \emph{International Conference on Learning Representations}, volume 2026, pp.\  126204--126228, 2026.

\bibitem[Xie et~al.(2025)Xie, Foster, Krishnamurthy, Rosset, Awadallah, and Rakhlin]{xie2025exploratory}
Tengyang Xie, Dylan Foster, Akshay Krishnamurthy, Corby Rosset, Ahmed~H Awadallah, and Alexander Rakhlin.
\newblock Exploratory preference optimization: Harnessing implicit q*-approximation for sample-efficient rlhf.
\newblock In \emph{International Conference on Learning Representations}, volume 2025, pp.\  43632--43669, 2025.

\bibitem[Xu et~al.(2026)Xu, Su, Tian, Diao, Qian, and Wu]{xu2026gac}
Haofeng Xu, Junwei Su, Yukun Tian, Lansong Diao, Zhengping Qian, and Chuan Wu.
\newblock Gac: Stabilizing asynchronous rl training for llms via gradient alignment control.
\newblock \emph{arXiv preprint arXiv:2603.01501}, 2026.

\bibitem[Yang et~al.(2025)Yang, Li, Yang, Zhang, Hui, Zheng, Yu, Gao, Huang, Lv, et~al.]{yang2025qwen3}
An~Yang, Anfeng Li, Baosong Yang, Beichen Zhang, Binyuan Hui, Bo~Zheng, Bowen Yu, Chang Gao, Chengen Huang, Chenxu Lv, et~al.
\newblock Qwen3 technical report.
\newblock \emph{arXiv preprint arXiv:2505.09388}, 2025.

\bibitem[Yu et~al.(2026)Yu, Zhang, Zhu, Yuan, Zuo, Yue, Dai, Fan, Liu, Liu, et~al.]{yu2026dapo}
Qiying Yu, Zheng Zhang, Ruofei Zhu, Yufeng Yuan, Xiaochen Zuo, Yu~Yue, Weinan Dai, Tiantian Fan, Gaohong Liu, Lingjun Liu, et~al.
\newblock Dapo: An open-source llm reinforcement learning system at scale.
\newblock \emph{Advances in Neural Information Processing Systems}, 38:\penalty0 113222--113244, 2026.

\bibitem[Zeng et~al.(2026)Zeng, Lv, Hou, Du, Zheng, Chen, Yin, Ge, Huang, Xie, et~al.]{zeng2026glm}
Aohan Zeng, Xin Lv, Zhenyu Hou, Zhengxiao Du, Qinkai Zheng, Bin Chen, Da~Yin, Chendi Ge, Chenghua Huang, Chengxing Xie, et~al.
\newblock Glm-5: from vibe coding to agentic engineering.
\newblock \emph{arXiv preprint arXiv:2602.15763}, 2026.

\bibitem[Zhao et~al.(2026)Zhao, Xie, Zheng, Gong, Lu, Yang, Hu, and Chen]{zhao2026deconstructing}
Guanqun Zhao, Zijun Xie, Binbin Zheng, Enlei Gong, Jiafeng Lu, Yehan Yang, Aoqi Hu, and Zeyu Chen.
\newblock Deconstructing off-policy ratios: Entropy-scaled trust regions for asynchronous reinforcement learning.
\newblock \emph{arXiv preprint arXiv:2607.22186}, 2026.

\bibitem[Zheng et~al.(2026)Zheng, Zhao, and Chen]{zheng2026prosperity}
Haizhong Zheng, Jiawei Zhao, and Beidi Chen.
\newblock Prosperity before collapse: How far can off-policy rl reach with stale data on llms?
\newblock In \emph{International Conference on Learning Representations}, volume 2026, pp.\  82657--82679, 2026.

\bibitem[Zhong et~al.(2024)Zhong, Zhang, Wu, Liu, Chen, Wan, Hu, Xia, Ming, Zhu, et~al.]{zhong2024rlhfuse}
Yinmin Zhong, Zili Zhang, Bingyang Wu, Shengyu Liu, Yukun Chen, Changyi Wan, Hanpeng Hu, Lei Xia, Ranchen Ming, Yibo Zhu, et~al.
\newblock Rlhfuse: Efficient rlhf training for large language models with inter-and intra-stage fusion.
\newblock \emph{arXiv preprint arXiv:2409.13221}, 1, 2024.

\bibitem[Zhou et~al.(2026)Zhou, Ye, Xu, Zhu, Yang, Gong, and Shi]{zhou2026demystifying}
Hongyi Zhou, Kai Ye, Erhan Xu, Jin Zhu, Ying Yang, Shijin Gong, and Chengchun Shi.
\newblock Demystifying group relative policy optimization: Its policy gradient is a u-statistic.
\newblock \emph{arXiv preprint arXiv:2603.01162}, 2026.

\end{thebibliography}
\bibliographystyle{iclr2027_conference}

\appendix

\section*{Appendix Contents}
\begin{enumerate}[
  label=\textbf{\Alph*.},
  leftmargin=2.5em,
  itemsep=1pt,
  topsep=3pt
]
  \item \hyperref[app:related work]{Related Work}\dotfill\pageref{app:related work}
  \item \hyperref[app:experiments]{Experiment Details and Supplemental Experiments}\dotfill\pageref{app:experiments}
  \begin{enumerate}[
    label=\textbf{\Alph{enumi}.\arabic*.},
    leftmargin=3.5em,
    itemsep=0pt,
    topsep=1pt
  ]
    \item \hyperref[sec:experiment-setup]{Detailed Experiment Setup}\dotfill\pageref{sec:experiment-setup}
    \item \hyperref[additional-experiment-results]{Supplemental Experiment Results}\dotfill\pageref{additional-experiment-results}
  \end{enumerate}
  \item \hyperref[sec:algorithm]{Mixed-Delay Setting, GMC-GRPO and TIC-GRPO Algorithms}\dotfill\pageref{sec:algorithm}
  \item \hyperref[sec:assumptions]{Technical Preliminaries}\dotfill\pageref{sec:assumptions}
  \item \hyperref[sec:weighted-class]{Optimization over a Second-Moment-Controlled Weighting Class}\dotfill\pageref{sec:weighted-class}
  \item \hyperref[sec:lemmas]{Convergence Analysis of GMC-GRPO}\dotfill\pageref{sec:lemmas}
    \item \hyperref[app:vanilla]{Convergence Analysis of TIC-GRPO}\dotfill\pageref{app:vanilla}
  \item \hyperref[sec:main-specializations]{From the Appendix Bounds to the Main-Text Rates}\dotfill\pageref{sec:main-specializations}
\end{enumerate}

\begin{table}[!ht]
\centering
\caption{Summary of the main notation used in the paper.}
\label{tab:notation}
\small
\renewcommand{\arraystretch}{1.12}
\begin{tabular}{@{}>{\raggedright\arraybackslash}p{0.15\textwidth}>{\raggedright\arraybackslash}p{0.76\textwidth}@{}}
\toprule
\textbf{Notation} & \textbf{Description} \\
\midrule
$\mathcal X,\cD,\cA$ & Prompt space, training prompt distribution, and token vocabulary. \\
$x,y,y_t,y_{<t}$ & Prompt, response, token at position $t$, and prefix $(y_1,\ldots,y_{t-1})$. \\
$T$ & Maximum response length; also the normalization factor in each gradient update. \\
$\theta,\pi_\theta(y\mid x)$ & Policy parameter in $\mathbb R^d$ and trajectory likelihood $\prod_{t=1}^T\pi_\theta(y_t\mid x,y_{<t})$. \\
$r_\phi(x,y),R$ & Verifier score and its upper bound: $r_\phi(x,y)\in[0,R]$. \\
$J_x(\theta),J(\theta),J^\star$ & Expected reward for prompt $x$, its average over $x\sim\cD$, and $\sup_\theta J(\theta)$; see \cref{eq:rl-objective}. \\
\midrule
$n,N,\eta$ & Learner update index $n\in\{0,\ldots,N-1\}$, $N$ total updates, and step size $\eta>0$. \\
$B,G,b,i$ & Groups per mini-batch, responses per group ($G\ge2$), group index $b\in[B]$, and response index $i\in[G]$. \\
$\cM_n$ & Mini-batch of $B$ completed prompt groups consumed at iteration $n$. \\
$\theta_n,\theta_{n-\tau_n^{(b)}}$ & Current learner parameter and behavior parameter used to generate group $b$. \\
$\tau_n^{(b)}$ & Number of learner updates between the launch and consumption of group $b$; $\tau_n^{(b)}\in\{0,\ldots,n\}$. \\
$\overline\tau_k$ & Averaged $k$th delay moment: $\frac1{N}\sum_{n=0}^{N-1}\E[\frac1B\sum_{b=1}^B(\tau_n^{(b)})^k]$, $k\ge1$. \\
$\widehat A^{(b,i)}$ & Leave-one-out advantage $r^{(b,i)}-\frac1{G-1}\sum_{j\ne i}r^{(b,j)}$, where $r^{(b,i)}=r_\phi(x^{(b)},y^{(b,i)})$. \\
\midrule
$\rho_n^{(b,i)}$ & Current-to-behavior trajectory likelihood ratio $\pi_{\theta_n}(y^{(b,i)}\mid x^{(b)})/\pi_{\theta_{n-\tau_n^{(b)}}}(y^{(b,i)}\mid x^{(b)})$. \\
$\overline\rho_n^{(b)}$ & Average raw importance ratio in group $b$: $G^{-1}\sum_{i=1}^G\rho_n^{(b,i)}$. \\
$\eps,\ceps$ & Threshold parameter $\eps>0$ (the cap is $1+\eps$), and $\ceps=\log(1+\eps)-\eps/(1+\eps)$. \\
$w_n^{(b,i)}$ & Correction weight, held fixed during differentiation,
satisfying $0\le w_n^{(b,i)}\le1$ and
$G^{-1}\sum_i w_n^{(b,i)}\rho_n^{(b,i)}\le1+\eps$. \\
$\alpha_n^{(b),\star}$ & Common GMC scaling factor $\min\{1,(1+\eps)/\overline\rho_n^{(b)}\}$. \\
$\widehat\rho_n^{(b,i)}$ & GMC effective importance ratio $\alpha_n^{(b),\star}\rho_n^{(b,i)}$. \\
$g_n^{\mathrm{GMC}},g_n^{\mathrm{TIC}},g_n^w$ & Update estimators for GMC, TIC, and a general weighted correction; see \cref{eq:gmcgrad-main}, \cref{eq:appvangrad}, and \cref{eq:weighted-update}. \\
\midrule
$\cF_n^-,\cH_n$ & Information before observing the returned mini-batch, and $\cH_n=\sigma(\cF_n^-,\{x^{(b)},\tau_n^{(b)}\}_{b=1}^B)$ after revealing its prompts and delays. \\
$b_n$ & Conditional update bias $\E[g_n\mid\cF_n^-]-T^{-1}\nabla J(\theta_n)$ for the estimator $g_n$ under study. \\
$L,\mathcal C$ & Token-score Lipschitz constant and token-coverage constant ($\pi_\theta(y_t\mid x,y_{<t})\ge1/\mathcal C$); see Assumptions~\ref{as:lip} and~\ref{as:coverage}. \\
$D_n^{(b)},\chi_n^{2,(b)}$ & Trajectory KL and chi-square divergences from $\pi_{\theta_n}(\cdot\mid x^{(b)})$ to $\pi_{\theta_{n-\tau_n^{(b)}}}(\cdot\mid x^{(b)})$; see Appendix~\ref{sec:algorithm}. \\
$\delta$ & Local-overlap bound: $|\rho_n^{(b,i)}-1|\le\delta$ almost surely, for the refinement in Corollary~\ref{cor:gmc-local-overlap}. \\
\bottomrule
\end{tabular}
\end{table}

\section{Related Work}\label{app:related work}
\textbf{Asynchronous RL Systems for LLM Post-Training.}
LLM RL pipelines typically alternate between autoregressive rollout generation
and policy optimization, incurring idle time from variable response lengths and
repeated switches between inference and training \citep{zhong2024rlhfuse,sheng2025hybridflow}. Recent work improves this
pipeline along two main directions. Asynchronous RLHF and LLM RL systems
decouple rollout generation from policy optimization so that the two stages can
proceed concurrently \citep{noukhovitch2025asynchronous,fu2026areal,hu2026dora}.
Complementary approaches improve rollout-side utilization through continuous
grouped sampling or scale exploration and learning through asynchronous replay
\citep{wang2025infinite,bartoldson2026trajectory}. This decoupling, however,
causes the learner to consume trajectories generated by lagged behavior
policies, making rollout staleness an intrinsic off-policy optimization issue.
Rather than proposing another systems architecture, we aim at finding robust algorithms against data staleness in asynchronous RL systems: we analyze GRPO-style optimization under
stochastic rollout delays and design an importance correction that controls the
resulting bias without sacrificing stability.

\textbf{Off-Policy Correction for Stale LLM Rollouts.}
Stale rollouts can induce heavy-tailed importance ratios, allowing a small number of trajectories to dominate the gradient estimate and destabilize training.
Learning from such data therefore requires a careful choice of how the update is corrected. Existing corrections differ mainly in where they intervene. A common strategy acts locally
on individual ratios: TIC-GRPO clips a trajectory once its ratio exceeds the
clipping threshold \citep{pang2025tic}, while related methods replace fixed
clipping with tapered weights, adaptive or asymmetric thresholds, and
entropy-scaled trust regions
\citep{roux2025tapered,xi2026bapo,zhao2026deconstructing, li20253po}. A second strategy
controls batch- or update-level statistics through second-moment constraints,
variance-aware step sizes, or gradient-alignment control
\citep{zheng2026prosperity,huang2026stable,xu2026gac}. A third changes how stale
experience is reused through staged training with selective filtering, replay
buffers, or objectives designed for lagged inference policies
\citep{tian2026off,arnal2026efficient,ritter2026llms}. Together, these methods
show that stale data can remain informative when its effect on the update is
controlled. Our work instead moves from individual correction to
group-level control. We identify a cap on corrected importance mass that
preserves gradient second-moment bound and derive a uniform rescaling rule
that reduces delay-induced bias without selectively clipping trajectories.

\textbf{Theory of GRPO-Style Optimization under Delay.}
The convergence theory of asynchronous optimization is well developed, with rigorous and tight convergence analysis that matches the lower bound under various settings \citep{joulani2013online,koloskova2022sharper,mishchenko2022asynchronous,tyurin2026asynchronous}.
In parallel, a growing line of work establishes theoretical guarantees for
RL-based LLM post-training \citep{xie2025exploratory,huang2025correcting,foster2025good,huang2025self}. Despite this progress, most existing RL-for-LLM theories are developed within the natural policy gradient and RLHF frameworks, and rigorous convergence guarantees for widely used GRPO-style algorithms remain
limited, particularly under asynchronous rollout generation.
\citet{zhou2026demystifying} characterize the
GRPO gradient as a U-statistic and analyze its finite-sample error and
dependence on group size, while \citet{pang2025tic} introduce TIC-GRPO and
provide the first rigorous convergence analysis of GRPO-style methods. More
recently, \citet{xu2026gac} analyze asynchronous GRPO through a generic
biased-gradient framework, but leave the staleness-induced bias and higher-order
remainder terms as abstract quantities in the final bound. \textbf{To the best of our knowledge, we are the first to provide an
explicit delay-dependent convergence analysis of GRPO-style algorithms in asynchronous RL systems.}

\section{Experiment Details and supplemental Experiments}
\label{app:experiments}

\subsection{Detailed Experiment Setup}\label{sec:experiment-setup}

\textbf{Models and training data.}
We run the experiments with the Qwen3-1.7B-Base, Qwen3-4B-Base and Qwen3-8B-Base models
\citep{yang2025qwen3}. All methods update the full actor; we use neither LoRA
adapters nor an explicit KL-divergence penalty. The training set is constructed
from the seven subject-specific training splits of MATH
\citep{hendrycks2021measuring}: algebra, counting and probability, geometry,
intermediate algebra, number theory, prealgebra, and precalculus. After removing
examples for which the reference answer cannot be extracted, the resulting
training set contains 7,498 problems. Training and evaluation use the same
two-shot prompt, which instructs the model to reason step by step and place its
final answer in \verb|\boxed{}|.

\textbf{Shared training configuration.}
We generate rollouts with vLLM \citep{kwon2023efficient} in BF16, using temperature
$1.0$ and top-$p=1.0$. Prompts and completions are truncated at 1,024 and 4,096
tokens, respectively. Each rollout mini-batch contains eight prompts and eight
responses per prompt, for a total of 64 trajectories. Except for the zero-delay
runs in the fixed-delay suite described below, we set
\texttt{inner\_steps=1}, so these 64 trajectories produce one actor update.

We optimize with 8-bit AdamW \citep{dettmers20218}, a learning rate
of $10^{-6}$, $\beta_1=0.9$, $\beta_2=0.999$, and zero weight decay. The learning
rate follows a cosine schedule with a 3\% warmup, and the gradient norm is
clipped at $1.0$. 
We use a mix of NVIDIA RTX 6000 Ada Generation GPUs and NVIDIA A100 Tensor Core GPUs for training, depending on
resource availability.

\textbf{Controlled fixed delays.}
The fixed-delay study uses Qwen3-1.7B-Base and sweeps the delay over
$D\in\{0,32,64,128,256,512\}$ optimizer updates. For $D=0$, we use
\texttt{inner\_steps=4}, so each rollout batch comprises four 64-trajectory
optimizer mini-batches, for 256 responses in total. For every $D>0$, we use
\texttt{inner\_steps=1} and maintain a FIFO queue of $D$ 64-trajectory rollout
batches. The queue is initialized with batches generated by the initial policy.
In these delayed runs, at each learner iteration, the oldest batch is consumed
in one optimizer update and a new batch from the current policy is appended,
yielding a steady-state behavior-policy lag of $D$ optimizer updates after the
initial fill transition.
All remaining optimization, rollout, and sequence-length settings match the
shared configuration above.

\begin{figure}[t]
    \centering
    \includegraphics[width=1\linewidth]{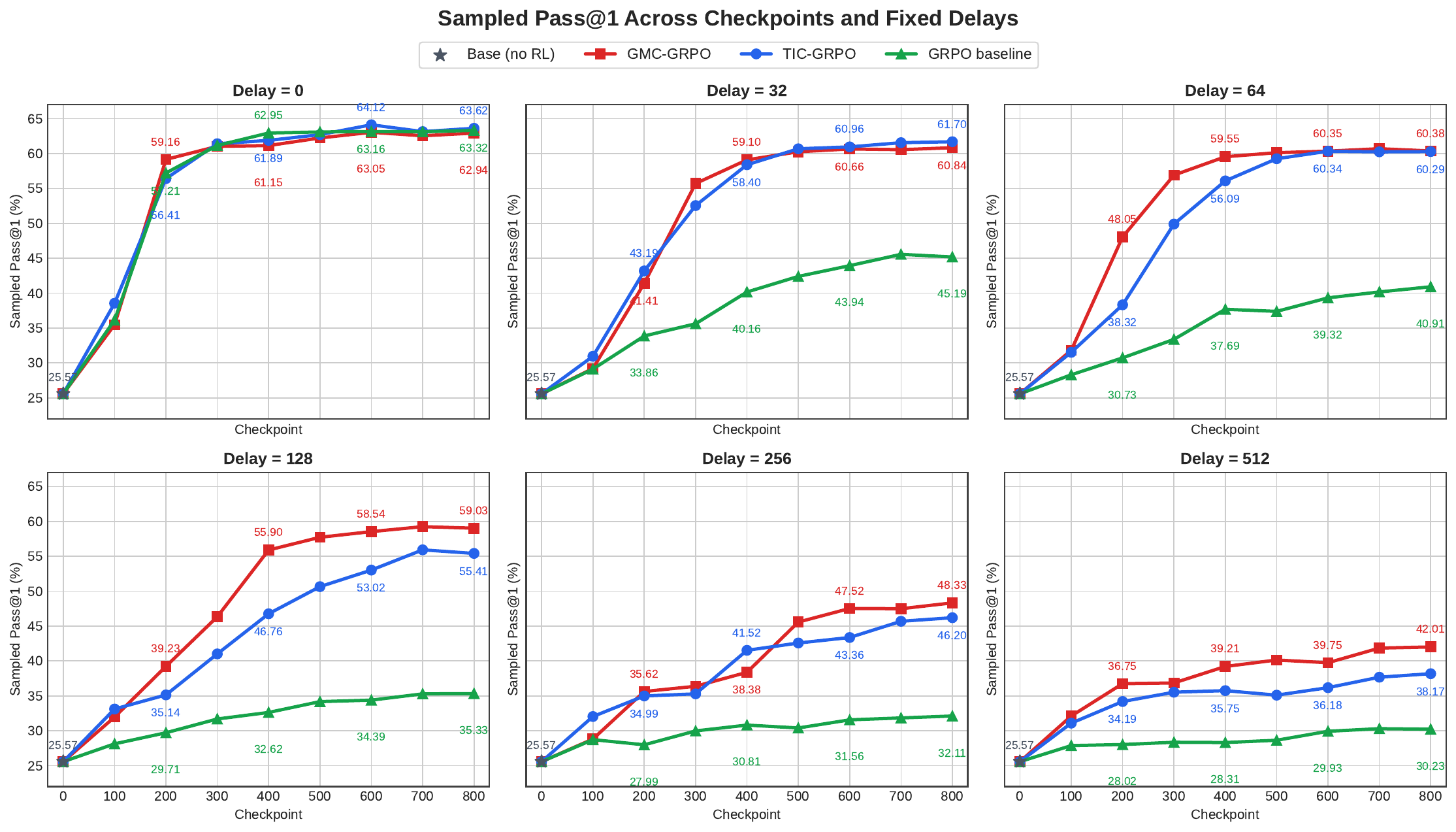}
    \caption{A full version of Figure \ref{fig:fixed-delay} with delay={0, 32, 64, 128, 256, 512}.}
    \label{fig:fixed-delay-full}
\end{figure}

\begin{figure}[t]
    \centering
    \includegraphics[width=1\linewidth]{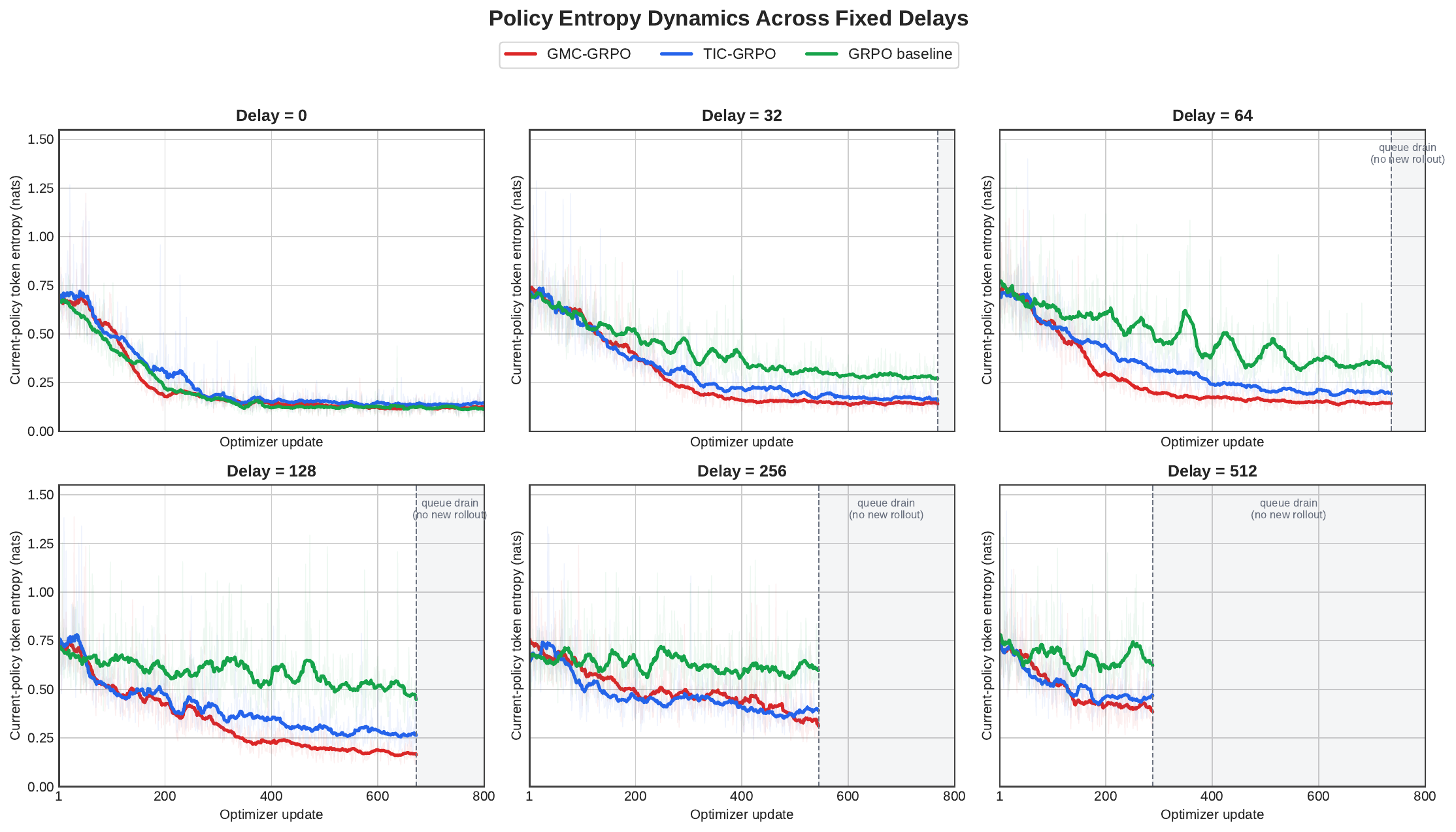}
    \caption{Current-policy token entropy across optimizer updates under fixed
    rollout delays. Faint curves show the per-update measurements, and solid
    curves show centered rolling means with a window of 25 updates. Dashed
    vertical lines mark queue drain, after which no new rollouts are generated.
    GMC-GRPO generally reaches a stable low-entropy regime earlier as the delay
    increases, while all methods behave similarly without delay.}
    \label{fig:fixed-delay-entropy}
\end{figure}

\textbf{Simulated asynchronous delays.}
The second suite models the coupling between response length and asynchronous
rollout latency. We maintain a hidden in-flight buffer and a virtual clock
$\tau$. Launching a 64-trajectory rollout batch advances the clock by
$C_{\mathrm{gen}}=85$. If the batch has mean completion length $L$, it becomes
visible at
\[
  \tau_{\mathrm{ready}}
  = \tau_{\mathrm{launch}} + \lambda L,
  \qquad \lambda=28.
\]
Each optimizer update advances the clock by $C_{\mathrm{update}}=1$, and ready
batches are exposed in completion-time order. Thus, longer batches tend to be
consumed under older learner policies. A batch generated at optimizer step $g$
and consumed at step $t$ has realized delay $d=t-g$; its approximate
steady-state mean is
\[
  \E[d]
  \simeq
  \frac{\lambda\,\E[L]}{C_{\mathrm{gen}}+C_{\mathrm{update}}}
  = \frac{28\,\E[L]}{86}.
\]
Consequently, $\lambda$ controls the delay distribution, while the measured
mean delay---approximately 160 updates in the reported experiments---also
depends on the model's response-length dynamics.

At initialization, we launch rollouts from policy 0 until the first batch is
visible. If the visible queue later becomes empty, we launch the minimum number
of additional batches needed to expose the next sample; both warmup and refill
use a safety cap of 400 launches. For
Qwen3-1.7B, rollout generation and actor updates execute sequentially on one
GPU. For Qwen3-4B, one GPU serves vLLM rollouts and a second GPU performs actor
optimization, with rollout generation overlapped with learner updates. This
systems-level overlap does not alter the sample order or the simulated-delay
schedule.

\begin{remark}[\textbf{Scope of the asynchronous-sampling assumption}]
The fixed-delay suite satisfies Assumption~\ref{as:async} exactly, since every group's delay is fixed in advance and independent of its responses. The simulated suite intentionally relaxes the conditional-independence part of the assumption: because delay is determined by response length, conditioning on a long delay tilts the returned responses toward longer completions. Our theorems therefore do not formally cover this setting. The key mechanism behind Assumption~3 is nevertheless preserved. The pipeline is lossless, so response length changes only \emph{when} a rollout is consumed, not \emph{whether} it is consumed, and each importance ratio uses the exact behavior policy that generated the rollout. As a result, the length--delay coupling largely averages out over the course of training, unlike completion-time filtering, which systematically removes long responses. Moreover, this coupling makes the problem harder rather than easier: long responses are both more stale and more prone to extreme importance ratios. We therefore view the simulated suite as a \textbf{stress test }beyond the idealized assumption, and GMC-GRPO's robustness there suggests that its advantage does not hinge on this idealization.
\end{remark}

\textbf{Method-specific settings.}
Vanilla GRPO \citep{shao2024deepseekmath} uses the two-sided clipping range
$\epsilon=0.2$ and the unnormalized group-centered advantage
$r_i-\overline r_g$. VCPO \citep{huang2026stable} uses
effective-sample-size-guided step-size scaling and its closed-form off-policy
minimum-variance baseline. GAC \citep{xu2026gac} uses cosine-similarity
thresholds $0.05$ and $0.30$. TIC-GRPO
\citep{pang2025tic} and GMC-GRPO both use leave-one-out advantages,
$\epsilon_{\mathrm{high}}=0.28$. All other training, rollout, data-order, and delay parameters are held
fixed across methods within each suite.

\textbf{Evaluation.}
For the simulated asynchronous study, we evaluate the step-800 checkpoint on
six mathematical-reasoning benchmarks: MATH-500 (500 problems)
\citep{hendrycks2021measuring,lightman2024let}, GSM8K (1,319)
\citep{cobbe2021training}, Gaokao2023-Math-En (385) \citep{liao2024mario}, the
\texttt{OE\_TO\_maths\_en\_COMP} subset of OlympiadBench (674)
\citep{he2024olympiadbench}, Minerva Math (272) \citep{lewkowycz2022solving},
and AMC23 (40) \citep{aops-amc}. Greedy accuracy uses temperature zero
and one output per problem. Sampled evaluation uses temperature $1.0$,
top-$p=1.0$. We draw 16 outputs per problem for
MATH-500, Minerva Math, and AMC23, and eight for GSM8K,
Gaokao2023-Math-En, and OlympiadBench. Each main-table entry is reported as
\emph{greedy/avg@$n$}, where avg@$n$ is the mean correctness across all $n$
sampled outputs rather than pass@any. The final column is the unweighted
macro-average over the six benchmarks, computed separately for greedy accuracy
and avg@$n$.

The fixed-delay study evaluates GRPO, TIC-GRPO, and GMC-GRPO on MATH-500 at
checkpoint 800. We use greedy decoding and 16 sampled responses per problem to
compute the unbiased pass@$k$ estimators for
$k\in\{1,2,4,8,16\}$. Both suites use the same 1,024-token
prompt limit, 4,096-token generation limit, two-shot prompt, and
last-complete-\verb|\boxed{}| answer scorer as in training.

\subsection{Supplemental Experiment Results}\label{additional-experiment-results}

\subsubsection{Additional experiment data and training dynamics}
\begin{table*}[t]
\centering
\caption{Performance on MATH500 under different fixed behavior-policy delays (Qwen3-1.7B-Base, checkpoint 800).}
\label{tab:math500-fixed-delay-seed1-final}
\begin{tabular}{llrrrrrr}
\toprule
Delay & Method & Greedy & pass@1 & pass@2 & pass@4 & pass@8 & pass@16 \\
\midrule
-- & Base (no RL) & 44.40 & 25.57 & 37.52 & 50.00 & 61.65 & 72.40 \\
\midrule
\multirow{3}{*}{0}
  & GMC-GRPO & 66.40 & 62.94 & 72.39 & 79.25 & 84.64 & 89.00 \\
  & TIC-GRPO & 66.80 & 63.62 & 72.59 & 79.47 & 84.90 & 89.00 \\
  & GRPO     & 66.20 & 63.32 & 72.41 & 79.25 & 84.84 & 88.80 \\
\cmidrule(lr){1-8}
\multirow{3}{*}{32}
  & GMC-GRPO & 64.80 & 60.84 & 70.47 & 77.85 & 83.89 & 88.60 \\
  & TIC-GRPO & 66.40 & 61.70 & 71.09 & 78.38 & 84.24 & 88.80 \\
  & GRPO     & 56.20 & 45.19 & 58.01 & 69.04 & 77.62 & 83.60 \\
\cmidrule(lr){1-8}
\multirow{3}{*}{64}
  & GMC-GRPO & 65.20 & 60.38 & 69.97 & 77.30 & 83.30 & 88.00 \\
  & TIC-GRPO & 65.40 & 60.29 & 70.21 & 77.42 & 82.92 & 86.80 \\
  & GRPO     & 52.00 & 40.91 & 53.93 & 65.43 & 74.47 & 81.20 \\
\cmidrule(lr){1-8}
\multirow{3}{*}{128}
  & GMC-GRPO & 68.40 & 59.03 & 69.57 & 77.34 & 83.45 & 88.00 \\
  & TIC-GRPO & 63.20 & 55.41 & 66.86 & 75.64 & 82.59 & 87.80 \\
  & GRPO     & 51.40 & 35.33 & 47.89 & 59.87 & 70.46 & 78.60 \\
\cmidrule(lr){1-8}
\multirow{3}{*}{256}
  & GMC-GRPO & 59.60 & 48.33 & 61.64 & 72.02 & 79.85 & 85.40 \\
  & TIC-GRPO & 56.00 & 46.20 & 59.78 & 71.12 & 80.05 & 86.20 \\
  & GRPO     & 47.60 & 32.11 & 44.78 & 57.10 & 68.17 & 77.60 \\
\cmidrule(lr){1-8}
\multirow{3}{*}{512}
  & GMC-GRPO & 56.40 & 42.01 & 55.97 & 67.78 & 76.91 & 83.40 \\
  & TIC-GRPO & 53.80 & 38.17 & 51.07 & 62.59 & 72.47 & 80.60 \\
  & GRPO     & 46.80 & 30.23 & 42.73 & 55.03 & 66.16 & 76.00 \\
\bottomrule
\end{tabular}
\vspace{2pt}
\begin{minipage}{0.98\linewidth}
\footnotesize Greedy uses temperature 0. pass@$k$ uses the unbiased estimator from $n=16$ samples per question (temperature 1.0, top-$p$ 1.0), averaged over 500 questions.
\end{minipage}
\end{table*}

Our theoretical results concern optimization convergence rather than only final
task accuracy. To connect these guarantees to observable training dynamics,
Figure~\ref{fig:fixed-delay-entropy} reports current-policy token entropy, which
is the most direct diagnostic among our logged statistics of how quickly the
policy distribution stabilizes. Without delay, all three methods exhibit
similar entropy decay, consistent with their nearly identical Pass@1
trajectories in Figure~\ref{fig:fixed-delay}. As the delay increases, GMC-GRPO
typically enters a stable low-entropy regime earlier than TIC-GRPO and vanilla
GRPO. Although entropy is not the formal stationarity measure used in our
theory, its faster decay and earlier stabilization, together with the Pass@1
curves, provide empirical evidence consistent with the predicted faster
convergence of GMC-GRPO under stale rollouts.

Table~\ref{tab:math500-fixed-delay-seed1-final} complements
Figure~\ref{fig:fixed-delay} by reporting the complete checkpoint-800 results,
including greedy accuracy and pass@$k$ for $k\in\{1,2,4,8,16\}$. The three
methods remain nearly indistinguishable at zero delay. Under stale rollouts,
GMC-GRPO consistently improves over vanilla GRPO and is generally competitive
with or stronger than TIC-GRPO, with the clearest gains at delays 128 and 512.
Thus, the trend in the main-text Pass@1 curves extends to greedy decoding and a
broader range of sampling budgets.

\subsubsection{GMC-GRPO benefits from larger group sizes}
\label{sec:group-size-benefits}

Corollary~\ref{cor:gmc-local-overlap} shows that larger groups reduce the
delay-dependent term in the GMC-GRPO convergence bound. To examine this benefit empirically, we vary the group size over
$G\in\{4,8,32\}$ while keeping all other experimental settings identical to
the fixed-delay experiment with delay $128$.
Figure~\ref{fig:group-size-reward} shows that GMC-GRPO establishes a reward
advantage over TIC-GRPO progressively earlier as the group size increases:
the curves begin to separate around step $300$ for $G=4$, around step $200$
for $G=8$, and around step $100$ for $G=32$.
This earlier separation indicates that GMC-GRPO benefits more promptly from
larger groups, providing empirical evidence consistent with the faster
convergence predicted by Corollary~\ref{cor:gmc-local-overlap}.

The training-reward gap also becomes smaller at larger group sizes toward the
end of training, with the two curves nearly meeting for $G=32$.
This pattern is consistent with both methods approaching a reward plateau:
GMC-GRPO makes faster progress early on, while TIC-GRPO catches up as training
continues. Thus, the benefit of larger groups is most apparent in the learning
dynamics, even when the final training rewards are similar.

Table~\ref{tab:group-size-checkpoints} complements these training curves with
MATH500 evaluations at checkpoints $200$, $300$, $400$, and $800$.
At checkpoint $200$, GMC-GRPO's sampled Pass@1 advantage over TIC-GRPO
increases from $-1.11$ percentage points at $G=4$ to $3.79$ at $G=8$ and
$7.10$ at $G=32$, showing that the earlier training gains also translate into
better evaluation performance. At the final checkpoint, GMC-GRPO retains
higher sampled Pass@1 at all three group sizes and matches or exceeds
TIC-GRPO in greedy accuracy. In particular, at $G=32$, both methods achieve
$66.00\%$ greedy accuracy, but GMC-GRPO still achieves higher sampled Pass@1
($61.18\%$ versus $59.80\%$).

\begin{figure}[t]
    \centering
    \includegraphics[width=1\linewidth]{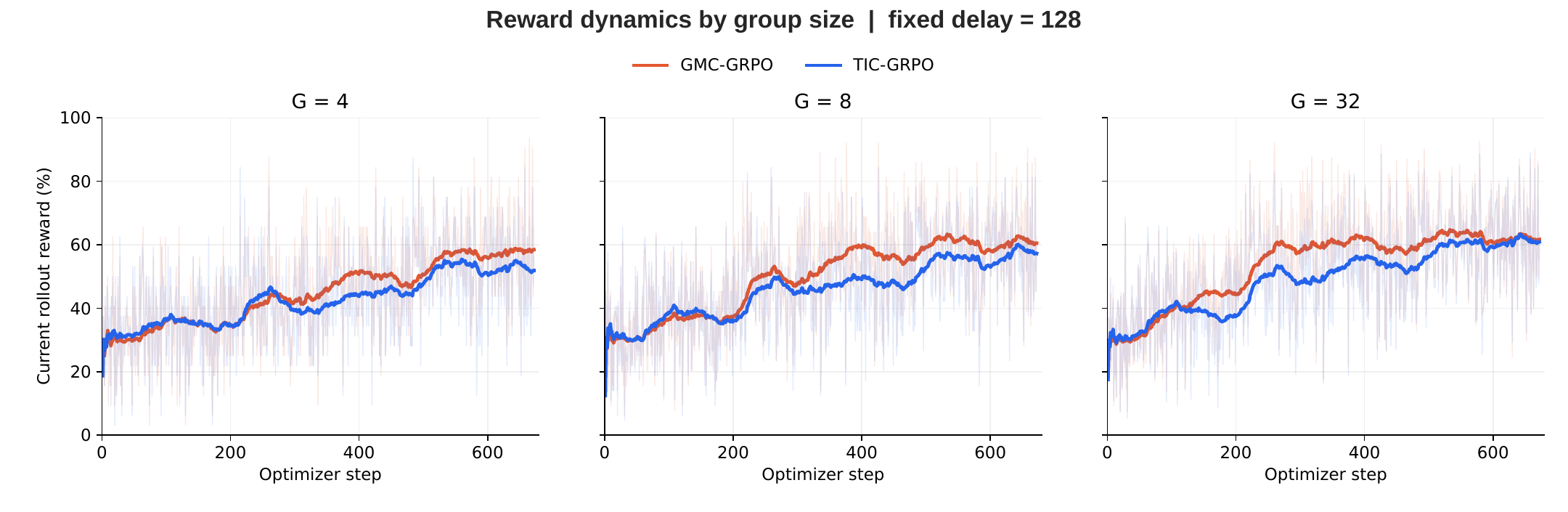}
    \caption{Training-reward dynamics for GMC-GRPO and TIC-GRPO with group
    sizes $G\in\{4,8,32\}$ under the fixed-delay setting with delay $128$.
    All other experimental settings are unchanged. Faint curves show
    individual measurements, and solid curves show smoothed trends.
    GMC-GRPO establishes a reward advantage earlier as the group size
    increases, while the gap narrows toward the end of training,
    particularly for $G=32$.}
    \label{fig:group-size-reward}
\end{figure}

\begin{table}[t]
\centering
\small
\setlength{\tabcolsep}{8pt}
\caption{MATH500 evaluation of Qwen3-1.7B-Base at checkpoints $200$, $300$,
$400$, and $800$ under the fixed-delay setting with delay $128$.
All settings except group size match the fixed-delay experiment.
Each method entry reports greedy accuracy / sampled Pass@1 (\%), with
sampled Pass@1 computed as mean correctness over 16 responses per problem
(avg@16). $\Delta$ denotes the sampled Pass@1 of GMC-GRPO minus that of
TIC-GRPO, measured in percentage points (pp).}
\label{tab:group-size-checkpoints}
\begin{tabular}{ccccc}
\toprule
Checkpoint & $G$ & GMC-GRPO & TIC-GRPO & $\Delta$ Pass@1 (pp) \\
\midrule
200 & 4  & 51.40 / 33.70 & 51.60 / 34.81 & $-1.11$ \\
200 & 8  & 55.40 / 39.23 & 50.20 / 35.44 & $+3.79$ \\
200 & 32 & 62.80 / 46.51 & 53.60 / 39.41 & $+7.10$ \\
\midrule
300 & 4  & 57.00 / 39.02 & 54.80 / 36.05 & $+2.97$ \\
300 & 8  & 61.40 / 46.34 & 55.20 / 41.01 & $+5.33$ \\
300 & 32 & 65.40 / 56.34 & 56.00 / 42.10 & $+14.24$ \\
\midrule
400 & 4  & 62.60 / 48.52 & 58.60 / 42.29 & $+6.23$ \\
400 & 8  & 67.00 / 55.90 & 60.00 / 46.76 & $+9.14$ \\
400 & 32 & 64.00 / 58.55 & 62.80 / 51.41 & $+7.14$ \\
\midrule
800 & 4  & 67.00 / 56.17 & 63.00 / 53.91 & $+2.26$ \\
800 & 8  & 68.40 / 59.03 & 63.20 / 55.41 & $+3.62$ \\
800 & 32 & 66.00 / 61.18 & 66.00 / 59.80 & $+1.38$ \\
\bottomrule
\end{tabular}
\end{table}

\subsubsection{Comparison against A-3PO.}
\begin{table}[t]
\centering
\small
\setlength{\tabcolsep}{3.5pt}
\renewcommand{\arraystretch}{1.25}
\caption{Comparison of A-3PO and GMC-GRPO under simulated asynchronous delays.
Each benchmark cell is greedy/avg@$n$, with $n=16$ for MATH500, Minerva,
and AMC23, and $n=8$ for GSM8K, Gaokao, and OlympiadBench.
\emph{Avg. perf.} is the average over the six benchmarks, reported as
mean greedy/mean avg@$n$.
For Qwen3-1.7B and Qwen3-4B, the better result for each metric is bolded.
A-3PO exhibited training collapse in our Qwen3-8B experiment.}
\label{tab:a3po-comparison}
\resizebox{\linewidth}{!}{%
\begin{tabular}{lcccccccc}
\toprule
\textbf{Method}
& \textbf{Avg. Delay}
& \textbf{MATH500}
& \textbf{GSM8K}
& \textbf{Gaokao}
& \textbf{Olympiad}
& \textbf{Minerva}
& \textbf{AMC23}
& \textbf{Avg. perf.} \\
\midrule
\multicolumn{9}{c}{\textit{Qwen3-1.7B}} \\
\midrule
A-3PO
& 160.80
& \textbf{67.40} \,/\, \textbf{61.26}
& \textbf{81.88} \,/\, \textbf{76.42}
& 53.77 \,/\, \textbf{51.69}
& \textbf{30.71} \,/\, \textbf{25.56}
& \textbf{28.68} \,/\, \textbf{24.20}
& 32.50 \,/\, \textbf{32.50}
& \textbf{49.16} \,/\, \textbf{45.27} \\
\ours \textbf{GMC-GRPO}
& 159.75
& 66.00 \,/\, 57.29
& 81.65 \,/\, 72.66
& \textbf{55.32} \,/\, 47.73
& 28.19 \,/\, 22.61
& 26.47 \,/\, 20.11
& \textbf{35.00} \,/\, 28.44
& 48.77 \,/\, 41.47 \\
\midrule
\multicolumn{9}{c}{\textit{Qwen3-4B}} \\
\midrule
A-3PO
& 138.34
& 70.80 \,/\, 69.16
& 90.14 \,/\, \textbf{88.55}
& 59.48 \,/\, 58.38
& 36.80 \,/\, 35.16
& 31.62 \,/\, 29.78
& 50.00 \,/\, 46.88
& 56.47 \,/\, 54.65 \\
\ours \textbf{GMC-GRPO}
& 141.01
& \textbf{77.60} \,/\, \textbf{72.54}
& \textbf{91.13} \,/\, 87.03
& \textbf{65.19} \,/\, \textbf{60.91}
& \textbf{43.18} \,/\, \textbf{37.72}
& \textbf{34.19} \,/\, \textbf{30.10}
& \textbf{52.50} \,/\, \textbf{48.91}
& \textbf{60.63} \,/\, \textbf{56.20} \\
\midrule
\multicolumn{9}{c}{\textit{Qwen3-8B}} \\
\midrule
A-3PO
& -
& \multicolumn{7}{c}{\textit{Training collapsed}} \\
\ours \textbf{GMC-GRPO}
& 141.1
& 78.80 \,/\, 75.22
& 93.33 \,/\, 91.37
& 69.09 \,/\, 63.64
& 44.96 \,/\, 40.34
& 40.07 \,/\, 34.35
& 60.00 \,/\, 52.81
& 64.37 \,/\, 59.62 \\
\bottomrule
\end{tabular}%
}
\end{table}

\begin{figure}[t]
    \centering
    \includegraphics[width=0.8\linewidth]{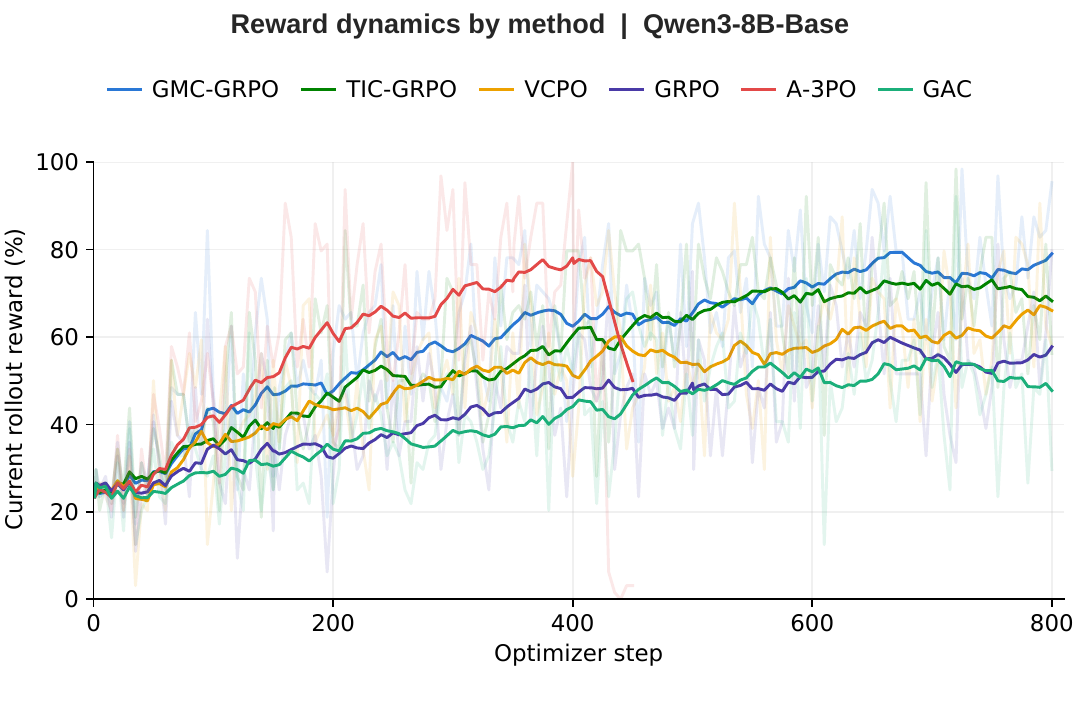}
    \caption{Training reward on Qwen3-8B under simulated asynchronous delays.
    Faint curves show individual measurements, and solid curves show smoothed
    trends. A-3PO initially improves rapidly, but its recorded reward falls to
    near zero around steps 430--440. GMC-GRPO continues to improve over the
    remaining training steps.}
    \label{fig:a3po-training-collapse}
\end{figure}

In this section, we present an independent comparison between GMC-GRPO and A-3PO \citep{li20253po}, which achieves the best average performance on Qwen3-1.7B but suffers training collapse on Qwen3-8B in our experiments. We therefore report this comparison separately and discuss the observed instability in detail.

To make use of stale rollouts without overly restricting policy updates,
A-3PO follows decoupled policy optimization \citep{fu2026areal}: it uses a
proximal anchor for clipping while retaining importance correction relative to
the behavior policy. Suppressing token arguments $(y_t\mid x,y_{<t})$, the
decoupled surrogate is
\begin{equation}\label{eq:a3po-decoupled}
L_{\mathrm{decoupled}}^{\mathrm{CLIP}}(\theta)
=\mathbb{E}\!\left[
\frac{\pi_{\mathrm{prox}}}{\pi_{\theta_{\mathrm{behav}}}}
\min\!\left\{
\frac{\pi_\theta}{\pi_{\mathrm{prox}}}\widehat A,\,
\operatorname{clip}\!\left(
\frac{\pi_\theta}{\pi_{\mathrm{prox}}},1-\epsilon,1+\epsilon
\right)\widehat A
\right\}\right],
\end{equation}
where the expectation is over sampled rollout tokens and $\widehat A$ is the
group-based advantage. A-3PO approximates the proximal anchor by
\[
\log\pi_{\mathrm{prox}}
=\alpha\log\pi_{\theta_{\mathrm{behav}}}
 +(1-\alpha)\log\pi_\theta,
\qquad
\alpha=\frac{1}{\max\{1,\tau\}},
\]
This design is reasonable when delays are small: the policy shift is limited,
and retaining more learning signals from stale rollouts can accelerate learning.
Under large delays, however, $\alpha$ approaches zero, so the ratio used for
clipping can remain close to one and the update can resemble
\textit{unclipped GRPO}. Since the outer importance factor
$\pi_{\mathrm{prox}}/\pi_{\theta_{\mathrm{behav}}}$ remains unclipped,
large importance weights can destabilize updates and even lead to training collapse in extreme cases.

This is also demonstrated in our experiments: Table~\ref{tab:a3po-comparison} shows that A-3PO outperforms GMC-GRPO on
Qwen3-1.7B in mean greedy/mean avg@$n$ accuracy (49.16/45.27 versus 48.77/41.47).
This advantage is consistent with prior evidence that hard clipping can discard
useful gradients and restrict policy improvement
\citep{dwyer2025s, zheng2026prosperity}. Retaining more of these signals can
therefore improve performance when extreme importance ratios are rare.

The Qwen3-8B experiment illustrates the accompanying stability challenge.
In Figure~\ref{fig:a3po-training-collapse}, A-3PO initially learns faster but
its training reward collapses around step 430, whereas GMC-GRPO continues to
improve and reaches mean greedy/mean avg@$n$ accuracy of 64.37/59.62.
This trajectory resembles the prosperity-before-collapse behavior reported by
\citet{zheng2026prosperity}, where removing the trust region exposes useful
stale-data signals but also permits destabilizing updates. The comparison
highlights the motivation for GMC: preserving learning signals must be coupled
with control of update magnitude. GMC caps the average importance mass within
each group and preserves the relative contributions of its responses,
controlling the update second moment under the assumptions of
Proposition~\ref{prop:weighted-second-moment}.

\section{Mixed-Delay Setting, GMC-GRPO and TIC-GRPO Algorithms}\label{sec:algorithm}

TIC-GRPO was originally proposed for synchronous training
\citep{pang2025tic}, so Algorithm~\ref{alg:tic-sync} first restates this
original synchronous form for reference. Algorithm~\ref{alg:async} then gives the asynchronous versions of GMC-GRPO and TIC-GRPO analyzed in this paper. The two asynchronous methods share the same rollout
and learner schedule and differ only in their importance correction.
Each completed group stores its prompt,
responses, rewards, behavior-policy version $s$, and behavior
log-probabilities $\ell^{(i)}=\log\pi_{\theta_s}(y^{(i)}\mid x)$.
For a group $b$ consumed at iteration $n$, both algorithms compute
\[
\rho_n^{(b,i)}
=\exp\!\left(\log\pi_{\theta_n}(y^{(b,i)}\mid x^{(b)})
-\ell^{(b,i)}\right).
\]
For GMC-GRPO, the learner forms $g_n^{\mathrm{GMC}}$ in
\cref{eq:gmcgrad-main} using the group-rescaled ratios
$\widehat\rho_n^{(b,i)}$. For TIC-GRPO, it instead uses the individually
truncated ratios $\rho_n^{(b,i)}\mathbf1\{\rho_n^{(b,i)}\le1+\eps\}$
to form $g_n^{\mathrm{TIC}}$, as defined in
\cref{eq:appvangroup,eq:appvangrad}. In either case, the learner takes the
ascent step $\theta_{n+1}=\theta_n+\eta g_n$. Both methods use the
leave-one-out advantages defined in Section~\ref{sec:prelim-rl}.
Behavior log-probabilities and advantages are held fixed during each update.

In the synchronous schedule, $K$ mini-batches share one behavior-policy
snapshot: rollout collection finishes before optimization begins, and the
next global batch is generated only after these updates finish. Thus,
synchronous training can still use stale data during successive mini-batch
updates. In the asynchronous schedule, rollout workers and the learner run
concurrently. The learner needs only $B$ completed groups to update; it does
not wait for all workers or require a common behavior-policy version.

\begin{algorithm}[t]
\caption{TIC-GRPO with synchronous rollouts}
\label{alg:tic-sync}
\begin{algorithmic}[1]
\Require initial parameter $\theta_0$; step size $\eta$; total updates $N$;
group size $G\ge2$; mini-batch size $B$; mini-batches per global batch $K\ge1$;
clipping threshold $1+\eps$; prompt distribution $\cD$; reward $r_\phi$.
\State $n\gets0$.
\While{$n<N$}
  \State freeze the behavior-policy snapshot $\theta_s\gets\theta_n$ with $s\gets n$;
  set $K_n\gets\min\{K,N-n\}$.
  \State draw $K_nB$ prompts from $\cD$ and generate $G$ responses per prompt
  under the fixed policy $\pi_{\theta_s}$.
  \State wait until all $K_nB$ groups are generated and scored by $r_\phi$;
  store their behavior log-probabilities and version $s$.
  \State partition the completed global batch into $K_n$ mini-batches
  of $B$ whole groups each.
  \For{each mini-batch in the partition}
    \State denote its groups by $b\in[B]$ and set $s_n^{(b)}\gets s$,
    $\tau_n^{(b)}\gets n-s$.
    \State form $\widehat A^{(b,i)}$ from the rewards of all $G$ responses
    in each group.
    \State compute $\rho_n^{(b,i)}$ from the current policy and stored
    behavior log-probabilities for all $b\in[B]$, $i\in[G]$.
    \State form $g_n^{\mathrm{TIC}}$ using \cref{eq:appvangrad} and update
    $\theta_{n+1}\gets\theta_n+\eta g_n^{\mathrm{TIC}}$.
    \State $n\gets n+1$.
  \EndFor
\EndWhile
\end{algorithmic}
\end{algorithm}

\begin{algorithm}[t]
\caption{TIC-GRPO and GMC-GRPO with asynchronous rollouts}
\label{alg:async}
\label{alg:tic-async}
\label{alg:gmc}
\begin{algorithmic}[1]
\Require method $\mathsf{method}\in\{\mathrm{TIC},\mathrm{GMC}\}$;
initial parameter $\theta_0$; step size $\eta$; total updates $N$;
group size $G\ge2$; mini-batch size $B$; threshold $1+\eps$;
rollout workers over $\cD$ with reward $r_\phi$; shared buffer $\mathcal Q$.
\State initialize $\mathcal Q\gets\varnothing$ and
publish the snapshot $(0,\theta_0)$.
\Statex \textbf{Run the following worker and learner loops concurrently.}
\Statex \textbf{Each rollout worker:}
\While{training is active}
  \State read a published snapshot $(s,\theta_s)$ and draw $x\sim\cD$.
  \State generate $y^{(1)},\ldots,y^{(G)}\sim\pi_{\theta_s}(\cdot\mid x)$,
  keeping $\theta_s$ fixed throughout this group.
  \State compute $r^{(i)}\gets r_\phi(x,y^{(i)})$ and store
  $\ell^{(i)}\gets\log\pi_{\theta_s}(y^{(i)}\mid x)$ for all $i\in[G]$.
  \State append $(x,\{y^{(i)},r^{(i)},\ell^{(i)}\}_{i=1}^G,s)$ to
  $\mathcal Q$ as one completed group.
\EndWhile
\Statex \textbf{Learner:}
\For{$n=0,1,\ldots,N-1$}
  \State once $|\mathcal Q|\ge B$, pop $B$ completed groups $\cM_n$;
  their behavior-policy versions need not match.
  \State denote each group's stored version by $s_n^{(b)}$ and set
  $\tau_n^{(b)}\gets n-s_n^{(b)}$.
  \State form $\widehat A^{(b,i)}$ from the rewards of all $G$ responses
  in each group.
  \State compute $\rho_n^{(b,i)}$ from the current policy and stored
  behavior log-probabilities for all $b\in[B]$, $i\in[G]$.
  \If{$\mathsf{method}=\mathrm{TIC}$}
    \State form $g_n\gets g_n^{\mathrm{TIC}}$ using \cref{eq:appvangrad}.
  \Else \Comment{GMC-GRPO}
    \For{each group $b\in[B]$}
      \State compute $\overline\rho_n^{(b)}\gets
      G^{-1}\sum_{i=1}^G\rho_n^{(b,i)}$ using log-sum-exp.
      \State set $\widehat\rho_n^{(b,i)}\gets
      \rho_n^{(b,i)}/\max\{1,\overline\rho_n^{(b)}/(1+\eps)\}$
      for all $i\in[G]$.
    \EndFor
    \State form $g_n\gets g_n^{\mathrm{GMC}}$ using \cref{eq:gmcgrad-main}.
  \EndIf
  \State update $\theta_{n+1}\gets\theta_n+\eta g_n$.
  \State publish $(n+1,\theta_{n+1})$ without waiting for ongoing rollouts;
  in-progress groups retain their launch snapshots.
\EndFor
\State signal the rollout workers to stop.
\end{algorithmic}
\end{algorithm}

We use the asynchronous sampling model and notation introduced in
Section~\ref{sec:prelim-async}. In particular, $\tau_n^{(b)}$ is the delay of
group $b$, $\rho_n^{(b,i)}$ is the delayed importance ratio in
\cref{eq:delayed-ratio}, and $\widehat A^{(b,i)}$ is the leave-one-out
advantage defined in Section~\ref{sec:prelim-rl}. Let $\cM_n$ denote the
mini-batch of $B$ completed groups returned at learner iteration $n$.
To aggregate the mixed delays over the training horizon, for any integer
$k\ge 1$, define the averaged $k$th delay moment
\begin{equation}
\label{eq:average-delay-moment}
\overline\tau_k
:=\frac{1}{N}\sum_{n=0}^{N-1}
\E\!\left[\frac{1}{B}\sum_{b=1}^B
\bigl(\tau_n^{(b)}\bigr)^k\right].
\end{equation}
Thus, $\overline\tau_2$ and $\overline\tau_4$ denote the averaged second and
fourth delay moments, respectively. The second moment controls intermediate
parameter-drift terms, while the fourth moment appears in the final convergence
bounds after the delay-induced bias is squared.

For the appendix proofs, we additionally write the average importance ratio in
group $b$ as
$\overline\rho_n^{(b)}:=G^{-1}\sum_{i=1}^G\rho_n^{(b,i)}$. Softmax policies
have full support, so every ratio is finite and
$\overline\rho_n^{(b)}>0$ almost surely. Using the common scaling factor from
Proposition~\ref{prop:gmc-retained-mass}, the GMC effective ratio admits the
equivalent form
\begin{equation}\label{eq:gmcratio}
\widehat\rho_n^{(b,i)}
:=
\frac{\rho_n^{(b,i)}}
{\max\{1,\overline\rho_n^{(b)}/(1+\eps)\}}.
\end{equation}
This form is used in the groupwise decomposition and moment bounds below.

The GMC branch of Algorithm~\ref{alg:async} applies this groupwise rescaling
within the shared asynchronous training loop. The behavior log-probabilities
stored with each completed group are sufficient to compute every delayed
importance ratio at consumption time.

\begin{remark}[Autodifferentiation implementation]
\Cref{eq:gmcgrad-main} defines the update directly. In a
surrogate-loss implementation,
$\max\{1,\overline\rho_n^{(b)}/(1+\eps)\}$ and
$\widehat A^{(b,i)}$ must be treated as constants using stop-gradient. The
numerator ratio supplies the policy gradient. Differentiating through the common
group denominator defines a different estimator and is not covered by the
theorems below.
\end{remark}

\paragraph{Proof-specific conditioning and divergence notation.}
The filtration pair $\cF_n^-\subseteq\cH_n$ is defined in
Section~\ref{sec:prelim-async}, and Assumption~\ref{as:async} specifies the
corresponding conditional sampling properties. For each returned group, define
the current-to-behavior divergences
\begin{align*}
D_n^{(b)}
&:=\KL\!\left(
\pi_{\theta_n}(\cdot\mid x^{(b)})\,\|\,
\pi_{\theta_{n-\tau_n^{(b)}}}(\cdot\mid x^{(b)})
\right),\\
\chi_n^{2,(b)}
&:=\chi^2\!\left(
\pi_{\theta_n}(\cdot\mid x^{(b)})\,\|\,
\pi_{\theta_{n-\tau_n^{(b)}}}(\cdot\mid x^{(b)})
\right).
\end{align*}
Both divergences, as well as $J_{x^{(b)}}(\theta_n)$, are
$\cH_n$-measurable and hence fixed inside conditional expectations given
$\cH_n$. Because $\cF_n^-\subseteq\cH_n$, every integrable random variable
$Z$ satisfies
\[
\E[Z\mid\cF_n^-]=\E\!\left[\E[Z\mid\cH_n]\mid\cF_n^-\right].
\]
The proofs first condition on $\cH_n$ to hold the prompt, delay, and behavior
policy fixed, and then average back to $\cF_n^-$ to recover a statement about
the learner iterate.

With the algorithm and conditioning convention fixed, the next section
collects the analytic tools shared by the GMC-GRPO and TIC-GRPO proofs.

\section{Technical preliminaries}\label{sec:assumptions}

The convergence proofs for both algorithms rely on a common set of analytic
tools. Starting from Assumption~\ref{as:lip}, we establish trajectory-score
bounds, control trajectory KL divergence by the squared parameter distance,
and prove smoothness of the reward objective. We then use this smoothness
bound to prove Lemma~\ref{lem:master-descent}, which reduces convergence
analysis to controlling the update second moment and the squared conditional
bias. The following sections supply these controls for GMC-GRPO and TIC-GRPO.

\begin{lemma}[Trajectory score bounds]\label{lem:score-properties}
Under Assumption~\ref{as:lip}, for every prompt $x$, parameter $\theta$, and
length-$T$ response $y$,
\begin{align}
\E_{y\sim\pi_\theta(\cdot\mid x)}
\norm{\nabla\log\pi_\theta(y\mid x)}_2^2
&\le2LT\log|\cA|,
\label{eq:scoremoment}\\
\norm{\nabla\log\pi_\theta(y\mid x)}_2^2
&\le-2LT\log\pi_\theta(y\mid x).
\label{eq:selfbound}
\end{align}
If Assumption~\ref{as:coverage} also holds, then
\begin{equation}\label{eq:coverage-score}
\norm{\nabla\log\pi_\theta(y\mid x)}_2
\le T\sqrt{2L\log\mathcal C}.
\end{equation}
\end{lemma}

\begin{proof}
Fix a token position $t$ and its prefix $(x,y_{<t})$, and write
$f_t(\theta):=-\log\pi_\theta(y_t\mid x,y_{<t})$. Assumption~\ref{as:lip}
states that $\nabla f_t$ is $L$-Lipschitz. The descent lemma therefore gives,
for every vector $v$,
\[
 f_t(\theta+v)\le f_t(\theta)+\ip{\nabla f_t(\theta)}{v}
+\frac L2\norm{v}_2^2.
\]
Choose $v=-L^{-1}\nabla f_t(\theta)
=L^{-1}\nabla\log\pi_\theta(y_t\mid x,y_{<t})$. Since $f_t$ is
nonnegative,
\begin{align*}
0
&\le-\log\pi_{\theta+L^{-1}\nabla\log\pi_\theta(y_t\mid x,y_{<t})}
      (y_t\mid x,y_{<t})\\
&\le-\log\pi_\theta(y_t\mid x,y_{<t})
   -\frac1{2L}\norm{\nabla\log\pi_\theta(y_t\mid x,y_{<t})}_2^2,
\end{align*}
Rearranging the last inequality shows that each token score satisfies
\begin{equation}\label{eq:token-score-bound}
\norm{\nabla\log\pi_\theta(y_t\mid x,y_{<t})}_2^2
\le-2L\log\pi_\theta(y_t\mid x,y_{<t}).
\end{equation}
Now let $s_t:=\nabla\log\pi_\theta(y_t\mid x,y_{<t})$. The autoregressive
factorization gives $\nabla\log\pi_\theta(y\mid x)=\sum_{t=1}^Ts_t$, so
Cauchy--Schwarz and \cref{eq:token-score-bound} yield
\[
\norm{\nabla\log\pi_\theta(y\mid x)}_2^2
\le T\sum_{t=1}^T\norm{s_t}_2^2
\le-2LT\sum_{t=1}^T\log\pi_\theta(y_t\mid x,y_{<t})
=-2LT\log\pi_\theta(y\mid x),
\]
which proves \cref{eq:selfbound}.

For the expected bound, first expand the conditional score identity:
\[
\E[s_t\mid x,y_{<t}]
=\sum_{a\in\cA}\pi_\theta(a\mid x,y_{<t})
\nabla\log\pi_\theta(a\mid x,y_{<t})
=\sum_{a\in\cA}\nabla\pi_\theta(a\mid x,y_{<t})
=\nabla 1=0.
\]
If $s<t$, then $s_s$ is determined by $(x,y_{<t})$. The tower property
therefore gives
$\E\ip{s_s}{s_t}=\E\ip{s_s}{\E[s_t\mid x,y_{<t}]}=0$; hence distinct token
scores are orthogonal in expectation. Expanding the square and applying
\cref{eq:token-score-bound},
\begin{align*}
\E_{y\sim\pi_\theta(\cdot\mid x)}
\norm{\nabla\log\pi_\theta(y\mid x)}_2^2
&=\E\norm{\sum_{t=1}^Ts_t}_2^2
=\sum_{t=1}^T\E\norm{s_t}_2^2\\
&\le2L\sum_{t=1}^T\E\!\left[-\log\pi_\theta(y_t\mid x,y_{<t})\right]
\le2LT\log|\cA|,
\end{align*}
where the last inequality uses
\[
\E[-\log\pi_\theta(y_t\mid x,y_{<t})]
=\E_{y_{<t}}\!\left[H\bigl(\pi_\theta(\cdot\mid x,y_{<t})\bigr)\right]
\le\log|\cA|.
\]
This proves \cref{eq:scoremoment}. Finally, Assumption~\ref{as:coverage}
implies $-\log\pi_\theta(y_t\mid x,y_{<t})\le\log\mathcal C$.
Combining this with \cref{eq:token-score-bound} and the triangle inequality,
\[
\norm{\nabla\log\pi_\theta(y\mid x)}_2
\le\sum_{t=1}^T\norm{s_t}_2
\le T\sqrt{2L\log\mathcal C},
\]
which is \cref{eq:coverage-score}.
\end{proof}

\begin{lemma}[Trajectory KL is quadratically controlled]\label{lem:trajectory-kl}
Under Assumption~\ref{as:lip}, for every prompt $x$ and all $\theta,\theta'$,
\begin{equation}\label{eq:trajectory-kl}
\KL\bigl(\pi_\theta(\cdot\mid x)\,\|\,
\pi_{\theta'}(\cdot\mid x)\bigr)
\le\frac{TL}{2}\norm{\theta-\theta'}_2^2.
\end{equation}
\end{lemma}

\begin{proof}
For a fixed prefix $h=(x,y_{<t})$ and token $a$, the $L$-Lipschitz score
implies the two-sided smoothness bound
\[
\log\pi_\theta(a\mid h)-\log\pi_{\theta'}(a\mid h)
\le\ip{\nabla\log\pi_\theta(a\mid h)}{\theta-\theta'}
+\frac L2\norm{\theta-\theta'}_2^2.
\]
The autoregressive KL chain rule sums this inequality over token positions:
\begin{align*}
\KL\bigl(\pi_\theta(\cdot\mid x)\,\|\,\pi_{\theta'}(\cdot\mid x)\bigr)
&=\E_{y\sim\pi_\theta(\cdot\mid x)}
\sum_{t=1}^T\log\frac{\pi_\theta(y_t\mid x,y_{<t})}
{\pi_{\theta'}(y_t\mid x,y_{<t})}\\
&\le\sum_{t=1}^T\E_{y\sim\pi_\theta(\cdot\mid x)}
\left[
\ip{\nabla\log\pi_\theta(y_t\mid x,y_{<t})}{\theta-\theta'}
+\frac L2\norm{\theta-\theta'}_2^2
\right]\\
&=\frac{TL}{2}\norm{\theta-\theta'}_2^2.
\end{align*}
For the last equality, condition on $(x,y_{<t})$ in each linear term. The score
identity proved above gives
\[
\E\!\left[
\ip{\nabla\log\pi_\theta(y_t\mid x,y_{<t})}{\theta-\theta'}
\,\middle|\,x,y_{<t}\right]
=\ip{\sum_{a\in\cA}\nabla\pi_\theta(a\mid x,y_{<t})}{\theta-\theta'}=0.
\]
Thus only the $L\norm{\theta-\theta'}_2^2/2$ remainder remains at each of the
$T$ positions, proving \cref{eq:trajectory-kl}.
\end{proof}

\begin{lemma}[Smoothness of the reward objective]\label{lem:objective-smoothness}
Under Assumption~\ref{as:lip}, the objective has
$TRL(2\log|\cA|+1)$-Lipschitz gradient. Consequently, any update
$\theta_{n+1}=\theta_n+\eta g_n$ satisfies
\begin{equation}\label{eq:objective-smoothness}
J(\theta_{n+1})
\ge J(\theta_n)+\eta\ip{\nabla J(\theta_n)}{g_n}
-\frac{TRL(2\log|\cA|+1)\eta^2}{2}\norm{g_n}_2^2.
\end{equation}
\end{lemma}

\begin{proof}
Fix a prompt--response pair $(x,y)$ and write
$S_\theta(y\mid x):=\nabla\log\pi_\theta(y\mid x)$. The autoregressive
decomposition and Assumption~\ref{as:lip} imply
\[
\norm{S_\theta(y\mid x)-S_{\theta'}(y\mid x)}_2
\le TL\norm{\theta-\theta'}_2.
\]
Let $\Delta:=\theta-\theta'$ and $\gamma(s):=\theta'+s\Delta$. The restriction
$s\mapsto S_{\gamma(s)}(y\mid x)$ is
$TL\norm{\Delta}_2$-Lipschitz, hence absolutely continuous and differentiable
for almost every $s$, with
\[
\norm{\frac{\mathrm{d}}{\mathrm{d}s}S_{\gamma(s)}(y\mid x)}_2
\le TL\norm{\Delta}_2.
\]
Moreover,
$\nabla\pi_\theta(y\mid x)=\pi_\theta(y\mid x)S_\theta(y\mid x)$ and
$\frac{\mathrm{d}}{\mathrm{d}s}\pi_{\gamma(s)}(y\mid x)
=\pi_{\gamma(s)}(y\mid x)\ip{S_{\gamma(s)}(y\mid x)}{\Delta}$.
Both factors are bounded on the compact segment, so their product is absolutely
continuous. At almost every $s$, the one-dimensional product rule gives
\[
\frac{\mathrm{d}}{\mathrm{d}s}\nabla\pi_{\gamma(s)}(y\mid x)
=\pi_{\gamma(s)}(y\mid x)
\left\{
\ip{S_{\gamma(s)}(y\mid x)}{\Delta}S_{\gamma(s)}(y\mid x)
+\frac{\mathrm{d}}{\mathrm{d}s}S_{\gamma(s)}(y\mid x)
\right\}.
\]
Lemma~\ref{lem:score-properties} now allows us to sum the derivative bound over
the finite trajectory space:
\begin{align*}
\sum_y\norm{\frac{\mathrm{d}}{\mathrm{d}s}\nabla\pi_{\gamma(s)}(y\mid x)}_2
&\le\norm{\Delta}_2\E_{y\sim\pi_{\gamma(s)}(\cdot\mid x)}\!\left[
\norm{S_{\gamma(s)}(y\mid x)}_2^2
+TL\right]\\
&\le TL(2\log|\cA|+1)\norm{\Delta}_2.
\end{align*}
The fundamental theorem of calculus for absolutely continuous functions now
yields the promptwise estimate
\begin{align*}
\sum_y\norm{\nabla\pi_\theta(y\mid x)-\nabla\pi_{\theta'}(y\mid x)}_2
&\le\int_0^1\sum_y
\norm{\frac{\mathrm{d}}{\mathrm{d}s}\nabla\pi_{\gamma(s)}(y\mid x)}_2\,\mathrm{d}s\\
&\le TL(2\log|\cA|+1)\norm{\theta-\theta'}_2.
\end{align*}

It remains to pass from a fixed prompt to the objective. First,
\[
\sum_y\norm{\nabla\pi_\theta(y\mid x)}_2
=\E_{y\sim\pi_\theta(\cdot\mid x)}\norm{S_\theta(y\mid x)}_2
\le\sqrt{2LT\log|\cA|},
\]
where Cauchy--Schwarz and \cref{eq:scoremoment} give the last inequality. This
uniform bound, together with $|r_\phi|\le R$, dominates the first difference
quotients of $J$. Therefore
$\nabla J(\theta)=\E_x\sum_y r_\phi(x,y)\nabla\pi_\theta(y\mid x)$, and the
promptwise estimate above gives
\begin{align*}
\norm{\nabla J(\theta)-\nabla J(\theta')}_2
&\le R\E_x\sum_y
\norm{\nabla\pi_\theta(y\mid x)-\nabla\pi_{\theta'}(y\mid x)}_2\\
&\le TRL(2\log|\cA|+1)\norm{\theta-\theta'}_2.
\end{align*}
To obtain the claimed one-step inequality, integrate once more along
$\theta_n+s\eta g_n$:
\begin{align*}
J(\theta_{n+1})-J(\theta_n)-\eta\ip{\nabla J(\theta_n)}{g_n}
&=\eta\int_0^1\ip{\nabla J(\theta_n+s\eta g_n)-\nabla J(\theta_n)}{g_n}\,\mathrm{d}s\\
&\ge-TRL(2\log|\cA|+1)\eta^2\norm{g_n}_2^2
\int_0^1s\,\mathrm{d}s\\
&=-\frac{TRL(2\log|\cA|+1)\eta^2}{2}\norm{g_n}_2^2.
\end{align*}
Rearranging proves \cref{eq:objective-smoothness}.
\end{proof}

\begin{proof}[Proof of Lemma~\ref{lem:master-descent}]
Conditioning \cref{eq:objective-smoothness} on $\cF_n^-$, substituting
$\E[g_n\mid\cF_n^-]=T^{-1}\nabla J(\theta_n)+b_n$, and using that
$\theta_n$ is $\cF_n^-$-measurable gives the first line below. For the second
line, apply Young's inequality in the explicit form
\[
\ip{u}{v}\ge-\frac1{2T}\norm{u}_2^2-\frac T2\norm{v}_2^2
\qquad (u,v\in\R^d)
\]
with $u=\nabla J(\theta_n)$ and $v=b_n$:
\begin{align*}
\E[J(\theta_{n+1})\mid\cF_n^-]-J(\theta_n)
&\ge\frac\eta T\norm{\nabla J(\theta_n)}_2^2
+\eta\ip{\nabla J(\theta_n)}{b_n}
-\frac{TRL(2\log|\cA|+1)\eta^2}{2}
\E[\norm{g_n}_2^2\mid\cF_n^-]\\
&\ge\frac\eta{2T}\norm{\nabla J(\theta_n)}_2^2
-\frac{\eta T}{2}\norm{b_n}_2^2
-\frac{TRL(2\log|\cA|+1)\eta^2}{2}
\E[\norm{g_n}_2^2\mid\cF_n^-].
\end{align*}
After taking expectations, move the objective increment to the right and sum
from $n=0$ to $N-1$. The increments telescope according to
\[
\sum_{n=0}^{N-1}\E[J(\theta_{n+1})-J(\theta_n)]
=\E[J(\theta_N)]-\E[J(\theta_0)]
\le J^\star-\E[J(\theta_0)].
\]
Consequently,
\[
\frac\eta{2T}\sum_{n=0}^{N-1}\E\norm{\nabla J(\theta_n)}_2^2
\le J^\star-\E[J(\theta_0)]
+\frac{\eta T}{2}\sum_{n=0}^{N-1}\E\norm{b_n}_2^2
+\frac{TRL(2\log|\cA|+1)\eta^2}{2}
\sum_{n=0}^{N-1}\E\norm{g_n}_2^2.
\]
Dividing by $N$ and multiplying by $2T/\eta$ gives exactly
Lemma \ref{lem:master-descent}.
\end{proof}

\section{Optimization over a Second-Moment-Controlled Weighting Class}
\label{sec:weighted-class}

This section justifies the weighting rule underlying GMC-GRPO. We first prove
that the admissible mass constraint gives a uniform second-moment bound and,
under coverage, bounds the correction bias by the removed importance mass.
We then show that GMC uniquely maximizes retained mass among admissible
direction-preserving corrections. Comparing the mass removed by GMC and TIC
establishes the corresponding ordering of their certified bias envelopes,
providing the basis for the convergence analysis that follows.

\begin{proof}[Proof of Proposition \ref{prop:weighted-second-moment}]
For brevity, denote the contribution of group $b$ to \cref{eq:weighted-update}
by
\[
U_{n,b}^{w}:=\frac1{TG}\sum_{i=1}^G w_n^{(b,i)}\rho_n^{(b,i)}
\nabla\log\pi_{\theta_n}(y^{(b,i)}\mid x^{(b)})\widehat A^{(b,i)},
\qquad
g_n^w=\frac1B\sum_{b=1}^B U_{n,b}^w.
\]
For one group, weighted Cauchy--Schwarz and
\cref{eq:admissible-weights} give the pathwise bound
\begin{align*}
\norm{U_{n,b}^{w}}_2^2
&\le
\left(\frac1G\sum_{i=1}^G w_n^{(b,i)}\rho_n^{(b,i)}\right)\times\left(\frac1G\sum_{i=1}^G w_n^{(b,i)}\rho_n^{(b,i)}
\norm{\frac1T\nabla\log\pi_{\theta_n}(y^{(b,i)}\mid x^{(b)})
\widehat A^{(b,i)}}_2^2\right)\\
&\le\frac{1+\eps}{G}\sum_{i=1}^G\rho_n^{(b,i)}
\norm{\frac1T\nabla\log\pi_{\theta_n}(y^{(b,i)}\mid x^{(b)})
\widehat A^{(b,i)}}_2^2.
\end{align*}
The first factor is at most $1+\eps$, while
$w_n^{(b,i)}\le1$ removes the weights from the second factor.  Taking
conditional expectation, using $|\widehat A^{(b,i)}|\le2R$, and changing
measure from the behavior policy to the current policy yield
\begin{align*}
\E[\norm{U_{n,b}^{w}}_2^2\mid\cH_n]
&\le\frac{4R^2(1+\eps)}{T^2G}\sum_{i=1}^G
\E\!\left[\rho_n^{(b,i)}
\norm{\nabla\log\pi_{\theta_n}(y^{(b,i)}\mid x^{(b)})}_2^2
\,\middle|\,\cH_n\right]\\
&=\frac{4R^2(1+\eps)}{T^2}
\E_{y\sim\pi_{\theta_n}(\cdot\mid x^{(b)})}
\norm{\nabla\log\pi_{\theta_n}(y\mid x^{(b)})}_2^2\\
&\le\frac{8(1+\eps)LR^2\log|\cA|}{T},
\end{align*}
where the last line is \cref{eq:scoremoment}.  Finally, convexity of the
squared norm and the tower property give
\begin{align*}
\E[\norm{g_n^{w}}_2^2\mid\cF_n^-]
&\le\E\!\left[\frac1B\sum_{b=1}^B\norm{U_{n,b}^{w}}_2^2
\,\middle|\,\cF_n^-\right]\\
&=\E\!\left[\frac1B\sum_{b=1}^B
\E[\norm{U_{n,b}^{w}}_2^2\mid\cH_n]
\,\middle|\,\cF_n^-\right]
\le\frac{8(1+\eps)LR^2\log|\cA|}{T}.
\end{align*}
For the bias envelope, $w_n^{(b,i)}\le1$ implies
$(1-w_n^{(b,i)})\rho_n^{(b,i)}\ge0$. Conditional Jensen's inequality, the
triangle inequality, $|\widehat A^{(b,i)}|\le2R$, and
\cref{eq:coverage-score} then give
\begin{align*}
\norm{\E[b_{n,b}^{w}\mid\cH_n]}_2
&\le\frac{2R}{TG}\sum_{i=1}^G
\E\!\left[(1-w_n^{(b,i)})\rho_n^{(b,i)}
\norm{\nabla\log\pi_{\theta_n}(y^{(b,i)}\mid x^{(b)})}_2
\,\middle|\,\cH_n\right]\\
&\le2R\sqrt{2L\log\mathcal C}\,
\E\!\left[\frac1G\sum_{i=1}^G
(1-w_n^{(b,i)})\rho_n^{(b,i)}
\,\middle|\,\cH_n\right],
\end{align*}
which is \cref{eq:weighted-bias-envelope1}.
\end{proof}

\begin{proof}[Proof of Proposition~\ref{prop:gmc-retained-mass}]
For a direction-preserving correction with common weight $\alpha$, the
retained mass is $\alpha\overline\rho_n^{(b)}$. Since
$\overline\rho_n^{(b)}>0$, maximizing the retained mass---or, equivalently,
minimizing the removed mass---amounts to choosing the largest $\alpha$
satisfying $0\le\alpha\le1$ and
$\alpha\overline\rho_n^{(b)}\le1+\eps$. The unique optimizer is therefore
the stated $\alpha_n^{(b),\star}$, which is exactly the common scaling in
\cref{eq:gmcratio}.
\end{proof}

The same calculation also yields a comparison with TIC-GRPO.
GMC removes importance mass
$(\overline\rho_n^{(b)}-(1+\eps))_+$. This quantity is zero when
$\overline\rho_n^{(b)}\le1+\eps$. Otherwise,
\begin{align*}
\overline\rho_n^{(b)}-(1+\eps)
&=\frac1G\sum_{i:\,\rho_n^{(b,i)}>1+\eps}\rho_n^{(b,i)}
+\frac1G\sum_{i:\,\rho_n^{(b,i)}\le1+\eps}\rho_n^{(b,i)}-(1+\eps)\le\frac1G\sum_{i:\,\rho_n^{(b,i)}>1+\eps}\rho_n^{(b,i)},
\end{align*}
because the second normalized sum is at most $1+\eps$. Thus, in either case,
GMC removes no more importance mass than TIC-GRPO. Applying
\cref{eq:weighted-bias-envelope1} and taking conditional expectations gives
the corresponding comparison between their certified bias envelopes.

The scope of Proposition~\ref{prop:gmc-retained-mass} is deliberate. It does
not claim that GMC minimizes the exact vector bias over arbitrary
sample-specific rules: equal retained mass can still lead to different
cancellations and hence different biases. Rather, Propositions
\ref{prop:weighted-second-moment} and~\ref{prop:gmc-retained-mass} identify GMC
as the unique direction-preserving rule that retains the shared second-moment
guarantee while minimizing a certified bias envelope. TIC-GRPO enforces the
same mass cap by deleting individual high-ratio responses; GMC instead applies
the smallest uniform attenuation needed to satisfy the cap. The next section
uses this groupwise structure to derive the GMC conditional-bias and
convergence bounds.

\section{Convergence analysis of GMC-GRPO}\label{sec:lemmas}

This section establishes the convergence guarantee for GMC-GRPO under mixed
rollout delays. By Lemma~\ref{lem:master-descent}, it suffices to control the
update second moment and the squared conditional bias. We first decompose the
GMC update into the current-policy gradient and an explicit groupwise
correction term, and show that the group-mass cap bounds the update second
moment independently of the delay. We then bound the correction bias through
trajectory divergences between the current and behavior policies, and relate
these divergences to the parameter drift accumulated within each delay window.
Combining these estimates with Lemma~\ref{lem:master-descent} yields the
delay-explicit guarantee in Theorem~\ref{thm:coverage}. A final local-overlap
argument converts the chi-square branch into the group-size refinement used in
the main text. Throughout, conditional statements rely on the sampling and
filtration conditions in Assumption~\ref{as:async}.

For each group, define the raw importance-sampling, self-normalized, and
group-mass-capped estimators
\begin{align}
U_{n,b}^{\mathrm{IS}}
&:=\frac1{TG}\sum_{i=1}^G
\rho_n^{(b,i)}
\nabla\log\pi_{\theta_n}(y^{(b,i)}\mid x^{(b)})
\widehat A^{(b,i)},
\label{eq:isgroup}\\
U_{n,b}^{\mathrm{SN}}
&:=\frac1{TG}\sum_{i=1}^G
\frac{\rho_n^{(b,i)}}{\overline\rho_n^{(b)}}
\nabla\log\pi_{\theta_n}(y^{(b,i)}\mid x^{(b)})
\widehat A^{(b,i)},
\nonumber\\
U_{n,b}^{\mathrm{GMC}}
&:=\frac1{TG}\sum_{i=1}^G
\widehat\rho_n^{(b,i)}
\nabla\log\pi_{\theta_n}(y^{(b,i)}\mid x^{(b)})
\widehat A^{(b,i)}.
\nonumber
\end{align}

\begin{proposition}[Conditional gradient decomposition]
\label{prop:decomposition}
Under Assumption~\ref{as:lip} and the sampling conditions in
Assumption~\ref{as:async}, for every $n,b$, the following pathwise identities
hold:
\begin{align}
\frac1G\sum_{i=1}^G\widehat\rho_n^{(b,i)}
&=\min\{\overline\rho_n^{(b)},1+\eps\},
\nonumber\\
U_{n,b}^{\mathrm{GMC}}
&=
\begin{cases}
U_{n,b}^{\mathrm{IS}},&\overline\rho_n^{(b)}\le 1+\eps,\\
(1+\eps)U_{n,b}^{\mathrm{SN}},&\overline\rho_n^{(b)}>1+\eps,
\end{cases}
\label{eq:tworegimes}\\
U_{n,b}^{\mathrm{GMC}}-U_{n,b}^{\mathrm{IS}}
&=-(\overline\rho_n^{(b)}-(1+\eps))_+U_{n,b}^{\mathrm{SN}}.
\label{eq:biasidentity}
\end{align}
Moreover,
$0\le\widehat\rho_n^{(b,i)}\le\rho_n^{(b,i)}$ for every $i$.
Define
\begin{align}
b_{n,b}
&:=-\E\left[
(\overline\rho_n^{(b)}-(1+\eps))_+U_{n,b}^{\mathrm{SN}}
\,\middle|\,\cH_n\right],
\label{eq:groupbiasdef}\\
b_n
&:=\E\left[\frac1B\sum_{b=1}^B b_{n,b}
\,\middle|\,\cF_n^-\right].
\label{eq:globalbiasdef}
\end{align}
Then
\begin{align}
\E[U_{n,b}^{\mathrm{GMC}}\mid\cH_n]
&=\frac1T\nabla J_{x^{(b)}}(\theta_n)+b_{n,b},
\label{eq:groupdecomp}\\
\E[g_n^{\mathrm{GMC}}\mid\cF_n^-]
&=\frac1T\nabla J(\theta_n)+b_n,
\label{eq:meandecomp}
\end{align}
Thus every subsequent bias lemma only needs to control the explicit correction
term in \cref{eq:groupbiasdef}.
\end{proposition}

\begin{proof}
The definition \cref{eq:gmcratio} gives the pathwise chain
\begin{align*}
\frac1G\sum_{i=1}^G\widehat\rho_n^{(b,i)}
&=\frac{\overline\rho_n^{(b)}}
{\max\{1,\overline\rho_n^{(b)}/(1+\eps)\}}
=\min\{\overline\rho_n^{(b)},1+\eps\},\\
U_{n,b}^{\mathrm{GMC}}
&=\min\{\overline\rho_n^{(b)},1+\eps\}U_{n,b}^{\mathrm{SN}},
\qquad U_{n,b}^{\mathrm{IS}}=\overline\rho_n^{(b)}U_{n,b}^{\mathrm{SN}},\\
U_{n,b}^{\mathrm{GMC}}-U_{n,b}^{\mathrm{IS}}
&=-(\overline\rho_n^{(b)}-(1+\eps))_+U_{n,b}^{\mathrm{SN}}.
\end{align*}
This proves \cref{eq:tworegimes}--\cref{eq:biasidentity}; the denominator in
\cref{eq:gmcratio} is at least one, so also
$0\le\widehat\rho_n^{(b,i)}\le\rho_n^{(b,i)}$.

It remains to compute the mean of the raw estimator. Since
$\rho_n^{(b,i)}$ is the current-to-behavior likelihood ratio, it cancels the
behavior-policy density conditional on $\cH_n$. Hence, for the reward paired
with the $i$th score,
\begin{align*}
\E\!\left[\rho_n^{(b,i)}
\nabla\log\pi_{\theta_n}(y^{(b,i)}\mid x^{(b)})r^{(b,i)}\mid\cH_n\right]
&=\E_{y\sim\pi_{\theta_n}(\cdot\mid x^{(b)})}\!\left[
\nabla\log\pi_{\theta_n}(y\mid x^{(b)})r_\phi(x^{(b)},y)\right]\\
&=\sum_y\pi_{\theta_n}(y\mid x^{(b)})
\nabla\log\pi_{\theta_n}(y\mid x^{(b)})r_\phi(x^{(b)},y)
=\nabla J_{x^{(b)}}(\theta_n).
\end{align*}
For a leave-one-out reward $r^{(b,j)}$ with $j\ne i$, conditional independence
separates the two sampled responses. The score factor has conditional mean zero:
\begin{align*}
\E\!\left[\rho_n^{(b,i)}
\nabla\log\pi_{\theta_n}(y^{(b,i)}\mid x^{(b)})\mid\cH_n\right]
&=\sum_y\pi_{\theta_n}(y\mid x^{(b)})
\nabla\log\pi_{\theta_n}(y\mid x^{(b)})\\
&=\sum_y\nabla\pi_{\theta_n}(y\mid x^{(b)})=0,\\
\E\!\left[\rho_n^{(b,i)}
\nabla\log\pi_{\theta_n}(y^{(b,i)}\mid x^{(b)})r^{(b,j)}\mid\cH_n\right]
&=0\cdot\E[r^{(b,j)}\mid\cH_n]=0.
\end{align*}
Substitution of the leave-one-out advantage therefore yields
\begin{align*}
\E[U_{n,b}^{\mathrm{IS}}\mid\cH_n]
&=\frac1{TG}\sum_{i=1}^G\left(
\nabla J_{x^{(b)}}(\theta_n)-\frac1{G-1}\sum_{j\ne i}0\right)
=\frac1T\nabla J_{x^{(b)}}(\theta_n),\\
\E[U_{n,b}^{\mathrm{GMC}}\mid\cH_n]
&=\E[U_{n,b}^{\mathrm{IS}}\mid\cH_n]
-\E\!\left[(\overline\rho_n^{(b)}-(1+\eps))_+
U_{n,b}^{\mathrm{SN}}\mid\cH_n\right]
=\frac1T\nabla J_{x^{(b)}}(\theta_n)+b_{n,b},\\
\E[g_n^{\mathrm{GMC}}\mid\cF_n^-]
&=\E\!\left[\frac1B\sum_{b=1}^B
\E[U_{n,b}^{\mathrm{GMC}}\mid\cH_n]\,\middle|\,\cF_n^-\right]\\
&=\frac1T\E\!\left[\frac1B\sum_{b=1}^B
\nabla J_{x^{(b)}}(\theta_n)\,\middle|\,\cF_n^-\right]+b_n
=\frac1T\nabla J(\theta_n)+b_n.
\end{align*}
The final equality is the conditional prompt-unbiasedness clause in
Assumption~\ref{as:async}. These are \cref{eq:groupdecomp} and
\cref{eq:meandecomp}; bounded rewards and \cref{eq:scoremoment} ensure
integrability by Cauchy--Schwarz.
\end{proof}

\begin{lemma}[Expected update second moment]\label{lem:secondmoment}
Under Assumption~\ref{as:lip} and the sampling conditions in
Assumption~\ref{as:async},
\begin{equation}\label{eq:gmax}
\E[\norm{g_n^{\mathrm{GMC}}}_2^2\mid\cF_n^-]
\le \frac{8(1+\eps)LR^2\log|\cA|}{T}.
\end{equation}
\end{lemma}

\begin{proof}
We use the weighted Cauchy--Schwarz inequality
\[
\norm{\sum_i a_iv_i}_2^2\le\left(\sum_i a_i\right)
\left(\sum_i a_i\norm{v_i}_2^2\right),
\qquad a_i\ge0.
\]
Apply it with $a_i=\widehat\rho_n^{(b,i)}/G$ and
$v_i=T^{-1}\nabla\log\pi_{\theta_n}(y^{(b,i)}\mid x^{(b)})
\widehat A^{(b,i)}$. Proposition~\ref{prop:decomposition} identifies the sum
of the weights, while
$\widehat\rho_n^{(b,i)}\le\rho_n^{(b,i)}$ enlarges the second factor. Hence
\begin{align*}
\norm{U_{n,b}^{\mathrm{GMC}}}_2^2
&=\norm{\frac1{TG}\sum_i
\widehat\rho_n^{(b,i)}
\nabla\log\pi_{\theta_n}(y^{(b,i)}\mid x^{(b)})
\widehat A^{(b,i)}}_2^2\\
&\le\left(\frac1G\sum_i\widehat\rho_n^{(b,i)}\right)
\left(\frac1G\sum_i\widehat\rho_n^{(b,i)}
\norm{\frac1T
\nabla\log\pi_{\theta_n}(y^{(b,i)}\mid x^{(b)})
\widehat A^{(b,i)}}_2^2\right)\\
&\le\frac{\min\{1+\eps,\overline\rho_n^{(b)}\}}G
\sum_i\rho_n^{(b,i)}
\norm{\frac1T
\nabla\log\pi_{\theta_n}(y^{(b,i)}\mid x^{(b)})
\widehat A^{(b,i)}}_2^2\\
&\le\frac{1+\eps}{G}\sum_i\rho_n^{(b,i)}
\norm{\frac1T
\nabla\log\pi_{\theta_n}(y^{(b,i)}\mid x^{(b)})
\widehat A^{(b,i)}}_2^2.
\end{align*}
The second inequality uses
$G^{-1}\sum_i\widehat\rho_n^{(b,i)}
=\min\{1+\eps,\overline\rho_n^{(b)}\}$ and
$\widehat\rho_n^{(b,i)}\le\rho_n^{(b,i)}$; the third uses
$\min\{1+\eps,\overline\rho_n^{(b)}\}\le1+\eps$.
Taking conditional expectation and using $|\widehat A^{(b,i)}|\le2R$ gives
\begin{align*}
\E[\norm{U_{n,b}^{\mathrm{GMC}}}_2^2\mid\cH_n]
&\le\frac{4R^2(1+\eps)}{T^2G}\sum_{i=1}^G
\E\!\left[\rho_n^{(b,i)}
\norm{\nabla\log\pi_{\theta_n}(y^{(b,i)}\mid x^{(b)})}_2^2
\,\middle|\,\cH_n\right]\\
&=\frac{4R^2(1+\eps)}{T^2}
\E_{y\sim\pi_{\theta_n}(\cdot\mid x^{(b)})}
\norm{\nabla\log\pi_{\theta_n}(y\mid x^{(b)})}_2^2\\
&\le\frac{8(1+\eps)LR^2\log|\cA|}{T}.
\end{align*}
The equality changes measure from the behavior policy to the current policy, and
the final inequality is \cref{eq:scoremoment}. Finally, convexity gives the
pathwise batch bound
\[
\norm{g_n^{\mathrm{GMC}}}_2^2
=\norm{\frac1B\sum_{b=1}^BU_{n,b}^{\mathrm{GMC}}}_2^2
\le\frac1B\sum_{b=1}^B\norm{U_{n,b}^{\mathrm{GMC}}}_2^2.
\]
Conditioning first on $\cH_n$ and then on $\cF_n^-$ proves \cref{eq:gmax}.
\end{proof}

\begin{lemma}[Bias controls]\label{lem:covbias}
Under Assumptions~\ref{as:lip} and~\ref{as:coverage} and the sampling
conditions in Assumption~\ref{as:async}, for every $n,b$,
\begin{equation}\label{eq:covbias}
\norm{b_{n,b}}_2^2
\le\min\Biggl\{
\frac{8R^2L\log|\cA|}{T},
\frac{8R^2L\log|\cA|}{T\ceps}D_n^{(b)},
\frac{8R^2L\log\mathcal{C}}{\log^2(1+2\eps)}(D_n^{(b)})^2,
\frac{R^2L\log\mathcal{C}}{2G^2\eps^2}(\chi_n^{2,(b)})^2
\Biggr\}.
\end{equation}
\end{lemma}

\begin{proof}
Let $\varphi(u):=u\log u-u+1$ and abbreviate
$E_b:=\{\overline\rho_n^{(b)}>1+\eps\}$. Expanding the self-normalized update in
\cref{eq:groupbiasdef},
\[
b_{n,b}=-\frac1T\E\!\left[
\frac{(\overline\rho_n^{(b)}-(1+\eps))_+}{\overline\rho_n^{(b)}}
\frac1G\sum_{i=1}^G\rho_n^{(b,i)}
\nabla\log\pi_{\theta_n}(y^{(b,i)}\mid x^{(b)})\widehat A^{(b,i)}
\,\middle|\,\cH_n\right].
\]
The scalar prefactor is at most $\mathbf1_{E_b}$. Conditional Jensen and
$|\widehat A^{(b,i)}|\le2R$ therefore imply
\[
\norm{b_{n,b}}_2
\le\frac{2R}{T}\frac1G\sum_{i=1}^G
\E\!\left[\mathbf1_{E_b}\rho_n^{(b,i)}
\norm{\nabla\log\pi_{\theta_n}(y^{(b,i)}\mid x^{(b)})}_2
\,\middle|\,\cH_n\right].
\]
For each $i$, conditional Cauchy--Schwarz bounds the expectation inside the sum
as
\[
\begin{aligned}
\E\!\left[\mathbf1_{E_b}\rho_n^{(b,i)}
\norm{\nabla\log\pi_{\theta_n}(y^{(b,i)}\mid x^{(b)})}_2
\,\middle|\,\cH_n\right]
&\le\E[\mathbf1_{E_b}\rho_n^{(b,i)}\mid\cH_n]^{1/2}\times\E\!\left[\rho_n^{(b,i)}
\norm{\nabla\log\pi_{\theta_n}(y^{(b,i)}\mid x^{(b)})}_2^2
\,\middle|\,\cH_n\right]^{1/2}.
\end{aligned}
\]
A second Cauchy--Schwarz inequality over $i$ multiplies the averages of the two
right-hand factors. The first average is
$\E[\overline\rho_n^{(b)}\mathbf1_{E_b}\mid\cH_n]$, so we obtain
\begin{align*}
\norm{b_{n,b}}_2^2
&\le\frac{4R^2}{T^2}
\E\!\left[\overline\rho_n^{(b)}
\mathbf1\{\overline\rho_n^{(b)}>1+\eps\}\mid\cH_n\right]
\frac1G\sum_{i=1}^G\E\!\left[\rho_n^{(b,i)}
\norm{\nabla\log\pi_{\theta_n}(y^{(b,i)}\mid x^{(b)})}_2^2
\mid\cH_n\right]\\
&\le\frac{8R^2L\log|\cA|}{T}
\E\!\left[\overline\rho_n^{(b)}
\mathbf1\{\overline\rho_n^{(b)}>1+\eps\}\mid\cH_n\right].
\end{align*}
In the last line, change of measure turns each second conditional expectation in
the preceding display into the current-policy score moment, which is at most
$2LT\log|\cA|$ by \cref{eq:scoremoment}. Moreover,
$\E[\overline\rho_n^{(b)}\mid\cH_n]=1$ because each importance ratio has
conditional mean one. Dropping $\mathbf1_{E_b}$ proves the uniform branch.

For the KL branch,
$\varphi(u)/u=\log u-1+u^{-1}$ has derivative $(u-1)/u^2>0$ for $u>1$.
Its value at $u=1+\eps$ is
$\ceps=\log(1+\eps)-\eps/(1+\eps)$, so the pointwise inequality
\[
\ceps u\mathbf1\{u>1+\eps\}\le\varphi(u),
\qquad u\ge0,
\]
holds. Apply it to $u=\overline\rho_n^{(b)}$, then use convexity of $\varphi$:
\begin{align*}
\ceps\,\E\!\left[\overline\rho_n^{(b)}
\mathbf1\{\overline\rho_n^{(b)}>1+\eps\}\mid\cH_n\right]
&\le\E[\varphi(\overline\rho_n^{(b)})\mid\cH_n]\le\frac1G\sum_{i=1}^G\E[\varphi(\rho_n^{(b,i)})\mid\cH_n]
=D_n^{(b)}.
\end{align*}
For the last equality, change of measure gives
$\E[\varphi(\rho_n^{(b,i)})\mid\cH_n]
=\KL(\pi_{\theta_n}\|\pi_{\theta_{n-\tau_n^{(b)}}})=D_n^{(b)}$.
Substitution in the preceding bound proves the KL branch.

The first two branches do not use Assumption \ref{as:coverage}. For the
remaining branches, coverage and
\cref{eq:coverage-score} give
$\norm{\nabla\log\pi_{\theta_n}(y^{(b,i)}\mid x^{(b)})}_2
\le T\sqrt{2L\log\mathcal C}$ for every sampled response. Since the normalized
weights sum to one, the triangle inequality and
$|\widehat A^{(b,i)}|\le2R$ yield the pathwise bound
\[
\norm{U_{n,b}^{\mathrm{SN}}}_2
\le\frac{2R}{TG}\sum_{i=1}^G
\frac{\rho_n^{(b,i)}}{\overline\rho_n^{(b)}}
\norm{\nabla\log\pi_{\theta_n}(y^{(b,i)}\mid x^{(b)})}_2
\le2R\sqrt{2L\log\mathcal C}.
\]
Using this inside \cref{eq:groupbiasdef} and applying conditional Jensen gives
\begin{equation}\label{eq:excessbias}
\norm{b_{n,b}}_2
\le\E\!\left[(\overline\rho_n^{(b)}-(1+\eps))_+
\norm{U_{n,b}^{\mathrm{SN}}}_2\,\middle|\,\cH_n\right]
\le2R\sqrt{2L\log\mathcal C}\,
\E[(\overline\rho_n^{(b)}-(1+\eps))_+\mid\cH_n].
\end{equation}

For $u\ge1+\eps$, convexity and the tangent to $\varphi$ at $1+2\eps$ give
\begin{align*}
\varphi(u)
&\ge \varphi(1+2\eps)+\varphi'(1+2\eps)(u-1-2\eps)\\
&=\log(1+2\eps)u-2\eps\\
&\ge\log(1+2\eps)(u-(1+\eps)).
\end{align*}
The last inequality is equivalent to
$(1+\eps)\log(1+2\eps)\ge2\eps$, which follows from
$\log v\ge2(v-1)/(v+1)$ with $v=1+2\eps$. Applying the resulting pointwise
bound to $u=\overline\rho_n^{(b)}$, followed by Jensen and change of measure,
gives
\begin{align*}
\log(1+2\eps)\,
\E\left[(\overline\rho_n^{(b)}-(1+\eps))_+\,\middle|\,\cH_n\right]
&\le\E[\varphi(\overline\rho_n^{(b)})\mid\cH_n]\le\frac1G\sum_{i=1}^G\E[\varphi(\rho_n^{(b,i)})\mid\cH_n]
=D_n^{(b)}.
\end{align*}
Substitution in \cref{eq:excessbias} proves the quadratic-KL branch.

For the chi-square branch, the scalar inequality
\[
(u-1-\eps)_+\le\frac{(u-1)^2}{4\eps},\qquad u\ge0,
\]
is immediate when $u\le1+\eps$; when $u>1+\eps$, it is equivalent to
$(u-1-2\eps)^2\ge0$. Next expand the conditional variance of the group mean.
Because the responses are conditionally independent and each importance ratio
has conditional mean one, for $i\ne j$,
\[
\E[(\rho_n^{(b,i)}-1)(\rho_n^{(b,j)}-1)\mid\cH_n]=0.
\]
Therefore
\begin{align*}
\E\left[(\overline\rho_n^{(b)}-1)^2\,\middle|\,\cH_n\right]
&=\frac1{G^2}\sum_{i=1}^G
\E[(\rho_n^{(b,i)}-1)^2\mid\cH_n]
=\frac{\chi_n^{2,(b)}}G,\\
\E[(\overline\rho_n^{(b)}-(1+\eps))_+\mid\cH_n]
&\le\frac1{4\eps}
\E[(\overline\rho_n^{(b)}-1)^2\mid\cH_n]
=\frac{\chi_n^{2,(b)}}{4G\eps}.
\end{align*}
Substitution in \cref{eq:excessbias} proves the last branch.
\end{proof}

\begin{lemma}[KL--delay conversions]\label{lem:drift}
Suppose the iterates are generated by the GMC update \cref{eq:gmcgrad-main},
Assumption~\ref{as:lip} holds, and both the sampling and delay-regularity
conditions in Assumption~\ref{as:async} hold.
\begin{equation}\label{eq:expecteddrift}
\E[D_n^{(b)}\mid\tau_n^{(b)}]
\le 4(1+\eps)L^2R^2\log|\cA|\,
\eta^2(\tau_n^{(b)})^2.
\end{equation}
If Assumption~\ref{as:coverage} also holds, then
\begin{align}
D_n^{(b)}
&\le 4(1+\eps)^2TR^2L^2\log\mathcal{C}\,
\eta^2(\tau_n^{(b)})^2
\quad\text{almost surely},
\nonumber\\
\E[(D_n^{(b)})^2\mid\tau_n^{(b)}]
&\le 16(1+\eps)^3TL^4R^4\log|\cA|\log\mathcal{C}\,
\eta^4(\tau_n^{(b)})^4.
\label{eq:fourthdrift}
\end{align}
\end{lemma}

\begin{proof}
Fix $t\ge0$ with $\Pr(\tau_n^{(b)}=t)>0$. The learner recursion gives
\[
\theta_n-\theta_{n-t}=\eta\sum_{s=n-t}^{n-1}g_s^{\mathrm{GMC}},
\qquad
\norm{\theta_n-\theta_{n-t}}_2^2
\le\eta^2t\sum_{s=n-t}^{n-1}\norm{g_s^{\mathrm{GMC}}}_2^2,
\]
where the inequality is Cauchy--Schwarz over the $t$ learner updates.
The delay-regularity condition in Assumption~\ref{as:async} removes the
conditioning on the realized delay in the
first line below. The tower property and Lemma~\ref{lem:secondmoment} then bound
every update moment:
\begin{align*}
\E[\norm{\theta_n-\theta_{n-t}}_2^2\mid\tau_n^{(b)}=t]
&=\E\norm{\theta_n-\theta_{n-t}}_2^2\le\eta^2t\sum_{s=n-t}^{n-1}\E\norm{g_s^{\mathrm{GMC}}}_2^2
\le\frac{8(1+\eps)LR^2\log|\cA|}{T}\eta^2t^2,\\
\E[D_n^{(b)}\mid\tau_n^{(b)}=t]
&\le\frac{TL}{2}\E[\norm{\theta_n-\theta_{n-t}}_2^2\mid\tau_n^{(b)}=t]
\le4(1+\eps)L^2R^2\log|\cA|\,\eta^2t^2.
\end{align*}
The second line applies the trajectory-KL bound \cref{eq:trajectory-kl}; this
proves \cref{eq:expecteddrift}.

Under Assumption~\ref{as:coverage}, the capped group mass and
\cref{eq:coverage-score} give, for every group,
\begin{align*}
\norm{U_{s,b}^{\mathrm{GMC}}}_2
&\le\frac1{TG}\sum_{i=1}^G\widehat\rho_s^{(b,i)}
\norm{\nabla\log\pi_{\theta_s}(y^{(b,i)}\mid x^{(b)})}_2
|\widehat A^{(b,i)}|\le2R(1+\eps)\sqrt{2L\log\mathcal C}.
\end{align*}
The same bound holds for $g_s^{\mathrm{GMC}}$, because it is the average of
the group updates. The triangle inequality along the learner path therefore yields
\[
\norm{\theta_n-\theta_{n-t}}_2
\le\eta\sum_{s=n-t}^{n-1}\norm{g_s^{\mathrm{GMC}}}_2
\le2R(1+\eps)\sqrt{2L\log\mathcal C}\,\eta t.
\]
Combining this pathwise estimate with \cref{eq:trajectory-kl} gives
\begin{align*}
D_n^{(b)}
&\le\frac{TL}{2}\norm{\theta_n-\theta_{n-t}}_2^2
\le4(1+\eps)^2TR^2L^2\log\mathcal C\,\eta^2t^2=:D_{\max},\\
\E[(D_n^{(b)})^2\mid\tau_n^{(b)}=t]
&\le4(1+\eps)^2TR^2L^2\log\mathcal C\,\eta^2t^2
\E[D_n^{(b)}\mid\tau_n^{(b)}=t]\\
&\le16(1+\eps)^3TL^4R^4\log|\cA|\log\mathcal C\,\eta^4t^4,
\end{align*}
where the middle inequality uses $0\le D_n^{(b)}\le
4(1+\eps)^2TR^2L^2\log\mathcal C\,\eta^2t^2$, hence
$(D_n^{(b)})^2\le D_{\max}D_n^{(b)}$. This proves
\cref{eq:fourthdrift}.
\end{proof}

Define the local coverage penalty
\begin{equation}\label{eq:gammagmc}
\Gamma_{n,b}^{\mathrm{GMC}}
:=\min\Biggl\{
8LTR^2\log|\cA|,
\frac{8TR^2L\log|\cA|}{\ceps}D_n^{(b)},
\frac{8T^2R^2L\log\mathcal{C}}{\log^2(1+2\eps)}(D_n^{(b)})^2,
\frac{T^2R^2L\log\mathcal{C}}{2G^2\eps^2}(\chi_n^{2,(b)})^2
\Biggr\}.
\end{equation}
Since $D_n^{(b)}$ and $\chi_n^{2,(b)}$ are $\cH_n$-measurable,
$\Gamma_{n,b}^{\mathrm{GMC}}$ is a nonnegative $\cH_n$-measurable random
variable bounded by $8LTR^2\log|\cA|$.

\begin{theorem}[Delay-explicit convergence guarantee for GMC-GRPO]
\label{thm:coverage}
Suppose the iterates are generated by the GMC update \cref{eq:gmcgrad-main},
Assumptions~\ref{as:lip} and~\ref{as:coverage} hold, and the sampling conditions
in Assumption~\ref{as:async} hold. Then, for every $\eta>0$ and $N\ge1$,
\begin{align}
\frac1{N}\sum_{n=0}^{N-1}
\E\norm{\nabla J(\theta_n)}_2^2
\le{}&
\frac{2T\bigl(J^\star-\E[J(\theta_0)]\bigr)}{\eta N}
+8T(1+\eps)L^2R^3\log|\cA|\,(2\log|\cA|+1)\eta
\nonumber\\
&+\frac1{N}\sum_{n=0}^{N-1}
\E\left[\frac1B\sum_{b=1}^B
\Gamma_{n,b}^{\mathrm{GMC}}\right].
\label{eq:coveragethm}
\end{align}
If the delay-regularity condition in Assumption~\ref{as:async} also holds, the
final term is at most
\begin{equation}\label{eq:coveragecompact}
\begin{aligned}
\min\Biggl\{&
8LTR^2\log|\cA|,
\frac{32(1+\eps)L^3R^4T(\log|\cA|)^2}{\ceps}
\eta^2\overline\tau_2,
\frac{128(1+\eps)^3L^5R^6T^3\log|\cA|(\log\mathcal C)^2}
{\log^2(1+2\eps)}\eta^4\overline\tau_4,\\[-1mm]
&\frac{T^2R^2L\log\mathcal C}{2G^2\eps^2}
\frac1{N}\sum_{n=0}^{N-1}
\E\left[\frac1B\sum_{b=1}^B(\chi_n^{2,(b)})^2\right]
\Biggr\},
\end{aligned}
\end{equation}
where a branch equal to $+\infty$ is simply inactive in the minimum.
\end{theorem}

Each branch of \cref{eq:coveragecompact} is a valid bound on its own. The
uniform branch shows that the correction bias always contributes a bounded
amount. The second- and fourth-order branches make the dependence on rollout
delay explicit through $\overline\tau_2$ and $\overline\tau_4$. The chi-square
branch yields the group-size refinement in Corollary~2. Theorem~1 in the main
text retains the fourth-order branch.

\begin{proof}
The proof has three steps: we verify the hypothesis of
Lemma~\ref{lem:master-descent}, bound its update-moment and bias terms, and
finally convert the bias penalty into delay moments.

\emph{Step 1: conditional-mean structure.}
By \cref{eq:meandecomp},
$\E[g_n^{\mathrm{GMC}}\mid\cF_n^-]=T^{-1}\nabla J(\theta_n)+b_n$, where $b_n$
is given by \cref{eq:globalbiasdef} and is $\cF_n^-$-measurable. The
second-moment bound in Lemma~\ref{lem:secondmoment} and the uniform branch in
\cref{eq:covbias} show that $g_n^{\mathrm{GMC}}$ and $b_n$ are square
integrable. Hence Lemma~\ref{lem:master-descent} applies and gives
\begin{align*}
\frac1{N}\sum_{n=0}^{N-1}\E\norm{\nabla J(\theta_n)}_2^2
\le{}&\frac{2T(J^\star-\E[J(\theta_0)])}{\eta N}
+\frac{T^2RL(2\log|\cA|+1)\eta}{N}
\sum_{n=0}^{N-1}\E\norm{g_n^{\mathrm{GMC}}}_2^2
+\frac{T^2}{N}\sum_{n=0}^{N-1}\E\norm{b_n}_2^2.
\end{align*}

\emph{Step 2: update moment and bias.}
Taking expectations in Lemma~\ref{lem:secondmoment} gives
$\E\norm{g_n^{\mathrm{GMC}}}_2^2\le 8(1+\eps)LR^2\log|\cA|/T$ for every $n$.
Therefore
\begin{align*}
\frac{T^2RL(2\log|\cA|+1)\eta}{N}
\sum_{n=0}^{N-1}\E\norm{g_n^{\mathrm{GMC}}}_2^2
&\le T^2RL(2\log|\cA|+1)\eta
\cdot\frac{8(1+\eps)LR^2\log|\cA|}{T}\\
&=8T(1+\eps)L^2R^3\log|\cA|(2\log|\cA|+1)\eta.
\end{align*}
For the bias term, substitute \cref{eq:globalbiasdef}. Conditional Jensen for
the squared norm, followed by convexity of the squared norm over the uniform
average of the $B$ groups, gives
\begin{align*}
\norm{b_n}_2^2
=\norm{\E\!\left[\frac1B\sum_{b=1}^Bb_{n,b}\,\middle|\,\cF_n^-\right]}_2^2
\le\E\!\left[\norm{\frac1B\sum_{b=1}^Bb_{n,b}}_2^2\,\middle|\,\cF_n^-\right]
\le\E\!\left[\frac1B\sum_{b=1}^B\norm{b_{n,b}}_2^2\,\middle|\,\cF_n^-\right].
\end{align*}
The bound \cref{eq:covbias} in Lemma~\ref{lem:covbias} holds pathwise for every
$(n,b)$, and $\Gamma_{n,b}^{\mathrm{GMC}}$ is exactly $T^2$ times its
four-branch minimum. Taking expectations therefore yields
\[
T^2\E\norm{b_n}_2^2
\le T^2\E\left[\frac1B\sum_{b=1}^B\norm{b_{n,b}}_2^2\right]
\le\E\left[\frac1B\sum_{b=1}^B\Gamma_{n,b}^{\mathrm{GMC}}\right].
\]
Substituting the moment and bias bounds into the inequality of Step~1 proves
\cref{eq:coveragethm}.

\emph{Step 3: conversion to delay moments.}
Suppose now that the delay-regularity condition in Assumption~\ref{as:async}
also holds. For nonnegative random variables $X_1,\dots,X_4$,
$\E[\min_kX_k]\le\min_k\E[X_k]$. Hence the final term of
\cref{eq:coveragethm} is at most the minimum of the four averaged branches, and
it suffices to bound each branch separately.

The uniform branch is the constant $8LTR^2\log|\cA|$. For the linear KL branch,
the tower property and \cref{eq:expecteddrift} give
$\E[D_n^{(b)}]=\E\bigl[\E[D_n^{(b)}\mid\tau_n^{(b)}]\bigr]
\le4(1+\eps)L^2R^2\log|\cA|\,\eta^2\,\E[(\tau_n^{(b)})^2]$.
Averaging over $b$ and $n$ and using the definition \cref{eq:average-delay-moment} of
$\overline\tau_2$,
\[
\frac1{N}\sum_{n=0}^{N-1}
\E\left[\frac1B\sum_{b=1}^B
\frac{8TR^2L\log|\cA|}{\ceps}D_n^{(b)}\right]
\le\frac{32(1+\eps)L^3R^4T(\log|\cA|)^2}{\ceps}
\eta^2\overline\tau_2.
\]
For the quadratic-KL branch, the tower property and \cref{eq:fourthdrift} give
$\E[(D_n^{(b)})^2]
\le16(1+\eps)^3TL^4R^4\log|\cA|\log\mathcal C\,\eta^4\,\E[(\tau_n^{(b)})^4]$.
Averaging in the same way,
\[
\frac1{N}\sum_{n=0}^{N-1}
\E\left[\frac1B\sum_{b=1}^B
\frac{8T^2R^2L\log\mathcal C}{\log^2(1+2\eps)}(D_n^{(b)})^2\right]
\le\frac{128(1+\eps)^3L^5R^6T^3\log|\cA|(\log\mathcal C)^2}
{\log^2(1+2\eps)}\eta^4\overline\tau_4.
\]
The chi-square branch is kept in its averaged form
\[
\frac{T^2R^2L\log\mathcal C}{2G^2\eps^2}
\frac1{N}\sum_{n=0}^{N-1}
\E\left[\frac1B\sum_{b=1}^B(\chi_n^{2,(b)})^2\right].
\]
Taking the minimum of these four upper bounds proves
\cref{eq:coveragecompact}.
\end{proof}

\begin{corollary}[Local overlap converts chi-square to KL]
\label{cor:local-overlap-kl}

Suppose additionally that, for some finite $\delta\ge0$, the trajectory
likelihood ratio satisfies
$|\rho_n^{(b,i)}-1|\le\delta$ almost surely for all $n,b,i$. Then
$\chi_n^{2,(b)}\le2(1+\delta)D_n^{(b)}$ almost surely for every $n,b$.
\end{corollary}

\begin{proof}
Let $\varphi(u)=u\log u-u+1$. Softmax full support implies
$\rho_n^{(b,i)}>0$ almost surely. For any $u>0$ satisfying
$|u-1|\le\delta$ and every $s\in[0,1]$,
$0<1+s(u-1)\le\max\{1,u\}\le1+\delta$. Hence
$\varphi''(1+s(u-1))\ge(1+\delta)^{-1}$. Since
$\varphi(1)=\varphi'(1)=0$, Taylor's formula with integral remainder gives
\[
\varphi(u)=(u-1)^2\int_0^1(1-s)
\varphi''(1+s(u-1))\,\mathrm{d}s
\ge\frac{(u-1)^2}{2(1+\delta)}.
\]
Apply this pointwise to $u=\rho_n^{(b,i)}$ and take expectation under the
behavior policy:
\[
D_n^{(b)}=\E[\varphi(\rho_n^{(b,i)})\mid\cH_n]
\ge\frac{1}{2(1+\delta)}
\E[(\rho_n^{(b,i)}-1)^2\mid\cH_n]
=\frac{\chi_n^{2,(b)}}{2(1+\delta)},
\]
which proves the claim.
\end{proof}

\section{Convergence Analysis of TIC-GRPO}\label{app:vanilla}

The TIC-GRPO analysis follows the same proof structure, making the effect of
individual ratio truncation explicit. We decompose its update into the raw
importance-sampling estimator and the discarded contribution, establish the
shared second-moment bound, and bound the clipping bias using trajectory
divergences. Combining these estimates with the parameter-drift argument and
Lemma~\ref{lem:master-descent} yields the TIC convergence bound. The resulting
coefficients allow a direct comparison with GMC under the same assumptions.

For every group, define the individually up-clipped estimator
\begin{equation}\label{eq:appvangroup}
U_{n,b}^{\mathrm{TIC}}
:=
\frac1{TG}\sum_{i=1}^G
\rho_n^{(b,i)}
\mathbf1\{\rho_n^{(b,i)}\le 1+\eps\}
\nabla\log\pi_{\theta_n}(y^{(b,i)}\mid x^{(b)})
\widehat A^{(b,i)},
\end{equation}
and the full mini-batch update
\begin{equation}\label{eq:appvangrad}
g_n^{\mathrm{TIC}}
:=\frac1B\sum_{b=1}^B U_{n,b}^{\mathrm{TIC}},
\qquad
\theta_{n+1}=\theta_n+\eta g_n^{\mathrm{TIC}}.
\end{equation}

\begin{proposition}[TIC-GRPO conditional gradient decomposition]
\label{prop:appvandecomp}
Under Assumption~\ref{as:lip} and the sampling conditions in
Assumption~\ref{as:async}, let
$U_{n,b}^{\mathrm{IS}}$ be the raw estimator in \cref{eq:isgroup}. Pathwise,
\begin{equation}\label{eq:appvanerror}
U_{n,b}^{\mathrm{TIC}}-U_{n,b}^{\mathrm{IS}}
=-\frac1{TG}\sum_{i=1}^G
\rho_n^{(b,i)}\mathbf1\{\rho_n^{(b,i)}>1+\eps\}
\nabla\log\pi_{\theta_n}(y^{(b,i)}\mid x^{(b)})
\widehat A^{(b,i)}.
\end{equation}
Set $b_{n,b}^{\mathrm{TIC}}
:=\E[U_{n,b}^{\mathrm{TIC}}-U_{n,b}^{\mathrm{IS}}\mid\cH_n]$. Then
\begin{align}
\E[U_{n,b}^{\mathrm{TIC}}\mid\cH_n]
&=\frac1T\nabla J_{x^{(b)}}(\theta_n)
+b_{n,b}^{\mathrm{TIC}},
\label{eq:appvangroupdecomp}\\
\E[g_n^{\mathrm{TIC}}\mid\cF_n^-]
&=\frac1T\nabla J(\theta_n)+b_n^{\mathrm{TIC}},
\label{eq:appvanfulldecomp}
\end{align}
where $b_n^{\mathrm{TIC}}:=
\E[B^{-1}\sum_{b=1}^B b_{n,b}^{\mathrm{TIC}}\mid\cF_n^-]$.
\end{proposition}

\begin{proof}
Subtracting \cref{eq:isgroup} from \cref{eq:appvangroup} gives
\cref{eq:appvanerror}. Proposition~\ref{prop:decomposition} already proves the
mean identity for the raw estimator, including cancellation of the
leave-one-out baseline. Hence
\begin{align*}
\E[U_{n,b}^{\mathrm{TIC}}\mid\cH_n]
&=\E[U_{n,b}^{\mathrm{IS}}\mid\cH_n]+b_{n,b}^{\mathrm{TIC}}
=\frac1T\nabla J_{x^{(b)}}(\theta_n)+b_{n,b}^{\mathrm{TIC}},\\
\E[g_n^{\mathrm{TIC}}\mid\cF_n^-]
&=\E\!\left[\frac1B\sum_{b=1}^B
\E[U_{n,b}^{\mathrm{TIC}}\mid\cH_n]\,\middle|\,\cF_n^-\right]
=\frac1T\nabla J(\theta_n)+b_n^{\mathrm{TIC}}.
\end{align*}
This proves \cref{eq:appvangroupdecomp}--\cref{eq:appvanfulldecomp}.
\end{proof}

\begin{remark}[The TIC bias does not depend on $G$]
Fix $\cH_n$ and write $\E_{\mathrm b}$ for expectation under the behavior
policy $\pi_{\theta_{n-\tau_n^{(b)}}}(\cdot\mid x^{(b)})$. Since the responses
in a group are conditionally i.i.d.\ and $r^{(b,j)}$ is independent of
$y^{(b,i)}$ for $j\ne i$, substituting the leave-one-out advantage into
\cref{eq:appvanerror} gives
\[
b_{n,b}^{\mathrm{TIC}}
=-\frac1T\,\E_{\mathrm b}\Bigl[\rho\,\mathbf 1\{\rho>1+\eps\}\,
\nabla\log\pi_{\theta_n}(y\mid x^{(b)})
\bigl(r_\phi(x^{(b)},y)-\E_{\mathrm b}[r_\phi(x^{(b)},y)]\bigr)\Bigr],
\]
where $\rho$ is the current-to-behavior ratio of $y$. The right-hand side
does not depend on $G$ and is nonzero in general. In contrast, the
chi-square branch of \cref{eq:covbias} shows that the GMC bias is $O\bigl(\chi_n^{2,(b)}/(G\eps)\bigr)$ and
vanishes as $G\to\infty$ for fixed behavior and current policies.
\end{remark}

\begin{lemma}[Expected TIC-GRPO update second moment]
\label{lem:appvanbounds}
Under Assumption~\ref{as:lip} and the sampling conditions in
Assumption~\ref{as:async},
\begin{equation}\label{eq:appvansecondmoment}
\E[\norm{g_n^{\mathrm{TIC}}}_2^2\mid\cF_n^-]
\le\frac{8(1+\eps)LR^2\log|\cA|}{T}.
\end{equation}
The GMC counterpart in Lemma~\ref{lem:secondmoment} has the same constant.
\end{lemma}

\begin{proof}
We use the weighted Cauchy--Schwarz inequality
\[
\norm{\sum_i a_iv_i}_2^2
\le\left(\sum_i a_i\right)
\left(\sum_i a_i\norm{v_i}_2^2\right),
\qquad a_i\ge0.
\]
Apply it with
$a_i=G^{-1}\rho_n^{(b,i)}
\mathbf1\{\rho_n^{(b,i)}\le1+\eps\}$ and
$v_i=T^{-1}\nabla\log\pi_{\theta_n}(y^{(b,i)}\mid x^{(b)})
\widehat A^{(b,i)}$. Since each retained ratio is at most $1+\eps$,
\begin{align*}
\norm{U_{n,b}^{\mathrm{TIC}}}_2^2
&\le
\left(\frac1G\sum_{i=1}^G
\rho_n^{(b,i)}\mathbf1\{\rho_n^{(b,i)}\le1+\eps\}\right)\times
\left(\frac1G\sum_{i=1}^G
\rho_n^{(b,i)}\mathbf1\{\rho_n^{(b,i)}\le1+\eps\}
\norm{\frac1T\nabla\log\pi_{\theta_n}(y^{(b,i)}\mid x^{(b)})
\widehat A^{(b,i)}}_2^2\right)\\
&\le\frac{1+\eps}{G}\sum_{i=1}^G\rho_n^{(b,i)}
\norm{\frac1T\nabla\log\pi_{\theta_n}(y^{(b,i)}\mid x^{(b)})
\widehat A^{(b,i)}}_2^2.
\end{align*}
Taking conditional expectation and using $|\widehat A^{(b,i)}|\le2R$ gives
\begin{align*}
\E[\norm{U_{n,b}^{\mathrm{TIC}}}_2^2\mid\cH_n]
&\le\frac{4R^2(1+\eps)}{T^2G}\sum_{i=1}^G
\E\!\left[\rho_n^{(b,i)}
\norm{\nabla\log\pi_{\theta_n}(y^{(b,i)}\mid x^{(b)})}_2^2
\,\middle|\,\cH_n\right]\\
&=\frac{4R^2(1+\eps)}{T^2}
\E_{y\sim\pi_{\theta_n}(\cdot\mid x^{(b)})}
\norm{\nabla\log\pi_{\theta_n}(y\mid x^{(b)})}_2^2\\
&\le\frac{8(1+\eps)LR^2\log|\cA|}{T}.
\end{align*}
The equality changes measure from the behavior policy to the current policy,
and the last inequality is \cref{eq:scoremoment}. Finally, convexity gives
\[
\norm{g_n^{\mathrm{TIC}}}_2^2
=\norm{\frac1B\sum_{b=1}^BU_{n,b}^{\mathrm{TIC}}}_2^2
\le\frac1B\sum_{b=1}^B\norm{U_{n,b}^{\mathrm{TIC}}}_2^2.
\]
Conditioning first on $\cH_n$ and then on $\cF_n^-$ proves
\cref{eq:appvansecondmoment}.
\end{proof}

\begin{lemma}[TIC-GRPO clipping-error bounds]
\label{lem:appvanbias}
Under Assumptions~\ref{as:lip} and~\ref{as:coverage} and the sampling
conditions in Assumption~\ref{as:async}, for every $n,b$,
\begin{equation}\label{eq:appvancovbias}
\norm{b_{n,b}^{\mathrm{TIC}}}_2^2
\le\min\Biggl\{
\frac{8R^2L\log|\cA|}{T},
\frac{8R^2L\log|\cA|}{T\ceps}D_n^{(b)},
\frac{8R^2L\log\mathcal{C}}{\ceps^2}(D_n^{(b)})^2,
\frac{8R^2L\log\mathcal{C}(1+\eps)^2}{\eps^4}
(\chi_n^{2,(b)})^2
\Biggr\}.
\end{equation}
\end{lemma}

\begin{proof}
Let $\varphi(u):=u\log u-u+1$. Expanding \cref{eq:appvanerror}, applying
conditional Jensen, and using $|\widehat A^{(b,i)}|\le2R$ give
\[
\norm{b_{n,b}^{\mathrm{TIC}}}_2
\le\frac{2R}{TG}\sum_{i=1}^G
\E\!\left[\rho_n^{(b,i)}\mathbf1\{\rho_n^{(b,i)}>1+\eps\}
\norm{\nabla\log\pi_{\theta_n}(y^{(b,i)}\mid x^{(b)})}_2
\,\middle|\,\cH_n\right].
\]
For each $i$, conditional Cauchy--Schwarz gives
\begin{align*}
&\E\!\left[\rho_n^{(b,i)}\mathbf1\{\rho_n^{(b,i)}>1+\eps\}
\norm{\nabla\log\pi_{\theta_n}(y^{(b,i)}\mid x^{(b)})}_2
\,\middle|\,\cH_n\right]^2\\
&\quad\le
\E\!\left[\rho_n^{(b,i)}\mathbf1\{\rho_n^{(b,i)}>1+\eps\}
\,\middle|\,\cH_n\right]
\E\!\left[\rho_n^{(b,i)}
\norm{\nabla\log\pi_{\theta_n}(y^{(b,i)}\mid x^{(b)})}_2^2
\,\middle|\,\cH_n\right].
\end{align*}
All $G$ responses have the same conditional law. Squaring the preceding bound,
changing measure in the score-moment factor, and applying
\cref{eq:scoremoment} therefore yield
\begin{align*}
\norm{b_{n,b}^{\mathrm{TIC}}}_2^2
&\le\frac{4R^2}{T^2}
\E[\rho_n^{(b,i)}\mathbf1\{\rho_n^{(b,i)}>1+\eps\}\mid\cH_n]
\E\!\left[\rho_n^{(b,i)}
\norm{\nabla\log\pi_{\theta_n}(y^{(b,i)}\mid x^{(b)})}_2^2
\mid\cH_n\right]\\
&\le\frac{8R^2L\log|\cA|}{T}
\E[\rho_n^{(b,i)}\mathbf1\{\rho_n^{(b,i)}>1+\eps\}\mid\cH_n].
\end{align*}
The tail mass is at most
$\E[\rho_n^{(b,i)}\mid\cH_n]=1$, which proves the uniform branch. For the KL
branch, $\varphi(u)/u=\log u-1+u^{-1}$ is increasing for $u>1$, and its value
at $u=1+\eps$ is $\ceps$. Hence
\[
\ceps u\mathbf1\{u>1+\eps\}\le\varphi(u),
\qquad u\ge0.
\]
Applying this inequality to $u=\rho_n^{(b,i)}$ and changing measure gives
\[
\ceps\,
\E[\rho_n^{(b,i)}\mathbf1\{\rho_n^{(b,i)}>1+\eps\}\mid\cH_n]
\le\E[\varphi(\rho_n^{(b,i)})\mid\cH_n]
=D_n^{(b)}.
\]
Substitution proves the KL branch.

The first two branches do not use Assumption \ref{as:coverage}. For the remaining branches, \cref{eq:coverage-score} bounds the score norm in the first display uniformly,
so
\begin{align*}
\norm{b_{n,b}^{\mathrm{TIC}}}_2
&\le2R\sqrt{2L\log\mathcal C}\,
\E[\rho_n^{(b,i)}\mathbf1\{\rho_n^{(b,i)}>1+\eps\}\mid\cH_n]\\
&\le2R\sqrt{2L\log\mathcal C}\,
\min\left\{\frac{D_n^{(b)}}{\ceps},
\frac{1+\eps}{\eps^2}\chi_n^{2,(b)}\right\},
\end{align*}
where the first alternative is the KL tail bound above. For the second, use the
pointwise inequality
\[
u\mathbf1\{u>1+\eps\}
\le\frac{1+\eps}{\eps^2}(u-1)^2.
\]
Indeed, $u/(u-1)^2$ is decreasing for $u>1$, so on the event
$u>1+\eps$ it is at most $(1+\eps)/\eps^2$. Taking behavior-policy expectation
turns $(u-1)^2$ into $\chi_n^{2,(b)}$. Squaring the two alternatives proves the
additional branches in \cref{eq:appvancovbias}.
\end{proof}

For each realized group, define the local TIC penalty
\begin{equation}\label{eq:gammavanilla}
\Gamma_{n,b}^{\mathrm{TIC}}
:=\min\Biggl\{
8LTR^2\log|\cA|,
\frac{8TR^2L\log|\cA|}{\ceps}D_n^{(b)},
\frac{8T^2R^2L\log\mathcal{C}}{\ceps^2}(D_n^{(b)})^2,
\frac{8T^2R^2L\log\mathcal{C}(1+\eps)^2}{\eps^4}
(\chi_n^{2,(b)})^2
\Biggr\}.
\end{equation}

\begin{theorem}[Delay-explicit convergence guarantee for TIC-GRPO]
\label{thm:appvanilla}
Suppose the iterates are generated by the TIC update \cref{eq:appvangrad},
Assumptions~\ref{as:lip} and~\ref{as:coverage} hold, and the sampling conditions
in Assumption~\ref{as:async} hold. Then, for every $\eta>0$ and $N\ge1$,
\begin{align}
\frac1{N}\sum_{n=0}^{N-1}
\E\norm{\nabla J(\theta_n)}_2^2
\le{}&
\frac{2T\bigl(J^\star-\E[J(\theta_0)]\bigr)}{\eta N}
+8T(1+\eps)L^2R^3\log|\cA|\,(2\log|\cA|+1)\eta
\nonumber\\
&+\frac1{N}\sum_{n=0}^{N-1}
\E\left[\frac1B\sum_{b=1}^B
\Gamma_{n,b}^{\mathrm{TIC}}\right],
\label{eq:appvantheorem}
\end{align}
where $\Gamma_{n,b}^{\mathrm{TIC}}$ is defined in
\cref{eq:gammavanilla}. If the delay-regularity condition in
Assumption~\ref{as:async} also holds, the final term is at most
\begin{equation}\label{eq:appvancompact}
\begin{aligned}
\min\Biggl\{&
8LTR^2\log|\cA|,
\frac{32(1+\eps)L^3R^4T(\log|\cA|)^2}{\ceps}
\eta^2\overline\tau_2,
\frac{128(1+\eps)^3L^5R^6T^3\log|\cA|(\log\mathcal C)^2}
{\ceps^2}\eta^4\overline\tau_4,\\[-1mm]
&\frac{8T^2R^2L\log\mathcal C(1+\eps)^2}{\eps^4}
\frac1{N}\sum_{n=0}^{N-1}
\E\left[\frac1B\sum_{b=1}^B(\chi_n^{2,(b)})^2\right]
\Biggr\}.
\end{aligned}
\end{equation}
\end{theorem}

\begin{proof}
\Cref{eq:appvanfulldecomp} verifies the conditional-mean hypothesis of
Lemma~\ref{lem:master-descent}. The update-moment bound
\cref{eq:appvansecondmoment} contributes
\[
T^2RL(2\log|\cA|+1)\eta
\,\cdot\,
\frac{8(1+\eps)LR^2\log|\cA|}{T}
=8T(1+\eps)L^2R^3\log|\cA|(2\log|\cA|+1)\eta.
\]
For the bias term, conditional Jensen and \cref{eq:appvancovbias} give
\[
T^2\E\norm{b_n^{\mathrm{TIC}}}_2^2
\le T^2\E\left[\frac1B\sum_{b=1}^B
\norm{b_{n,b}^{\mathrm{TIC}}}_2^2\right]
\le\E\left[\frac1B\sum_{b=1}^B
\Gamma_{n,b}^{\mathrm{TIC}}\right].
\]
Substituting both estimates into Lemma \ref{lem:master-descent} gives
\begin{align*}
\frac1{N}\sum_{n=0}^{N-1}\E\norm{\nabla J(\theta_n)}_2^2
&\le\frac{2T(J^\star-\E[J(\theta_0)])}{\eta N}
+8T(1+\eps)L^2R^3\log|\cA|(2\log|\cA|+1)\eta\\
&\quad+\frac1{N}\sum_{n=0}^{N-1}
\E\left[\frac1B\sum_{b=1}^B\Gamma_{n,b}^{\mathrm{TIC}}\right].
\end{align*}
This proves \cref{eq:appvantheorem}.

Now additionally impose the delay-regularity condition in
Assumption~\ref{as:async}. Fix $t\ge0$ with
$\Pr(\tau_n^{(b)}=t)>0$. The learner recursion and Cauchy--Schwarz over the
$t$ updates give
\[
\theta_n-\theta_{n-t}=\eta\sum_{s=n-t}^{n-1}g_s^{\mathrm{TIC}},
\qquad
\norm{\theta_n-\theta_{n-t}}_2^2
\le\eta^2t\sum_{s=n-t}^{n-1}\norm{g_s^{\mathrm{TIC}}}_2^2.
\]
This condition removes the conditioning on the realized delay, and
\cref{eq:appvansecondmoment} bounds every update moment. Therefore
\begin{align*}
\E[\norm{\theta_n-\theta_{n-t}}_2^2\mid\tau_n^{(b)}=t]
&=\E\norm{\theta_n-\theta_{n-t}}_2^2\le\frac{8(1+\eps)LR^2\log|\cA|}{T}\eta^2t^2,\\
\E[D_n^{(b)}\mid\tau_n^{(b)}=t]
&\le\frac{TL}{2}
\E[\norm{\theta_n-\theta_{n-t}}_2^2\mid\tau_n^{(b)}=t]\le4(1+\eps)L^2R^2\log|\cA|\,\eta^2t^2.
\end{align*}

Under coverage, every retained ratio is at most $1+\eps$, and
\cref{eq:coverage-score} gives, for every group,
\begin{align*}
\norm{U_{s,b}^{\mathrm{TIC}}}_2
&\le\frac{2R}{TG}\sum_{i=1}^G
\rho_s^{(b,i)}\mathbf1\{\rho_s^{(b,i)}\le1+\eps\}
\norm{\nabla\log\pi_{\theta_s}(y^{(b,i)}\mid x^{(b)})}_2\le2R(1+\eps)\sqrt{2L\log\mathcal C}.
\end{align*}
The same bound holds for $g_s^{\mathrm{TIC}}$, because it is the average of
the group updates. The triangle inequality along the learner path thus yields
\[
\norm{\theta_n-\theta_{n-t}}_2
\le\eta\sum_{s=n-t}^{n-1}\norm{g_s^{\mathrm{TIC}}}_2
\le2R(1+\eps)\sqrt{2L\log\mathcal C}\,\eta t.
\]
Combining this pathwise estimate with \cref{eq:trajectory-kl} gives
\begin{align*}
D_n^{(b)}
&\le\frac{TL}{2}\norm{\theta_n-\theta_{n-t}}_2^2
\le4(1+\eps)^2TR^2L^2\log\mathcal C\,\eta^2t^2=:D_{\max},\\
\E[(D_n^{(b)})^2\mid\tau_n^{(b)}=t]
&\le4(1+\eps)^2TR^2L^2\log\mathcal C\,\eta^2t^2
\E[D_n^{(b)}\mid\tau_n^{(b)}=t]\\
&\le16(1+\eps)^3TL^4R^4\log|\cA|\log\mathcal C\,\eta^4t^4.
\end{align*}
The middle inequality uses the almost-sure upper bound on $D_n^{(b)}$, namely
$(D_n^{(b)})^2\le D_{\max}D_n^{(b)}$.
For example, substituting the last line into the quadratic-KL branch gives
\begin{align*}
&\frac{8T^2R^2L\log\mathcal C}{\ceps^2}
\frac1{N}\sum_{n=0}^{N-1}
\E\left[\frac1B\sum_{b=1}^B(D_n^{(b)})^2\right]\le
\frac{128(1+\eps)^3L^5R^6T^3\log|\cA|(\log\mathcal C)^2}{\ceps^2}
\eta^4\overline\tau_4.
\end{align*}
The expected-drift line gives the second branch of \cref{eq:appvancompact}; the
uniform branch is constant, and the empirical chi-square branch is left in its
averaged form. The average pointwise minimum is no larger than any of these four
branchwise averages, so their minimum proves \cref{eq:appvancompact}. 

\end{proof}

\section{From the appendix bounds to the main-text rates}
\label{sec:main-specializations}

This section derives the main-text rates from the detailed bounds above.
We first prove a step-size balancing lemma for the optimization, stochastic,
and fourth-order delay terms, then apply it to the GMC bound. We next
optimize the GMC threshold asymptotically and use the local-overlap estimate
to obtain the group-size refinement. Applying the same balancing lemma to
TIC completes the main-text guarantees and makes their different
delay-dependent coefficients explicit.

\begin{lemma}[Balancing optimization, stochastic, and delay terms]
\label{lem:stepsize-balance}
Let $A,B>0$, $C\ge0$, and $M\ge1$. Suppose that, for every $\eta>0$,
\begin{equation}\label{eq:generic-rate}
F_M(\eta)\le \frac{A}{M\eta}+B\eta+C\eta^4,
\end{equation}
then, with $(A/(CM))^{1/5}:=+\infty$ when $C=0$, the choice
\begin{equation}\label{eq:generic-stepsize}
\eta^\star:=\min\left\{\sqrt{\frac{A}{BM}},
\left(\frac{A}{CM}\right)^{1/5}\right\},
\end{equation}
satisfies
\begin{equation}\label{eq:generic-balanced-rate}
F_M(\eta^\star)
\le2\sqrt{\frac{AB}{M}}
+2\left(\frac{A^4C}{M^4}\right)^{1/5}.
\end{equation}
\end{lemma}

\begin{proof}
If $C=0$, substituting $\eta^\star=\sqrt{A/(BM)}$ directly gives
$F_M(\eta^\star)\le2\sqrt{AB/M}$. Suppose henceforth that $C>0$.
The definition \cref{eq:generic-stepsize} implies
\begin{align*}
\frac{A}{M\eta^\star}
&=\max\left\{\sqrt{\frac{AB}{M}},
\left(\frac{A^4C}{M^4}\right)^{1/5}\right\}\le\sqrt{\frac{AB}{M}}
+\left(\frac{A^4C}{M^4}\right)^{1/5},\\
B\eta^\star&\le\sqrt{\frac{AB}{M}},
\qquad
C(\eta^\star)^4\le\left(\frac{A^4C}{M^4}\right)^{1/5},
\end{align*}
where the equality uses $1/\min\{u,v\}=\max\{1/u,1/v\}$ and the final
inequality uses $\max\{a,b\}\le a+b$ for $a,b\ge0$. Summing the three displayed
bounds proves \cref{eq:generic-balanced-rate}.
\end{proof}

\begin{proof}[Proof of Theorem~\ref{thm:gmc-main}]
Retain the quadratic-KL branch in \cref{eq:coveragecompact}. With $M=N$, the
result has the form \cref{eq:generic-rate} with
\begin{align*}
A&=2T\bigl(J^\star-\E[J(\theta_0)]\bigr),\\
B&=8T(1+\eps)L^2R^3\log|\cA|(2\log|\cA|+1),\\
C&=\frac{128(1+\eps)^3L^5R^6T^3\log|\cA|(\log\mathcal C)^2}
{\log^2(1+2\eps)}\overline\tau_4.
\end{align*}
This is the first bound in Theorem~\ref{thm:gmc-main}. Assume first that $A>0$
and $\overline\tau_4>0$. Lemma~\ref{lem:stepsize-balance} gives
\begin{align*}
\eta^\star
&=\min\left\{\sqrt{\frac{A}{BN}},
\left(\frac{A}{CN}\right)^{1/5}\right\},\\
\frac1{N}\sum_{n=0}^{N-1}\E\norm{\nabla J(\theta_n)}_2^2
&\le2\sqrt{\frac{AB}{N}}
+2\left(\frac{A^4C}{N^4}\right)^{1/5}.
\end{align*}
Substituting $C$ and suppressing the fixed problem constants gives
\[
\eta^\star\asymp\min\left\{
\frac1{\sqrt{(1+\eps)N}},
\left(\frac{\log^2(1+2\eps)}
{(1+\eps)^3\overline\tau_4N}\right)^{1/5}\right\},
\]
and the preceding optimized bound gives the second claim of
Theorem~\ref{thm:gmc-main}. When $\overline\tau_4=0$ and $A>0$, the $C=0$
branch of Lemma~\ref{lem:stepsize-balance} instead gives
$\eta^\star=\sqrt{A/(BN)}$ and the first term of the optimized bound. If
$A=0$, choose any positive step size of the main-text order and substitute it
directly into \cref{eq:generic-rate}; the optimization term is then zero and the
same stated upper order follows.
\end{proof}

\begin{proof}[Proof of Corollary~\ref{cor:gmc-threshold}]
Suppressing fixed positive problem-dependent coefficients, the leading
$\eps$-dependent terms in the optimized GMC bound have the form
\[
\mathcal B_N(\eps)
=aN^{-1/2}\sqrt{1+\eps}
+b\overline\tau_4^{1/5}N^{-4/5}
(1+\eps)^{3/5}\log^{-2/5}(1+2\eps),
\]
for constants $a,b>0$ independent of $N$, $\overline\tau_4$, and $\eps$.
The first term increases with $\eps$, while the second diverges as
$\eps\downarrow0$ and captures the correction-bias penalty.
For each $N$, $\mathcal B_N$ is continuous on $(0,\infty)$ and diverges
at both endpoints, so it attains a global minimum in $(0,\infty)$.

Let $h(\eps):=(1+\eps)^{3/5}\log^{-2/5}(1+2\eps)$.
Since $\log(1+2\eps)\sim2\eps$ as $\eps\downarrow0$,
\begin{align*}
-h'(\eps)
&=\frac{h(\eps)}{5}
\left(
\frac{4}{(1+2\eps)\log(1+2\eps)}-\frac{3}{1+\eps}
\right)
\asymp\eps^{-7/5}
\qquad(\eps\downarrow0).
\end{align*}
With $\overline\tau_4>0$ fixed, every minimizer tends to zero: at $\eps=N^{-1/5}$,
$\mathcal B_N(\eps)=aN^{-1/2}+o(N^{-1/2})$, whereas for any fixed $\eps_0>0$ and $\eps\ge\eps_0$,
$\mathcal B_N(\eps)\ge a\sqrt{1+\eps_0}\,N^{-1/2}$.
At any such minimizer, the first-order condition and
$\frac{\mathrm d}{\mathrm d\eps}\sqrt{1+\eps}\asymp1$ then give
\[
N^{-1/2}
\asymp
\overline\tau_4^{1/5}N^{-4/5}
(\eps_{\mathrm{GMC}}^\star)^{-7/5},
\]
or, equivalently,
\[
(\eps_{\mathrm{GMC}}^\star)^{7/5}
\asymp\overline\tau_4^{1/5}N^{-3/10}.
\]
Taking both sides to the power $5/7$ yields
\[
\eps_{\mathrm{GMC}}^\star
=\Theta\left(\overline\tau_4^{1/7}N^{-3/14}\right).
\]
\end{proof}

\begin{proof}[Proof of Corollary~\ref{cor:gmc-local-overlap}]
Corollary~\ref{cor:local-overlap-kl} gives
$(\chi_n^{2,(b)})^2\le4(1+\delta)^2(D_n^{(b)})^2$. Therefore the chi-square
branch in \cref{eq:gammagmc} and \cref{eq:fourthdrift} imply
\begin{align*}
\frac{T^2R^2L\log\mathcal C}{2G^2\eps^2}
\frac1{N}\sum_{n=0}^{N-1}
\E\left[\frac1B\sum_{b=1}^B(\chi_n^{2,(b)})^2\right]
&\le\frac{2(1+\delta)^2T^2R^2L\log\mathcal C}{G^2\eps^2}
\frac1{N}\sum_{n=0}^{N-1}
\E\left[\frac1B\sum_{b=1}^B(D_n^{(b)})^2\right]\\
&\le\frac{32(1+\delta)^2(1+\eps)^3L^5R^6T^3
\log|\cA|(\log\mathcal C)^2}{G^2\eps^2}
\eta^4\overline\tau_4.
\end{align*}
Substitution in \cref{eq:coveragethm} proves the first displayed bound in the
corollary. Lemma~\ref{lem:stepsize-balance}, with $C$ equal to the complete
coefficient of $\eta^4$ on the last line, gives the stated step size and optimized
rate.
\end{proof}

\begin{proof}[Proof of Theorem~\ref{thm:tic-main}]
Retain the quadratic-KL branch in \cref{eq:appvancompact}. With $M=N$, the
result has the form \cref{eq:generic-rate} with the same $A$ and $B$ as in the
proof of Theorem~\ref{thm:gmc-main}, and
\[
C=\frac{128(1+\eps)^3L^5R^6T^3\log|\cA|(\log\mathcal C)^2}
{\ceps^2}\overline\tau_4.
\]
This is the first bound in Theorem~\ref{thm:tic-main}. Assume first that $A>0$
and $\overline\tau_4>0$. Lemma~\ref{lem:stepsize-balance} gives
\begin{align*}
\eta^\star
&=\min\left\{\sqrt{\frac{A}{BN}},
\left(\frac{A}{CN}\right)^{1/5}\right\},\\
\frac1{N}\sum_{n=0}^{N-1}\E\norm{\nabla J(\theta_n)}_2^2
&\le2\sqrt{\frac{AB}{N}}
+2\left(\frac{A^4C}{N^4}\right)^{1/5}.
\end{align*}
Treating the objective gap and $T,L,R,|\cA|,\mathcal C$ as fixed constants
reduces the first line to
\[
\eta^\star\asymp\min\left\{
\frac1{\sqrt{(1+\eps)N}},
\left(\frac{\ceps^2}{(1+\eps)^3\overline\tau_4N}\right)^{1/5}\right\},
\]
and the second line is exactly the optimized order stated in the theorem. When
$\overline\tau_4=0$ and $A>0$, the $C=0$ branch of
Lemma~\ref{lem:stepsize-balance} instead gives
$\eta^\star=\sqrt{A/(BN)}$ and the first term of the optimized bound. If
$A=0$, choose any positive step size of the main-text order and substitute it
directly into \cref{eq:generic-rate}; the optimization term is then zero and the
same stated upper order follows.
\end{proof}

\end{document}